\documentclass{article}

\PassOptionsToPackage{numbers, compress}{natbib}

     \usepackage[final]{neurips_2020}

\usepackage{float} 
\usepackage[utf8]{inputenc} 
\usepackage[T1]{fontenc}    
\usepackage{hyperref}       
\usepackage{url}            
\usepackage{booktabs}       
\usepackage{amsfonts}       
\usepackage{nicefrac}       
\usepackage{microtype}      
\usepackage{wrapfig}
\usepackage{amsmath,amssymb}
\usepackage{amsthm}
\usepackage{mathtools}
\usepackage[labelformat=simple]{subcaption}

\usepackage{tabularx}
\usepackage{natbib}
\usepackage{paralist, tabularx}
\usepackage{tikz}
\usetikzlibrary{patterns}
\usetikzlibrary{arrows, automata}
\usepackage{pgfplots}
\usepackage[inline]{enumitem}
\usepackage{centernot}

\newtheorem{lemma}{Lemma}
\newtheorem*{proof*}{Proof}
\newtheorem{theorem}{Theorem}
\newtheorem{example}{Example}
\newtheorem*{condition*}{Monotonicity Condition}
\newtheorem{con}{Condition}

\newtheorem{proposition}{Proposition}
\newtheorem{assumption}{Assumption}

\newtheorem{corollary}{Corollary}

\newcommand{\rt}[1]{{\color{black} #1}}

\newcommand{\xr}[1]{{\color{black} #1}}

\title{
How Do Fair Decisions Fare  \\ in Long-term Qualification?
}

\author{
Xueru Zhang${}^{1,}$\thanks{Equal contribution} \hspace{1cm} Ruibo Tu${}^{2,}$\footnotemark[1] \hspace{1cm} Yang Liu${}^3$ \hspace{1cm} Mingyan Liu${}^1$ \AND
Hedvig Kjellström${}^2$ \hspace{1cm} Kun Zhang${}^4$ \hspace{1cm} Cheng Zhang${}^5$ \AND \\
${}^1$University of Michigan, \texttt{\{xueru,mingyan\}@umich.edu} 
\\
${}^2$KTH Royal Institute of Technology, \texttt{\{ruibo,hedvig\}@kth.se}
\\
${}^3$University of California, Santa Cruz, \texttt{yangliu@ucsc.edu}
\\
${}^4$Carnegie Mellon University,
   \texttt{kunz1@cmu.edu}
\\
 ${}^5$Microsoft Research,
   \texttt{Cheng.Zhang@microsoft.com}
}

\begin{document}
\maketitle
\vspace{-6pt}
\begin{abstract}

Although many fairness criteria have been proposed for decision making, their long-term impact on the well-being of a population remains unclear. In this work, we study the dynamics of population qualification and algorithmic decisions under a partially observed Markov decision problem setting. By characterizing the equilibrium of such dynamics, we analyze the long-term impact of static fairness constraints on the equality and improvement of group well-being. Our results show that static fairness constraints can either promote equality or exacerbate  disparity depending on the driving factor of qualification transitions and the effect of sensitive attributes on feature distributions. We also consider possible interventions that can effectively improve group qualification or promote equality of group qualification. Our theoretical results and experiments on static real-world datasets with simulated dynamics show that our framework can be used to facilitate social science studies. 
\end{abstract}

\vspace{-5pt}
\section{Introduction}\label{sec:introduction}
Automated decision making systems 
trained with real-world data can have inherent bias and exhibit discrimination against disadvantaged groups. One common approach to alleviating the issue is to impose fairness constraints on the decision such that certain statistical measures (e.g., true positive rate, positive classification rate, etc.) across multiple groups are (approximately) equalized. However, their effectiveness has been studied mostly in a static framework, where only the immediate impact of the constraint is assessed but not its long-term consequences. 
Recent studies have shown that imposing static fairness criteria intended to protect disadvantaged groups can actually lead to pernicious long-term effects \cite{pmlr-v80-liu18c,zhang2019group}. These long-term effects are heavily shaped by the interplay between algorithmic decisions and  individuals' reactions \citep{liu2019disparate}:  algorithmic decisions lead to changes in the underlying feature distribution, which then feeds back into the decision making process. Understanding how this type of coupled dynamics evolve is a major challenge \citep{d2020fairness}. 

Toward this end, we \rt{consider a discrete-time sequential decision process} applied to a certain population, where responses to the decisions made at each time step are manifested in changes in the features of the population in the next time step. Our goal is to understand how (static) fairness criteria in this type of decision making affect the evolution of group well-being and characterize any equilibrium state the system may converge to. In particular, we will focus on {\em myopic} policies that maximize the immediate utility under static fairness constraints, and examine their impact on different groups in the long run.

More specifically, we seek to study the dynamics of group qualification rates \citep{khajehnejad2019optimal,liu2019disparate,socialequality2019,williams2019dynamic} and evaluate the long-term impact of various static fairness constraints imposed on \rt{decision making.
We} examine whether these static fairness \rt{constraints}  mitigate or worsen the qualification disparity in the long-run. Our work can be applied to a variety of applications such as recruitment and bank lending. In these applications, an {\em institute} observes {\em individuals'} features (e.g., credit scores), and makes {\em myopic decisions} (e.g., issue loans) by assessing such features against some variables of interest (e.g., ability to repay) which are unknown and unobservable to the institute when making decisions. Individuals respond to the decisions by investing in effort to either improve or maintain their qualification in the next time step.  These actions collectively change the qualification rate of the population. In summary, our main contributions are:  

1. \emph{We analyze the equilibrium of qualification rates in different groups under a general class of fairness constraints (Section \ref{sec:evolution}).} 
We use a Partially Observed Markov Decision Process (POMDP) framework to model the sequential decision making in different scenarios (Section \ref{sec:model}). 
Using this model, we 
\rt{show that under our formulation optimal policies are of the threshold type}
and provide a way to compute the threshold.
We then prove the existence of an equilibrium (in terms of long-term qualification rates) using 
\rt{threshold policies}
and provide sufficient conditions for a unique equilibrium.
\\
2. \emph{We analyze the impact of fairness constraints on the disparity of qualification rates} when the equilibrium is unique (Section \ref{sec:impacts}). Our findings suggest that the same fairness constraint can have opposite impacts on the equilibrium depending on the underlying problem scenario.
\\
3. \emph{We explore alternative interventions that can be effective in improving qualification rates at the equilibrium and promoting equality across different groups (Section \ref{sec:interventions}).} 
\\
4. \emph{We examine our theory on synthetic Gaussian datasets and two real-world scenarios (Section \ref{sec:exp}).} Our experiments show \rt{that our framework can help examine findings} cross domains and  support real-life policy making. 
\vspace{-0.35em}
\section{Related Work}\label{app:relat}
\vspace{-0.35em}
Among existing works on fairness in sequential decision making problems \cite{zhang2020fairness}, many assume that the population's feature distribution neither changes over time nor is it affected by decisions; examples include studies on handling bias in online learning \citep{bechavod2019equal,dimitrakakis2019bayesian,ensign2018runaway,ensign2018decision, gillen2018online,heidari2018preventing, kallus2018residual,kilbertus2019improving} and bandits problems \citep{auer2002finite,chen2019fair,joseph2016fairness,joseph2018meritocratic,li2019combinatorial,liu2017calibrated,patil2019achieving,tang2020fair}. The goal of most of these work is to design algorithms that can learn near-optimal policy quickly from the sequentially arrived data and the partially observed information, and understand the impact of imposing fairness intervention on the learned policy (e.g., total utility, learning rate, sample complexity, etc.)

 However, recent studies \citep{aneja2019no,chaney2018algorithmic,fuster2018predictably} have shown that there exists a complex interplay between algorithmic decisions and individuals, e.g., user participation dynamics \citep{pmlr-v80-hashimoto18a,xueru1,zhang2019group},  strategic reasoning in a game \citep{hu2019disparate,khajehnejad2019optimal}, etc., such that decision making directly leads to changes in the underlying feature distribution, which then feeds back into the decision making process. Many studies thus aim at understanding the impacts of imposing fairness constraints when decisions affect underlying feature distribution. For example, \cite{pmlr-v80-liu18c,heidari2019long,kannan2019downstream,khajehnejad2019optimal} construct two-stage models where only the one-step impacts of fairness intervention on the underlying population are examined but not the long-term impacts in a sequential framework; 
 \citep{jabbari2017fairness, nabi2019learning} focus on the fairness in reinforcement learning, of which the goal is to learn a long-run optimal policy that maximizes the cumulative rewards subject to certain fairness constraint; \cite{pmlr-v80-hashimoto18a,zhang2019group} construct a user participation dynamics model where individuals respond to perceived decisions by leaving the system uniformly at random. The goal is to understand the impact of various fairness interventions on group representation.  
 
 Our work is most relevant to \citep{hu2019disparate, liu2019disparate,socialequality2019,williams2019dynamic}, which study the long-term impacts of decisions on the groups' qualification states with 
 different dynamics.  
 In \citep{hu2019disparate,liu2019disparate}, strategic individuals are assumed to be able to observe the current policy, based on which they can manipulate the qualification states strategically to receive better decisions. However, there is a lack of study on the influence of the sensitive attribute on dynamics and impact of fairness constraints. Besides, in many cases, the qualification states are affected by both the policy and the qualifications at the previous time step, which is considered in \citep{socialequality2019,williams2019dynamic}. 
However, they assume that the decision maker have access to qualification states and the dynamics of the qualification rates is the same in different groups, i.e.,the equally qualified people from different groups after perceiving the same decision will have the same future qualification state. In fact, the qualification states are unobservable in most cases, and the dynamics can vary across different groups. If considering such difference, the dynamics can be much more complicated such that the social equality can not be attained easily as concluded in \citep{socialequality2019,williams2019dynamic}.

\section{Problem Formulation} 
\vspace{-0.5em} 
\label{sec:model}
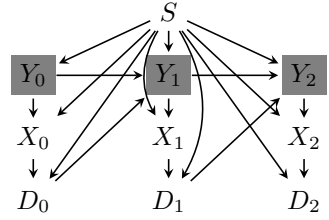
\begin{wrapfigure}[9]{r}{0.37\textwidth}
    \centering
    \vspace{-4.1em}
    \begin{tikzpicture}[
			scale=0.12,
            > = stealth, 
            shorten > = 1pt, 
            auto,
            node distance = 1cm, 
            semithick 
        ]

        \tikzstyle{state}=[
            draw = none,
        ]
        \tikzstyle{hstate}=[
            draw = none,
            fill = gray
        ]
      \node (S) at  (0,14){$S$};

        \node[state] (X0) at  (-15,0){$X_0$};
        \node[state] (X1) at  (0,0){$X_1$};
        \node[state] (X2) at  (15,0){$X_2$};
        
        \node[hstate] (Y0) at  (-15,7){$Y_0$};
        \node[hstate] (Y1) at  (0,7){$Y_1$};
        \node[hstate] (Y2) at  (15,7){$Y_2$};
        
        \node[state] (A0) at  (-15,-7){$D_0$};
        \node[state] (A1) at  (0,-7){$D_1$};
        \node[state] (A2) at  (15,-7){$D_2$};
        
        \path[<-] (X0) edge node {} (Y0);
        \path[->] (X0) edge node {} (A0);
        
        \path[<-] (X1) edge node {} (Y1);
        \path[->] (X1) edge node {} (A1);
        
        \path[<-] (X2) edge node {} (Y2);
        \path[->] (X2) edge node {} (A2);

        \path[->] (Y0) edge node {} (Y1);
        \path[->] (Y1) edge node {} (Y2);
        
        \path[->] (A0) edge node {} (Y1);
        \path[->] (A1) edge node {} (Y2);

        
        \path[->] (S) edge node {} (Y0);
        \path[->] (S) edge node {} (Y1);
        \path[->] (S) edge node {} (Y2);
        
        \path[->] (S) edge node {} (A0);
        \path[->] (S) edge[bend left] node {} (A1);
        \path[->] (S) edge node {} (A2);

        \path[->] (S) edge node {} (X0);
        \path[->] (S) edge[bend right] node{} (X1);
        \path[->] (S) edge node {} (X2);

    \end{tikzpicture}
    \caption{ The graphical representation of our model where gray shades indicate latent variables.}
  \label{fig:model_diagram}
\end{wrapfigure}
\textbf{Partially Observed Markov Decision Process (POMDP).} Consider two groups $\mathcal{G}_a$ and $\mathcal{G}_b$ distinguished by a sensitive {\em attribute}  ${S=s} \in \{a,b\}$ (e.g., gender), \rt{with fractions} 
$p_s\coloneqq \mathbb{P}(S = s)$ of the population. At time $t$, an \textit{individual} with attribute $s$ has feature\footnote{\xr{For simplicity of exposition, our analysis is based on one-dimensional feature space. However, the conclusions hold for high-dimensional features. This can be done by first mapping the feature space to a one-dimensional qualification profile space, this extension is given in Appendix \ref{app:highDim}}.} ${X_t = x} \in \mathbb{R}$ determined by a hidden \textit{qualification} state  ${Y_t=y}\in \{0,1\}$, \rt{and} both are time-varying. We adopt a natural assumption that an individual's attribute and current features constitute sufficient statistics, so that conditioned on these, the decision is independent of past features \rt{and}
decisions. This allows an {\em institute} (decision maker) to adopt a Markov policy:  it makes decisions ${D_t=d} \in \{0,1\}$ (reject or accept) using a policy \footnote{
\xr{We use group-dependent policies so that the optimal policies can achieve the \textit{perfect} fairness, i.e., certain statistical measures are equalized \textit{exactly}, 
which allows us to study the impact of fairness constraint \textit{precisely}. Although using group-dependent policies might be prohibited in some scenarios (e.g., criminal justice), our qualitative conclusions are applicable to cases when two groups share the same policy, under which the \textit{approximate} fairness is typically attained to maximize utility.}}
$\pi^s_t(x) \coloneqq \mathbb{P}(D_t = 1\mid X_t = x,S=s)$
to maximize an instantaneous utility $R_t(D_t,Y_t)$, possibly subject to certain constraints. An individual is informed of the decision, and subsequently takes actions that may change the qualification $Y_{t+1}$ and features $X_{t+1}$. The \rt{latter}
is used to drive the institute's decision at the next time step. This process is shown in Fig. \ref{fig:model_diagram}. Note that this model can be viewed as capturing either a randomly selected individual repeatedly going through the decision cycles, or population-wide average when all individuals are subject to the decision cycles. Thus, $\alpha_t^s \coloneqq \mathbb{P}(Y_t = 1\mid S = s)$ is the probability of an individual from $\mathcal{G}_s$ qualified at time $t$ at the individual level, while being the \textit{qualification rate} at the group level. One of our primary goals is to study \rt{how $\alpha^s_t$
evolves
under different (fair) policies.}

\textbf{Feature generation process.} 
\rt{In many real-world scenarios, \xr{equally qualified individuals from different groups can have different features, potentially due to the different culture backgrounds and physiological differences of different demographic
groups.} 
Therefore, we consider that} at time step $t$, given $Y_t=y$ and $S=s$, features $X_t$ are generated by $G^s_y(x) \coloneqq \mathbb{P}(X_t=x\mid Y_t = y, S =s)$.
This will be referred to as the \textit{feature distribution} and assumed time-invariant.  The convex combination 
$\mathbb{P}(X_t=x\mid S =s) = \alpha^s_tG^s_1(x)+(1-\alpha^s_t)G^s_0(x)$ will be referred to as the \textit{composite feature distribution} of group $\mathcal{G}_s$ at time $t$. 

\textbf{Transition of qualification states.}
At time $t$, after receiving decision $D_t$,  an individual takes actions such as exerting effort/investment, imitating others, etc., which results in a new qualification $Y_{t+1}$. This is modeled by a set of transitions $T^s_{yd}\coloneqq \mathbb{P}(Y_{t+1} = 1\mid Y_{t}=y, D_{t}=d, S=s)$, \xr{which are time-invariant and group-dependent}.
These transitions characterize individuals' ability to maintain or improve its qualification. Note that we don't model individuals' strategic responses as in \cite{hu2019disparate,khajehnejad2019optimal}, but rather use $T_{yd}^s$ to capture the overall effect; in other words, this single quantity may encapsulate the individual's willingness to exert effort, the cost of such effort, as well as the strength of community support, etc. Specifically, $T^s_{0d}$ (resp. $T^s_{1d}$) represents the probability of individuals from $\mathcal{G}_s$ who were previously unqualified (resp. qualified) became (resp. remain) qualified after receiving decision $d \in \{0,1\}$. \xr{Note that the case when feature distributions or transitions are group-independent is a special case of our formulation, i.e., by setting $G^a_{y} = G^b_{y}$ or $T^a_{yd} = T^b_{yd}$.}

\textbf{Fair myopic policy of an institute.}
A myopic policy $\pi_t$ at time $t$ aims at maximizing the instantaneous expected utility/reward $ \mathcal{U}(D_t,Y_t) = \mathbb{E}[R_t(D_t,Y_t)]$, where the institute gains $u_+ > 0$ by accepting a qualified individual and incurs a cost $u_- > 0$ by accepting an unqualified individual, i.e., $R_t(D_t,Y_t)\coloneqq \begin{cases}
u_+, & \text{ if } Y_t=1 \text{ and } D_t = 1\\
-u_-, & \text{ if } Y_t=0 \text{ and } D_t = 1\\
0, & \text{ if } D_t = 0
\end{cases}$. A fair myopic policy maximizes the above utility subject to a fairness constraint $\mathcal{C}$. We focus on a set of group fairness constraints that equalize certain statistical measure between $\mathcal{G}_a$ and $\mathcal{G}_b$. A commonly studied (one-shot) fair machine learning problem is to find $(\pi^a_t,\pi^b_t)$ that solves the following constrained optimization,
\begin{eqnarray}
 &\max_{\pi^a,\pi^b}&\mathcal{U}(D_t,Y_t) =
p_a\mathbb{E}[R_t(D_t,Y_t)|S=a]  +
p_b\mathbb{E}[R_t(D_t,Y_t)|S=b] \nonumber  \\ &\text{ s.t.  }& 
\mathbb{E}_{X_t\sim \mathcal{P}^a_{\mathcal{C}}}[\pi^a(X_t)] = \mathbb{E}_{X_t\sim \mathcal{P}^b_{\mathcal{C}}}[\pi^b(X_t)]~,  \label{eq:fair_constraint_e}
\end{eqnarray}
where $\mathcal{P}^s_{\mathcal{C}}$ is some probability distribution over features $X_t$ and specifies the fairness metric $\mathcal{C}$. Many popular fairness metrics can be written in this form, e.g., 
\begin{enumerate}[noitemsep,topsep=0pt]
     \item Equality of Opportunity (\texttt{EqOpt}) \cite{hardt2016equality}: this requires the true positive rate (TPR) \rt{to} be equal, 
     i.e., $\mathbb{P}(D_t = 1|Y_t = 1,S=a) = \mathbb{P}(D_t = 1|Y_t = 1,S=b)$. This is equivalent to $\mathbb{E}_{X_t|Y_t = 1,S=a}[\pi_t^a(X_t)] = \mathbb{E}_{X_t|Y_t = 1,S=b}[\pi_t^b(X_t)]$, i.e.,
     $\mathcal{P}^s_{\texttt{EqOpt}}(x) = G^s_1(x)$.
     \item Demographic Parity (\texttt{DP}) \cite{barocas-hardt-narayanan}: this requires the positive rate (PR) \rt{to} be equal, i.e., 
     $\mathbb{P}(D_t = 1|S=a) = \mathbb{P}(D_t = 1|S=b)$. This is equivalent to $\mathbb{E}_{X_t|S=a}[\pi_t^a(X_t)] = \mathbb{E}_{X_t|S=b}[\pi_t^b(X_t)]$, i.e.,
     $\mathcal{P}^s_{\texttt{DP}}(x) = (1-\alpha^s_t) G^s_0(x)+ \alpha^s_t G^s_1(x)$.
\end{enumerate}
We focus on this class of myopic polices in this paper, and refer to the solution to \eqref{eq:fair_constraint_e} as the optimal policy. 
We further define \textit{qualification profile}\footnote{We assume the institute has perfect knowledge of $\gamma^s_t(x)$. In practice, this is obtained via learning/estimating $\alpha^s_t$ and $G^s_y(x)$ from data \cite{james1978estimation,patra2016estimation}. 
},  
$\gamma_t^s(x)$, the probability an individual with features $x$ from group $\mathcal{G}_s$ is qualified at $t$, i.e.,
\begin{eqnarray} \label{eq:qlfn_profile}
\gamma_t^s(x) = \mathbb{P}(Y_t = 1\mid X_t = x, S=s)
= \frac{1}{\frac{G^s_0(x)}{G^s_1(x)}(\frac{1}{\alpha_t^s}-1) + 1}, ~~x\in \mathbb{R}.
\end{eqnarray}
Then the utility  obtained from the group  
$\mathcal{G}_s$ at time step $t$ is given by $\mathbb{E}[R_t(D_t,Y_t)|S=s] = \mathbb{E}_{X_t|S=s}[\pi^s_t(X_t)(\gamma_t^s(X_t)(u_++u_-) -u_-)]$. 
\rt{Detailed derivation is shown in Appendix \ref{app:derivations}.}

\section{Evolution and Equilibrium Analysis of Qualification Rates} \label{sec:evolution}
In this section, we first solve the one-shot optimization problem \eqref{eq:fair_constraint_e} (Sec. \ref{sec:opt_policy}). We then show that under the optimal policy, there exists an equilibrium of qualification rates in the long run, and that a sufficient condition for its uniqueness is also introduced (Sec. \ref{sec:evolutions_equilibrium}). 

\subsection{Threshold policies are optimal} \label{sec:opt_policy}
If an individual's qualification is observable, the optimal policy is straightforward absent of fairness constraints:  accepting all qualified ones and rejecting the rest. When qualification is not observable, the institute needs to infer from observed features and accepts those most likely to be qualified. Next we show that under mild assumptions, optimal policies are in the form of threshold policies.
\rt{
\begin{assumption}\label{ass:mono-inference}
$G_y^s(x)$ and the CDF, $\int_{-\infty}^{x} {G}_y^s(z)dz,$ 
are continuous in $x\in \mathbb{R}$, $\forall y,s$; 
\xr{$G^s_1(x)$ and $G^s_0(x)$ satisfy strict monotone likelihood ratio property, i.e., $\frac{G^s_1(x)}{G^s_0(x)}$ is strictly increasing in $x\in\mathbb{R}$.} 
\end{assumption}
\begin{assumption}\label{ass:fair_prob}
$\forall s \in \{a,b\}$, $\mathcal{P}_{\mathcal{C}}^s(x)$ is continuous in $x\in\mathbb{R}$;  
$\frac{\mathbb{P}(X=x\mid S=s)}{\mathcal{P}_{\mathcal{C}}^s(x)}$ is non-decreasing in $x\in\mathbb{R}$.
\end{assumption}
}

\rt{Assumption \ref{ass:mono-inference} says that an individual is more likely to be qualified as his/her feature value increases\footnote{When qualification increases as the feature value $x$ decreases, one can simply use the opposite of $x$.}. \xr{We show that }
\emph{under Assumption \ref{ass:mono-inference}, the optimal unconstrained policy is a threshold policy, i.e., $\forall x, t$ and $s\in \{a,b\}$, $\pi^s_t(x) = \textbf{1}(x\geq \theta^s_t)$ for some $\theta^s_t \in \mathbb{R}$.} Assumption \ref{ass:fair_prob} \xr{limits the types of fairness constraints, but is satisfied by many commonly used ones, including \texttt{EqOpt} and \texttt{DP}. We show that} 
\emph{for any fairness constraint $\mathcal{C}$ satisfying Assumption \ref{ass:fair_prob}, the optimal fair policy is a threshold policy.} \xr{The proof of these results is given in Appendix \ref{app:proofs}, which is consistent with Theorem 3.2 in \citep{corbett2017algorithmic}.}
Moreover, under Assumption \ref{ass:mono-inference} and \ref{ass:fair_prob}, a threshold as a function of qualification rates, $\theta^s_t\coloneqq \theta^s(\alpha^a_t,\alpha^b_t)$, is continuous and non-increasing in $\alpha^a_t$ and $\alpha^b_t$. \xr{In the next Lemma \ref{lemma:opt_fair_policy_eq}, we further characterize these optimal (fair) thresholds in the optimal (fair) policies.}}

\begin{lemma}[Optimal (fair) threshold]\label{lemma:opt_fair_policy_eq}
Let $(\gamma^a(x),\gamma^b(x))$ be a pair of qualification profiles for groups $\mathcal{G}_a$ and $\mathcal{G}_b$ at $t$. Let threshold pairs $(\theta^{a*}_{\texttt{UN}},\theta^{b*}_{\texttt{UN}})$ and $(\theta^{a*}_{\mathcal{C}},\theta^{b*}_{\mathcal{C}})$ be the unconstrained and fair optimal thresholds under constraint $\mathcal{C}$, respectively. Then we have ${\gamma}^a(\theta^{a*}_{\texttt{UN}}) = {\gamma}^b(\theta^{b*}_{\texttt{UN}}) = \frac{u_-}{u_++u_-}$ and 
\begin{eqnarray}
\resizebox{0.9\hsize}{!}{$%
p_a  
\Big({\gamma}^a(\theta^{a*}_{\mathcal{C}})-\frac{u_{-}}{u_{+}+u_{-}}\Big)
\frac{\mathbb{P}(X=\theta^{a*}_{\mathcal{C}}\mid S=a) }{\mathcal{P}^a_{\mathcal{C}}( \theta^{a*}_{\mathcal{C}})}
+p_b  
\Big({\gamma}^b(\theta^{b*}_{\mathcal{C}}) -\frac{u_{-}}{u_{+}+u_{-}}\Big)
\frac{\mathbb{P}(X=\theta^{b*}_{\mathcal{C}}\mid S=b) }
{\mathcal{P}_{\mathcal{C}}^b( \theta^{b*}_{\mathcal{C}})}
=0.$}
\label{eq:opt_fair_policy}
\end{eqnarray}
\end{lemma}
Here we have removed the subscript $t$ since the thresholds are not $t$-dependent; they only depend on current qualification rates. The solution to Eqn. \eqref{eq:opt_fair_policy} is the threshold pair $(\theta^{a*}_{\mathcal{C}},\theta^{b*}_{\mathcal{C}})$ that satisfies the fairness constraint $\int_{\theta^{a*}_{\mathcal{C}}}^{\infty}\mathcal{P}^a_{\mathcal{C}}(x)dx = \int_{\theta^{b*}_{\mathcal{C}}}^{\infty}\mathcal{P}^b_{\mathcal{C}}(x)dx$ in Eqn. \eqref{eq:fair_constraint_e} while maximizing the expected utility $\mathcal{U}(D,Y)$ at time $t$. Under  \texttt{DP} and \texttt{EqOpt} fairness, Eqn. \eqref{eq:opt_fair_policy} can be reduced to
\begin{eqnarray*}
p_a  
{\gamma}^a(\theta^{a*}_{\texttt{DP}})
+p_b  
{\gamma}^b(\theta^{b*}_{\texttt{DP}})
=\frac{u_{-}}{u_{+}+u_{-}}; ~~~~
\frac{p_a  \alpha^a}{{\gamma}^a(\theta^{a*}_{\texttt{EqOpt}})} + \frac{p_b  \alpha^b}{{\gamma}^b(\theta^{b*}_{\texttt{EqOpt}})}
 = \frac{p_a  \alpha^a}{
 \frac{u_-}{u_++u_-}
 } + \frac{p_b  \alpha^b}{
  \frac{u_-}{u_++u_-}
 }.
\end{eqnarray*}
Lemma \ref{lemma:opt_fair_policy_eq} also indicates the relation between the unconstrained and fair optimal polices, e.g., a group's qualification profile evaluated at the unconstrained threshold is the same as the weighted combination of two groups' qualification profiles evaluated at their corresponding fair thresholds under \texttt{DP}. 
\subsection{Evolution and equilibrium analysis} \label{sec:evolutions_equilibrium}
We next examine what happens as the institute repeatedly makes decisions based on the optimal (fair) policies derived in Sec. \ref{sec:opt_policy}, while individuals react by taking actions to affect their future qualifications.  We say the qualification rate of $\mathcal{G}_s$ is at an {\em equilibrium} if $\alpha^s_{t+1} = \alpha^s_t,\forall t \geq t_o$ for some $t_o$, or equivalently, if $\lim_{t\to \infty} \alpha^s_t = \alpha^s$ is well-defined for some $\alpha^s\in [0,1]$. Analyzing equilibrium helps us understand the property of the population in the long-run. 
We begin by characterizing the dynamics of qualification rates $\alpha_t^s$ under policy $\pi^s_t$ as follows (see Appendix \ref{app:derivations} for the derivation):  
\begin{eqnarray} \label{eq:dynamics}
\alpha_{t+1}^s = g^{0s}\xr{(\alpha_{t}^a,\alpha_{t}^b)\cdot}(1-\alpha_{t}^s) + g^{1s}\xr{(\alpha_{t}^a,\alpha_{t}^b)\cdot}\alpha_{t}^s,~~ s\in \{a,b\}~, 
\end{eqnarray}
where $g^{ys}\xr{(\alpha_{t}^a,\alpha_{t}^b)} \coloneqq \mathbb{E}_{X_{t}|Y_{t}=y,S=s}\Big[(1-\pi_{t}^s(X_{t}))T_{y0}^s+\pi_{t}^s(X_{t})T_{y1}^s\Big]$ \xr{depends on qualification rates $\alpha_{t}^a,\alpha_{t}^b$ through the policy $\pi_{t}^s$}. When $\pi^s_t(x) = \textbf{1}(x\geq \theta^s_t)$, this reduces to $g^{ys}\xr{(\alpha_{t}^a,\alpha_{t}^b)} 
\coloneqq 
T_{y0}^s\int_{-\infty}^{\theta^s_{t}} G_y^s(x)dx + T_{y1}^s\int_{\theta^s_{t}}^{\infty} G_y^s(x)dx$
,  $y\in \{0,1\}$. \xr{Denote $g^{ys}(\alpha_{t}^a,\alpha_{t}^b)\coloneqq g^{ys}(\theta^s(\alpha_{t}^a,\alpha_{t}^b))$. 

Dynamics \eqref{eq:dynamics} says that the qualified people at each time consists of two parts: the qualified people in the previous time step remain being qualified, and those who were unqualified in the previous time step change to become qualified.}

Theorem \ref{lemma:exist_equilibrium} below shows that for any transition and any threshold policy that are continuous in qualification rates, the above dynamical system always has at least one equilibrium.

\begin{theorem}[Existence of equilibrium]\label{lemma:exist_equilibrium} Consider a dynamics \eqref{eq:dynamics} 
\rt{with} a threshold policy $\theta^s(\alpha^a,\alpha^b)$ that is continuous in $\alpha^a$ and $\alpha^b$. $\forall\, T_{yd}^s\in (0,1)$, there exists at least one  equilibrium $(\widehat{\alpha}^a,\widehat{\alpha}^b)$.
\end{theorem}
While an equilibrium exists under any set of transitions, its specific property \xr{(e.g., quantity, value, etc.)} highly depends on 
\rt{transition probabilities which specify different \xr{user dynamics.}
}
We focus on two scenarios given in the condition below.

\begin{con}\label{con:transition}
$\forall s\in \{a,b\}$,
\begin{enumerate*}[label=(\Alph*),itemjoin={; \quad}]
    \item $T^s_{01} \leq T^s_{00}$ and  $T^s_{11} \leq T^s_{10}$ \label{con:transition_I}
     \item $T^s_{01} \geq T^s_{00}$ and  $T^s_{11} \geq T^s_{10}$.\label{con:transition_II}
\end{enumerate*}
\end{con}
As mentioned, transitions 
\xr{$T_{yd}^s$} 
characterize the ability of individuals from $\mathcal{G}_s$ to maintain/improve their future qualifications, \xr{this value summarizes individual's behaviors}. On one hand, an accepted individual may feel less motivated to remain qualified (if it was) or become qualified (if it was not). On the other hand, the accepted individual may have access to better resources or feel more inspired to remain or become qualified. These competing factors (referred to later as the ``\rt{lack of }motivation'' effect and the ``leg-up'' effect, respectively) may work simultaneously, and the net effect can be context dependent. 
Condition \ref{con:transition}\ref{con:transition_I} (resp. Condition \ref{con:transition}\ref{con:transition_II}) suggests that the first (resp.~second) effect is dominant for both qualified and unqualified individuals. There are two other combinations: (C) $T^s_{01} \geq T^s_{00}$ and  $T^s_{11} \leq T^s_{10}$; (D) $T^s_{01} \leq T^s_{00}$ and  $T^s_{11} \geq T^s_{10}$, \xr{under which the qualified and unqualified are dominant by different effects}. These cases incur more uncertainty; slight changes in feature distributions or transitions may result in opposite conclusions. More discussions are in Appendix \ref{app:discussion}.

Given the existence of an equilibrium, Theorem \ref{thm:unique_equilibrium} further introduces sufficient conditions for it to be unique. Based on the unique equilibrium, we can evaluate and compare policies (Sec. \ref{sec:impacts}), and design effective interventions to promote long-term  equality and/or the overall qualifications (Sec. \ref{sec:interventions}).

\begin{theorem}[Uniqueness of  equilibrium]\label{thm:unique_equilibrium}
Consider a \rt{decision-making system with dynamics}
\eqref{eq:dynamics} 
\rt{and} either unconstrained or fair optimal threshold policy. Let $ h^s(\theta^s(\alpha^a,\alpha^b)) \coloneqq \frac{1-g^{1s}(\theta^s(\alpha^a,\alpha^b))}{g^{0s}(\theta^s(\alpha^a,\alpha^b))}, s\in\{a,b\}$. Under Assumptions \ref{ass:mono-inference} and \ref{ass:fair_prob}, 
a sufficient condition for \eqref{eq:dynamics} to have a unique equilibrium is as follows, $\forall s\in\{a,b\}$: 
\begin{enumerate}[noitemsep,topsep=0pt,leftmargin=*]
\item Under Condition \ref{con:transition}\ref{con:transition_I}, $\big|\frac{\partial h^s(\theta^s(\alpha^a,\alpha^b))}{\partial \alpha^{- s}}\big| < 1$, $\forall \alpha^s \in [0,1]$, where $-s\coloneqq \{a,b\}\setminus s$;
\item Under Condition \ref{con:transition}\ref{con:transition_II}, $\big|\frac{\partial h^s(\theta^s(\alpha^a,\alpha^b))}{\partial \alpha^{-s}}\big| < 1$ and $\big|\frac{\partial h^s(\theta^s(\alpha^a,\alpha^b))}{\partial \alpha^{s}}\big| < 1, \forall \alpha^a,\alpha^b \in [0,1]$.
\end{enumerate}
\end{theorem}
These sufficient conditions can further be satisfied if for the qualified ($y=1$) and the unqualified ($y=0$), transitions $T^s_{y1}$ and $T^s_{y0}$ are sufficiently close (see Corollary \ref{rmk:unique},  Appendix \ref{app:corollary1}), i.e., policies have limited influence on the qualification dynamics. It is worth noting that \xr{the conditions of Theorem \ref{thm:unique_equilibrium} only guarantee uniqueness of equilibrium but not stability, i.e., it is possible that the qualification rates oscillate and don't converge under this discrete-time dynamics (see examples on COMPAS data in Sec. \ref{sec:exp}).}
 \xr{The uniqueness can be guaranteed and further attained if the dynamics} \eqref{eq:dynamics} satisfies $L$-Lipschitz condition with $L<1$. \xr{However, Lipschitz condition} is relatively stronger than the condition in Theorem \ref{thm:unique_equilibrium} (see the comparison in Appendix \ref{app:discussion}). 

\begin{wrapfigure}[11]{r}{0.645\textwidth}
\centering
\vspace{-1.7em}
\includegraphics[trim=0cm 0.4cm 0cm 0,clip=true, width=0.22\textwidth]{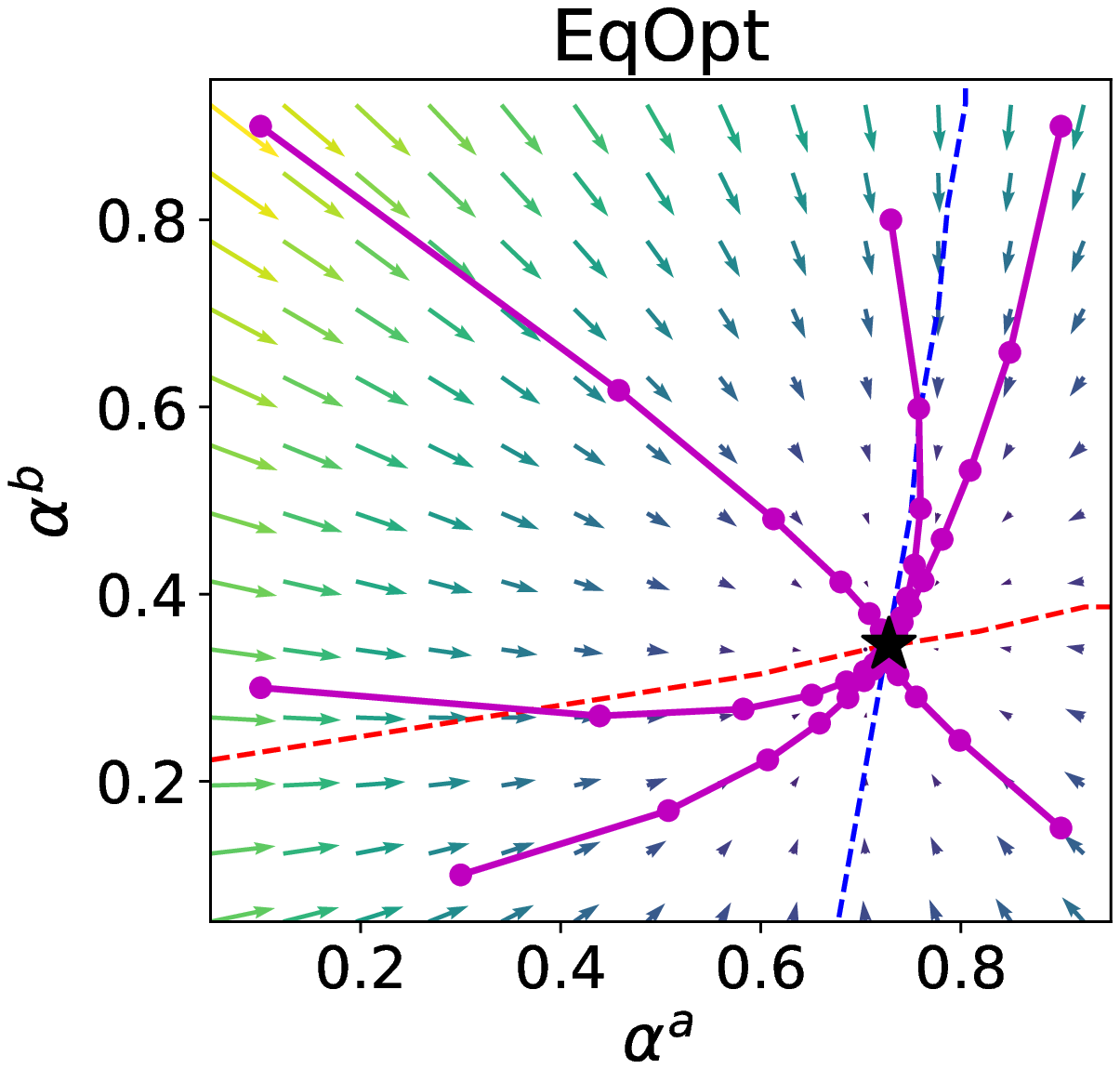}
\includegraphics[trim=0cm 0.4cm 0cm 0,clip=true, width=0.205\textwidth]{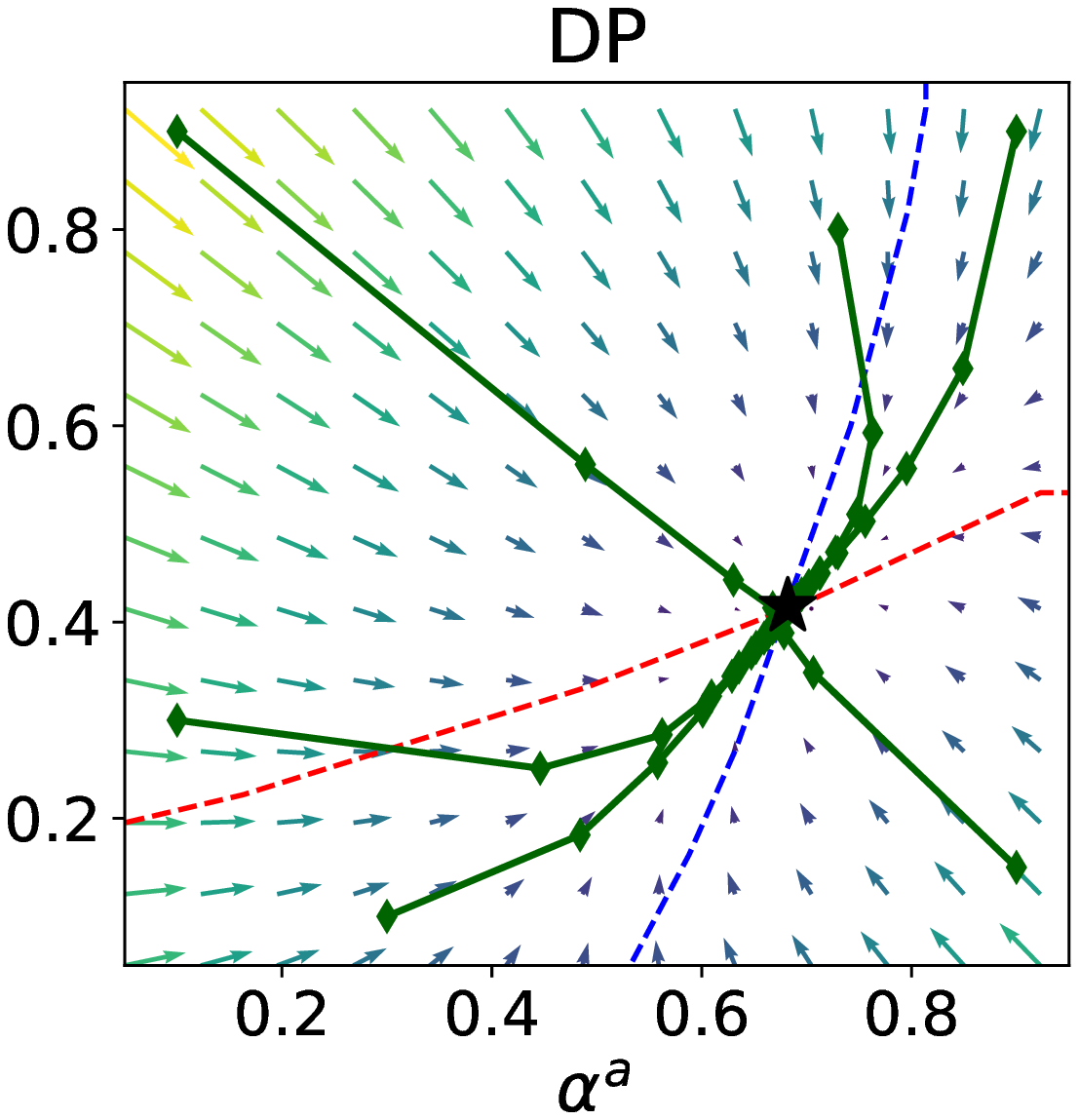}
\includegraphics[trim=0cm 0.4cm 0cm 0,clip=true, width=0.205\textwidth]{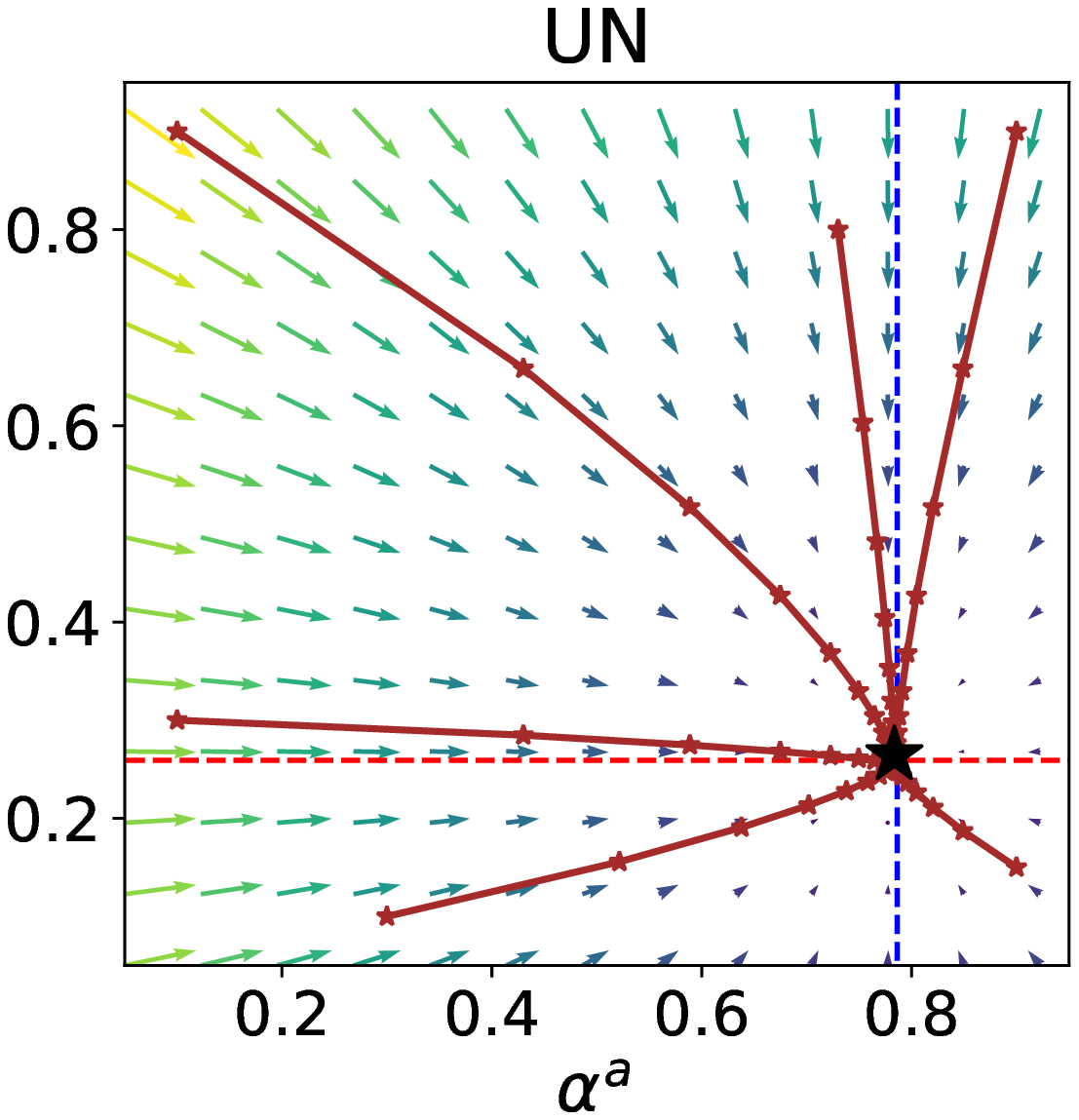}
\caption{Illustration of $\{(\alpha^a_t,\alpha^b_t)\}_t$
for a Gaussian case under \texttt{EqOpt}, \texttt{DP}, \texttt{UN} optimal policies: each plot shows 6 sample paths with each circle/diamond/star representing one pair of $(\alpha^a_t,\alpha^b_t)$. }
\label{fig:trajectory}
\end{wrapfigure}
Figure \ref{fig:trajectory} illustrates trajectories of qualification rates $(\alpha^a_t,\alpha^b_t)$ and the equilibrium for a Gaussian case under Condition \ref{con:transition}\ref{con:transition_II} (see details in Appendix \ref{app:exp}). Let $g^{ys}\coloneqq g^{ys}(\theta^s(\alpha^a,\alpha^b))$, the points ($\alpha^a$,$\alpha^b$) on the red, and blue dashed curves satisfy $\alpha^b = g^{0b}(1-\alpha^b)+ g^{1b}\alpha^b$ and $\alpha^a = g^{0a}(1-\alpha^a) + g^{1a}\alpha^a$, respectively. Their intersection (black star) is the equilibrium $(\widehat{\alpha}^a,\widehat{\alpha}^b)$. \rt{The sufficient conditions 
in Theorem \ref{thm:unique_equilibrium}}  guarantee these two curves have only one intersection. Moreover, observe that these two curves split the plane $\{(\alpha^a,\alpha^b):\alpha^a\in[0,1],\alpha^b\in[0,1]\}$ into four parts, which can be used for determining how $(\alpha^a_t,\alpha^b_t)$ will change at $t$. For example, if $(\alpha^a_t$,$\alpha^b_t)$ falls into the left side of the blue dashed curve, then $\alpha^a_{t+1} > \alpha^a_{t}$;  if $(\alpha^a_t$,$\alpha^b_t)$ falls into the lower side of the red dashed curve, then $\alpha^b_{t+1} > \alpha^b_{t}$.

\section{The Long-term Impact of Fairness Constraints}\label{sec:impacts} 
In this section, we analyze the long-term impact of imposing fairness constraints on the equality of group qualification. We will do so in the presence of {\em natural equality 
\xr{(and inequality)}}
\cite{socialequality2019} where equitable equilibria are attained naturally without imposing additional constraints (in our context, this means attaining  $\widehat{\alpha}^a_{\texttt{UN}} =\widehat{\alpha}^b_{\texttt{UN}}$ using  unconstrained polices).

Although there may exist multiple equilibria, in this section we will assume the conditions in Theorem \ref{thm:unique_equilibrium} hold \rt{under Assumption \ref{ass:mono-inference} and \ref{ass:fair_prob}}
and limit ourselves to the unique equilibrium cases under \texttt{DP} and \texttt{EqOpt}, thereby providing a theoretical foundation and an illustration of how their long-term impact can be compared. As shown below, these short-term fairness interventions may not necessarily promote long-term equity, and their impact can be sensitive to feature distributions and transitions. A small change in either can lead to contrarian results, suggesting the importance of  understanding the underlying population.

\textbf{Long-term impact on natural equality.}
When there is natural equality, an unconstrained optimal policy will result in two groups converging to the same qualification rate, thus rendering fairness constraints is unnecessary. The interesting question here is whether applying a fairness constraint can disrupt the equality. The theorem below shows that the \texttt{DP} and \texttt{EqOpt} fairness will do harm if the feature distributions are different.

\begin{theorem}\label{thm:general}
\label{thm:fair_policy_unfair}
For any feature distribution \rt{$G^s_y(x)$}
and $\forall \alpha_{\texttt{UN}}\in (0,1)$, there exist transitions $\{T_{yd}^s\}_{y,d,s}$ satisfying either Condition \ref{con:transition}\ref{con:transition_I} or Condition \ref{con:transition}\ref{con:transition_II} such that $\widehat{\alpha}_{\texttt{UN}}^a=\widehat{\alpha}_{\texttt{UN}}^b = \alpha_{\texttt{UN}}$. In this case, if $G_y^a(x)\neq G_y^b(x)$ (resp. $G_y^a(x)= G_y^b(x)$), then imposing either $\mathcal{C} = \texttt{DP}$ or \texttt{EqOpt} fair optimal policies will violate (resp. maintain) equality, i.e., $ \widehat{\alpha}_{\mathcal{C}}^a \neq \widehat{\alpha}_{\mathcal{C}}^b$ (resp. $ \widehat{\alpha}_{\mathcal{C}}^a = \widehat{\alpha}_{\mathcal{C}}^b$). 
\end{theorem}

Theorem \ref{thm:general} shows that $\forall \alpha_{\texttt{UN}}\in(0,1)$, there exists model parameters under which $\alpha_{\texttt{UN}}$ is the equilibrium and natural equality is attained. Also, natural equality is not disrupted by imposing either fairness constraint when \rt{feature distributions}
are the same across different groups (referred to as {\em demographic-invariant} below). However, imposing either constraint will lead to unequal outcomes if \rt{feature distributions}
are diverse across groups (referred to as {\em demographic-variant} below), which is more likely to happen in reality. Thus, in these natural equality cases, imposing fairness will often do harm.

\textbf{Long-term impact on natural inequality.}
Natural inequality, i.e., $\widehat{\alpha}^a_{\texttt{UN}} \neq \widehat{\alpha}^b_{\texttt{UN}}$, is more common than natural equality which only occurs under specific model parameters. This difference in qualification rates at equilibrium typically stems from the fact that either feature distributions or transitions \xr{or both} are different across different groups. Thus, below we study the impact of imposing fairness by considering these two sources of inequality separately, \xr{and we aim to examine whether fairness constraints can address the inequality caused by each}. Let \textit{disadvantaged group} be the group with a lower qualification rate at equilibrium. 

\emph{Demographic-invariant feature distribution with demographic-variant transition.}
In this case, we have
\rt{the same feature distributions but different transitions in different groups, \xr{i.e., $G_y^s=G_y^b, T^a_{yd}\neq T^b_{yd}$.}}
A real-world example is  college admission based on ACT/SAT scores: given the same qualification state, score distributions may be the same regardless of the applicant's socio-economic status, but the economically advantaged may be able to afford more investments and effort to improve their score after a rejection. 
\begin{theorem}
\label{thm:comp_dp_eqopt}
Under Condition \ref{con:transition}\ref{con:transition_I},
\texttt{DP} and \texttt{EqOpt} fairness exacerbate inequality, i.e., $|\widehat{\alpha}^a_{\mathcal{C}}-\widehat{\alpha}^b_{\mathcal{C}}| \geq |\widehat{\alpha}^a_{\texttt{UN}}-\widehat{\alpha}^b_{\texttt{UN}}| $; under Condition \ref{con:transition}\ref{con:transition_II}, \texttt{DP} and \texttt{EqOpt} fairness mitigate
inequality, i.e., $|\widehat{\alpha}^a_{\mathcal{C}}-\widehat{\alpha}^b_{\mathcal{C}}| \leq |\widehat{\alpha}^a_{\texttt{UN}}-\widehat{\alpha}^b_{\texttt{UN}}|$. Moreover, the disadvantaged group remains  disadvantaged in both cases, i.e.,  $(\widehat{\alpha}^a_{\texttt{UN}}-\widehat{\alpha}^b_{\texttt{UN}})(\widehat{\alpha}^a_{\mathcal{C}}-\widehat{\alpha}^b_{\mathcal{C}}) \geq 0$.
\end{theorem}
\begin{wrapfigure}[6]{r}{0.5\textwidth}
\vspace{-3.2em}
\centering
\includegraphics[trim=0cm 0cm 0cm 0cm,clip=true, width=0.5\textwidth]{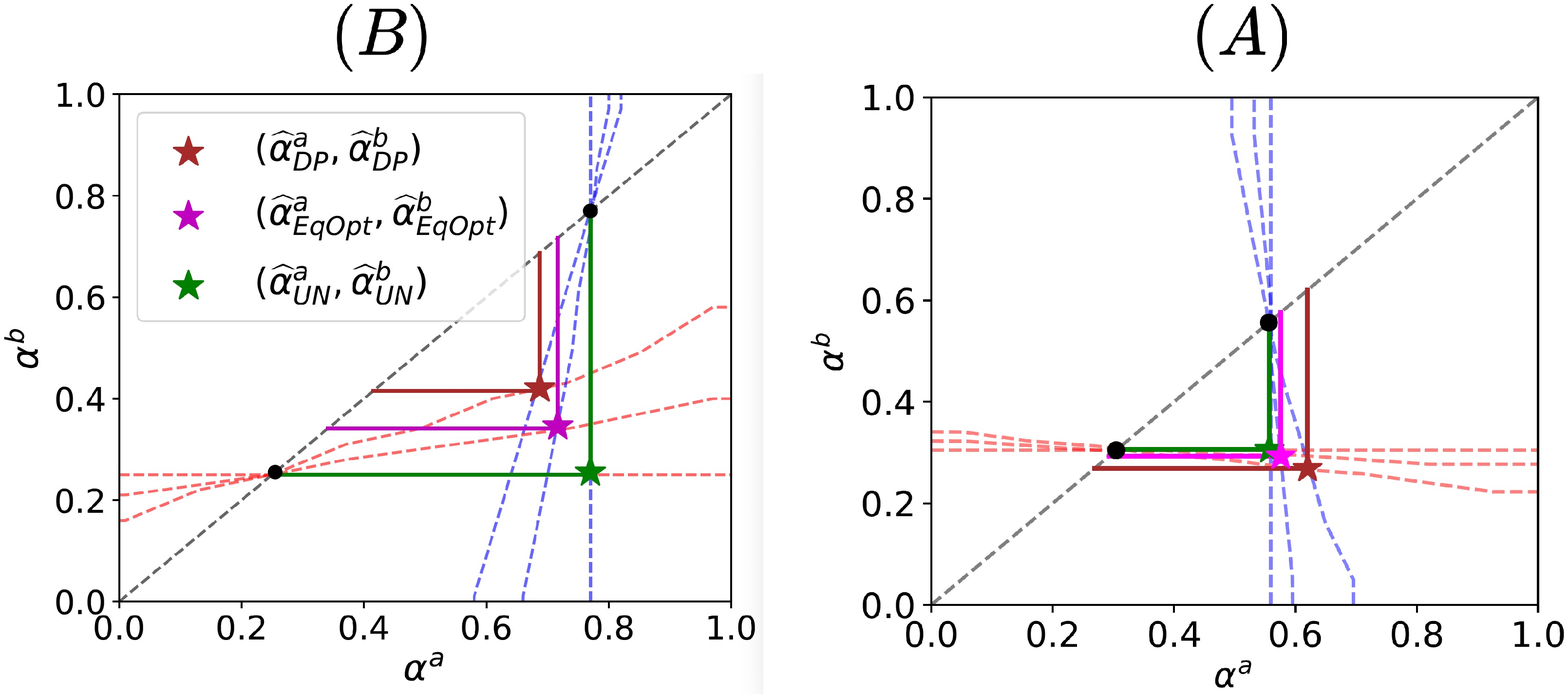}
\end{wrapfigure}
This theorem shows that imposing fairness only helps when the ``leg-up'' effect is more prominent than the ``\rt{lack of }motivation'' effect; alternatively, this suggests that when the ``\rt{lack of }motivation'' effect is dominant, imposing fairness should be accompanied by other support structure to dampen this effect (e.g., by helping those accepted to become or remain qualified).
Theorem \ref{thm:comp_dp_eqopt} is illustrated in the plot to the right, where transitions satisfy \xr{Condition \ref{con:transition}\ref{con:transition_I}-\ref{con:transition_II}} 
and $G^a_y(x) = G^b_y(x)$ is Gaussian distributed.
\xr{Each} plot includes 3 pairs of red/blue dashed curves corresponding to 3 policies (\texttt{EqOpt}, \texttt{DP}, \texttt{UN}). Points  $(\alpha^a,\alpha^b)$ on these curves satisfy $\alpha^b = g^{0b}\xr{(\alpha^a,\alpha^b)\cdot}(1-\alpha^b)+ g^{1b}\xr{(\alpha^a,\alpha^b)\cdot}\alpha^b$ and $\alpha^a = g^{0a}\xr{(\alpha^a,\alpha^b)\cdot}(1-\alpha^a) + g^{1a}\xr{(\alpha^a,\alpha^b)\cdot}\alpha^a$, respectively. Each intersection (colored star) is an equilibrium $(\widehat{\alpha}^a_{\mathcal{C}},\widehat{\alpha}^b_{\mathcal{C}})$; the length of colored segments represents $|\widehat{\alpha}^a_{\mathcal{C}}-\widehat{\alpha}^b_{\mathcal{C}}|$. The black circle is the intersection of all three blue/red curves.

\emph{Demographic-variant feature distribution with demographic-invariant transition.} In this case,
we have 
\rt{the same transitions and different feature distributions in different groups, \xr{i.e., $G_y^s\neq G_y^b, T^a_{yd}= T^b_{yd}$.}}
In the same example of college admission this is a case where the ACT/SAT scores are biased against a certain group but there is no difference in how different groups react to the decision. Here, we will focus on a class of feature distributions where those qualified have the same feature distribution regardless of group membership, while those unqualified from $\mathcal{G}_b$ are more likely to have lower features than those unqualified from $\mathcal{G}_a$. This is given in the condition below.

\begin{minipage}{0.75\textwidth}
\begin{con}\label{ass:inv_transition}
\xr{$G^s_y(x)$ is continuous in $x\in\mathbb{R}$; ${G}_1^a(x) = {G}_1^b(x), \forall x\in\mathbb{R}$; $G_0^a(x)$ and $G_0^b(x)$ satisfy strict monotone likelihood ratio property, i.e., $\frac{G_0^a(x)}{G_0^b(x)}$ is strict increasing in $x\in\mathbb{R}$.}
\end{con}
\end{minipage}
\begin{minipage}{0.22\textwidth}
\includegraphics[trim=0.5cm 0cm 1.0cm 0.6cm,clip=true, width=\textwidth]{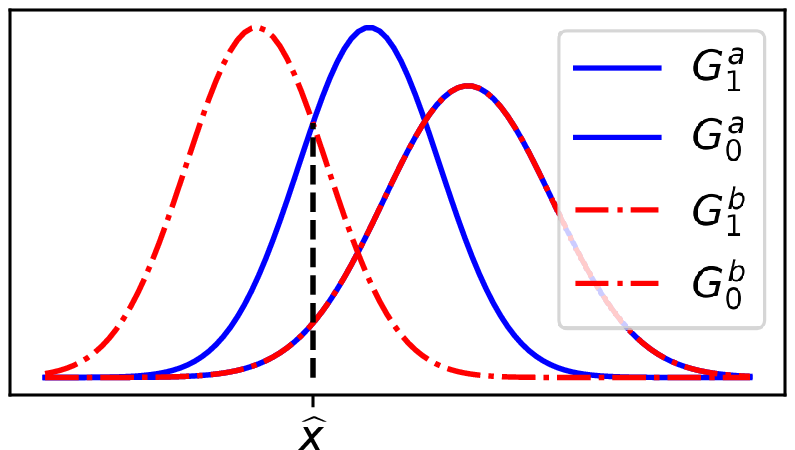}
\end{minipage}
\xr{Condition \ref{ass:inv_transition} also implies that $\int_{-\infty}^{x}{G}_0^b(z)dz\geq \int_{-\infty}^{x}{G}_0^a(z)dz, \forall x\in\mathbb{R}$. Let $\widehat{x}$ be defined such that $G_0^b(\widehat{x}) = G_0^a(\widehat{x})$ holds, which is unique. An example satisfying Condition \ref{ass:inv_transition} is shown on the right. }
\begin{theorem}\label{thm:transition}
Under Condition  \ref{con:transition}\ref{con:transition_II} and Condition \ref{ass:inv_transition}, if $ \frac{u_+}{u_-} \geq \frac{
G_0^s(\widehat{x})}{G_1^s(\widehat{x})}\frac{1-T_{10}}{T_{00}}$, we have
\begin{itemize}[noitemsep,topsep=0pt,leftmargin=*]
\item  $\widehat{\alpha}_{\texttt{UN}}^a > \widehat{\alpha}_{\texttt{UN}}^{b}$ and $\widehat{\alpha}_{\texttt{UN}}^a -  \widehat{\alpha}_{\texttt{UN}}^{b} > \widehat{\alpha}_{\texttt{EqOpt}}^a -  \widehat{\alpha}_{\texttt{EqOpt}}^{b} \geq
0$ hold, i.e., \texttt{EqOpt} fairness always mitigates inequality and the disadvantaged group $\mathcal{G}_b$ remains disadvantaged.
\item \texttt{DP} fairness may either (1) mitigate inequality, i.e., $\widehat{\alpha}_{\texttt{UN}}^a -  \widehat{\alpha}_{\texttt{UN}}^{b} > \widehat{\alpha}_{\texttt{DP}}^a -  \widehat{\alpha}_{\texttt{DP}}^{b} 
\geq 0$; or (2) flip the disadvantaged group from $\mathcal{G}_b$ to $\mathcal{G}_a$, i.e.,  $\widehat{\alpha}_{\texttt{DP}}^{b} \geq
\widehat{\alpha}_{\texttt{DP}}^a $.
\end{itemize}
\end{theorem}

Because $\mathcal{G}_a$ and $\mathcal{G}_b$ only differ in  $G^s_0(x)$, the condition in Thm \ref{thm:transition} ensures at least one group has enough unqualified people to be accepted and can be satisfied if benefit $u_+$ is sufficiently larger than cost $u_-$. We see that in this case the comparison is much more complex depending on the model parameters.  

\section{Effective interventions}
\label{sec:interventions}
\vspace{-0.5em}
As discussed, imposing static fairness constraints is not always a valid intervention in terms of its long-term impact.  In some cases it reinforces existing disparity; even when it could work, doing it right can be very hard due to its sensitivity to problem parameters. In this section, we present several alternative interventions that can be more effective in inducing more equitable outcomes or improving overall qualification rates in the long run. We shall assume that the sufficient conditions of Theorem \ref{thm:unique_equilibrium} hold 
\rt{under Assumption \ref{ass:mono-inference} and \ref{ass:fair_prob}}
so that the equilibrium is unique.

\textbf{Policy intervention.} In many instances, preserving static fairness at each time $t$ is important, for short-term violations may result in costly lawsuits \cite{target}. Proposition \ref{prop:intervention_policy} below shows that \xr{using \textit{sub-optimal} fair policies instead of the optimal ones} 
can improve overall qualification in the long run. 
\begin{proposition}\label{prop:intervention_policy}
Let $(\theta^a_{\mathcal{C}},\theta^b_{\mathcal{C}})$, $(\theta^{a'}_{\mathcal{C}},\theta^{b'}_{\mathcal{C}})$ be thresholds satisfying fairness constraint $\mathcal{C}$ under the optimal and an alternative policy, respectively. Let $(\widehat{\alpha}^{a}_{\mathcal{C}},\widehat{\alpha}^{b}_{\mathcal{C}}), (\widehat{\alpha}^{a'}_{\mathcal{C}},\widehat{\alpha}^{b'}_{\mathcal{C}})$ be the corresponding equilibrium. 
\begin{itemize}[noitemsep,topsep=0pt,leftmargin=*]
\item If $\theta^{s'}_{\mathcal{C}}(\alpha^a,\alpha^b) > \theta^s_{\mathcal{C}}(\alpha^a,\alpha^b)$, $\forall \alpha^s\in [0,1]$ under Condition \ref{con:transition}\ref{con:transition_I}, then $\widehat{\alpha}^{s'}_{\mathcal{C}}> \widehat{\alpha}^{s}_{\mathcal{C}}$, $\forall s\in\{a,b\}$;
\item  If $\theta^{s'}_{\mathcal{C}}(\alpha^a,\alpha^b) < \theta^s_{\mathcal{C}}(\alpha^a,\alpha^b)$, $\forall \alpha^s\in [0,1]$ under Condition \ref{con:transition}\ref{con:transition_II}, then $\widehat{\alpha}^{s'}_{\mathcal{C}}> \widehat{\alpha}^{s}_{\mathcal{C}}$, $\forall s\in\{a,b\}$.
\end{itemize}
\end{proposition}
Note that the sacrifice is in instantaneous utility, not necessarily in total utility in the long run (see an example in proof of Proposition \ref{prop:intervention_policy}, Appendix \ref{app:proofs}). If static fairness need not be maintained at all times, then we can employ separate policies for each group, and Proposition \ref{prop:existance_eq_polices} below shows that under certain conditions on transitions, threshold policies leading to equitable equilibrium always exist. 
\begin{proposition}\label{prop:existance_eq_polices}
Let \xr{$\mathcal{I}_s\coloneqq = \big[\frac{1-\max\{T_{11}^s,T^s_{10}\}}{\max\{T_{01}^s,T_{00}^s\}},\frac{1-\min\{T_{11}^s,T^s_{10}\}}{\min\{T_{01}^s,T_{00}^s\}} \big]$, $s\in \{a,b\}$.}
Under Condition \ref{con:transition}\ref{con:transition_I} or \ref{con:transition}\ref{con:transition_II}, if $\mathcal{I}_a\cap \mathcal{I}_b \neq \emptyset$, then $\forall \widehat{\alpha}\in \mathcal{I}_a\cap \mathcal{I}_b$, there exist threshold policies $\theta^{s}(\alpha^s)$, $\forall \alpha^s\in [0,1]$, under which $\alpha^s_t\rightarrow \widehat{\alpha},\forall s\in \{a,b\}$, i.e., equitable equilibrium is attained; if $\mathcal{I}_a\cap \mathcal{I}_b = \emptyset$, then there is no threshold policy that can result in equitable equilibrium. 
\end{proposition}
Proposition \ref{prop:existance_eq_polices} also indicates that when two groups' transitions are significantly different, manipulating policies cannot achieve equality. In this case, the following intervention can be considered.

\textbf{Transition Intervention.} Another intervention is to alter the value of transitions, e.g., by establishing support for both the accepted and rejected. Proposition \ref{prop:intervention_transition} shows that the qualification rate $\widehat{\alpha}^s$ at equilibrium can be improved by enhancing individuals' ability to maintain/improve qualification, which is consistent with the empirical findings in loan repayment \citep{paxton2000modeling,gotham1998race,homonoff2019does} and labor markets \citep{fehr2009behavioral}. 

\begin{proposition}\label{prop:intervention_transition}
$\forall s\in \{a,b\}$, increasing any transition probability $T^s_{yd}$, $d \in \{0,1\}$, $y \in \{0,1\}$ always increases the value of equilibrium qualification rates $\widehat{\alpha}^s$.
\end{proposition}

\section{Experiments}\label{sec:exp}
\begin{wrapfigure}[19]{r}{0.62\textwidth}
\vspace{-1.5 em}
\centering
\begin{subfigure}[b]{0.3\textwidth}
\centering
\includegraphics[trim=0.05cm 0cm 0.5cm 0.5cm,clip=true, width=\textwidth]{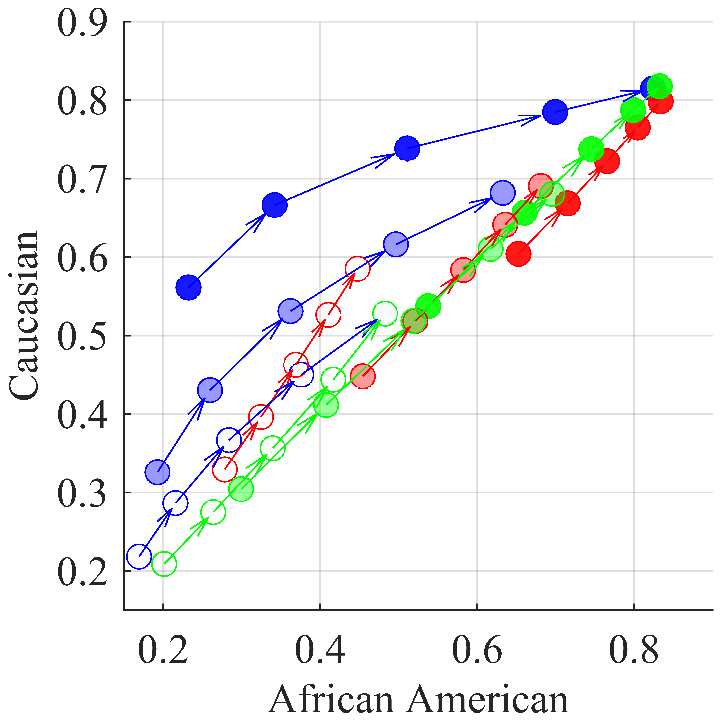}
\caption{D-invariant transitions}\label{fig:sm_t}
\end{subfigure}
\hfill
\begin{subfigure}[b]{0.3\textwidth}
\centering
\includegraphics[trim=0.05cm 0cm 0.5cm 0.5cm,clip=true, width=\textwidth]{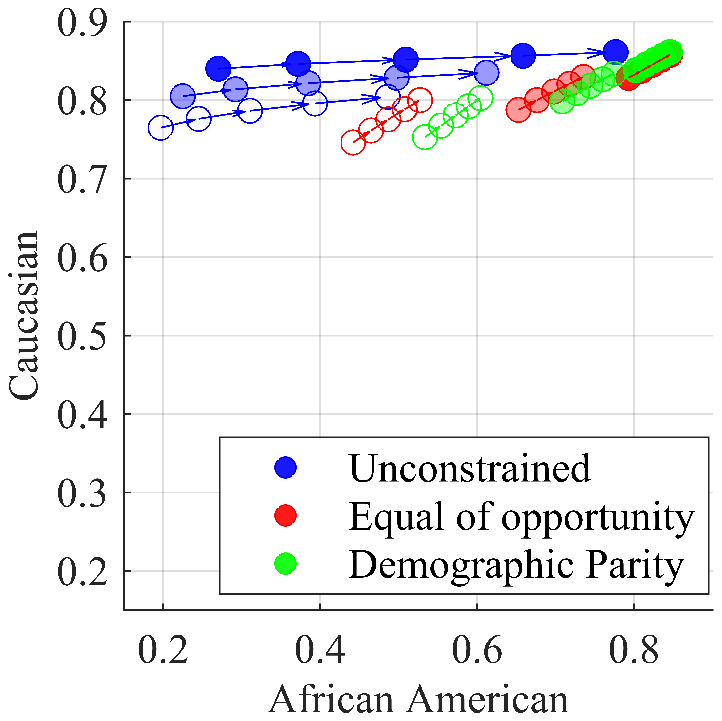}
\caption{
D-variant transitions}\label{fig:diff_t}
\end{subfigure}
\caption{Results on the FICO dataset: Points are the equilibria of repayment rates in  $\mathcal{G}_{AA}, \mathcal{G}_{C}$ under Condition \ref{con:transition}\ref{con:transition_II} with different transitions. Arrows indicate the direction of increasing $T_{01}^s$; a more transparent point represents the smaller value of $T_{10}^s$. In panel \subref{fig:sm_t}, $T^{AA}_{yd} = T^{C}_{yd}$, while in panel \subref{fig:diff_t}, $T^{AA}_{yd} < T^{C}_{yd}$.}
\label{fig:fico}
\end{wrapfigure} 
We conducted experiments on both Gaussian synthetic datasets and real-world datasets. We present synthetic data experiments in Appendix \ref{app:exp} and the results using real-world datasets here. These are static, one-shot datasets, which we use to create a simulated dynamic setting as detailed below.

\textbf{Loan repayment study.}
We use the  FICO score dataset \cite{reserve2007report} to study the long-term impact of fairness constraints \texttt{EqOpt} and \texttt{DP} and other interventions on loan repayment rates in the Caucasian group $\mathcal{G}_C$ and the African American group $\mathcal{G}_{AA}$. With the pre-processed data in \citep{hardt2016equality}, we simulate a dataset with loan repayment records and credits scores. We first compute the initial qualification (loan repayment) rates and estimate the feature distributions $G^s_y(x)$ with beta distributions based on the simulated data. Then, we compute the optimal \texttt{UN}, \texttt{EqOpt}, \texttt{DP} threshold according to Eqn. \eqref{eq:opt_fair_policy}. Consequently, with the dynamics \eqref{eq:dynamics}, we update the qualification rates in both groups. This process proceeds and qualification rates in both groups change over time.

Our results show consistent findings with studies in loan repayment literature \cite{gotham1998race,paxton2000modeling}. Specifically, \cite{paxton2000modeling} studied the loan repayment in group lending and pointed out that in practice effective training and leadership among the groups who were issued loans can increase their willingness to pay and improve the group repayment rate. Similar conclusion is also suggested by \cite{gotham1998race}. In our model, these interventions can be regarded as stimulating transitions (i.e., $T_{y1}$) to improve the future repayment rates. And the scenarios under such intervention would satisfy Condition \ref{con:transition}\ref{con:transition_II}. Fig. \ref{fig:fico} shows that under  Condition \ref{con:transition}\ref{con:transition_II}, increasing the transition $T_{01}^s$ always increases qualification rates, and \texttt{DP} in general can result in a more equitable equilibrium than \texttt{EqOpt}. Fig. \ref{fig:sm_t} shows that in Demographic-invariant (D-invariant) transition cases ($T^{AA}_{yd} = T^{C}_{yd}$):
(1) $\mathcal{G}_{AA}$ always remains as disadvantaged group;
(2) when $T_{10}$ is small, the inequality under \texttt{UN} optimal policies is small and the intervention on $T_{01}$ only has minor effects on equality; when $T_{10}$ is large (darker blue points), varying $T_{01}$ can affect disparity significantly;
(3) imposing \texttt{DP} attains equitable equilibria in general, which is robust to transitions and consistent with the conclusion in \citep{socialequality2019}; 
(4) when $T_{10}$ is small, imposing \texttt{EqOpt} exacerbates inequality as $T_{01}$ increases; while $T_{10}$ is sufficient large, equality can be attained and robust to transitions. In Fig. \ref{fig:diff_t}, it shows that in D-variant transition cases, by setting $T^{AA}_{yd} < T^{C}_{yd}$, the inequality between $\mathcal{G}_{AA}$ and $\mathcal{G}_{C}$ further gets reinforced.
In summary, the effectiveness of such intervention (increasing $T_{01}$) on promoting equality highly depends on the value of $T_{10}$ and policies. 

\begin{wraptable}[13]{r}{7.8cm}
\vspace{-1em}
\centering
\caption{osi$^*$/osi$_H$/osi$_L$ is the percentage that oscillation occurs among 125 set of different transitions under policy $\texttt{UN}^*$/$\texttt{UN}_{\theta_H}$/$\texttt{UN}_{\theta_L}$. Among transitions that lead to stable equilibrium, Col 2/Col 3 shows the percentage that $\texttt{UN}_{\theta_H}$/  $\texttt{UN}_{\theta_L}$ results in lower recidivism compared with $\texttt{UN}^*$.} 
\label{tbl:compas_update}
\begin{tabular}{llllll}
\hline
& $\widehat{\alpha}_{\theta_{H}} < \widehat{\alpha}^*$ & $\widehat{\alpha}_{\theta_{L}} < \widehat{\alpha}^*$ & osi$^*$& osi$_H$& osi$_L$ \\ \hline
$A$  & 0&	1&	0.29&	0.12&	0.36\\
$B$  & 0.99 &	0.01 &	0&	0&	0 \\
$C$  & 0.37 &	0.28 &	0&	0&	0\\
$D$  & 0.79&	0.63&	0.06&	0&	0.13\\\hline
\end{tabular}   
\end{wraptable}
\textbf{The COMPAS data.}
Our second set of experiments is conducted on a multivariate recidivism prediction dataset from Correctional Offender Management Profiling for Alternative Sanctions (COMPAS) \citep{angwin2016machine}. We again use this static and high-dimensional dataset to create a simulated \rt{decision-making}
process as the FICO experiments. Specifically, from the raw data we calculate the initial qualification (recidivism) rate and train optimal classifier using a logistic regression model, based on which recidivism rate is updated according to Eqn. \eqref{eq:dynamics} under a given set of transitions. 
\rt{In the context of recidivism prediction, we consider all the \xr{possible types }
of transitions \xr{under an unconstrained policy}, i.e., \xr{transitions satisfying} conditions \ref{con:transition}\ref{con:transition_I}-(D).} The classifier decision here corresponds to incarceration based on predicted likelihood of recidivism: the higher the predicted recidivism, the more likely an incarceration decision. 
In subsequent time steps, the data is re-sampled from the raw data proportional to the updated recidivism rates. 
This process repeats and the group recidivism rates change over time. 
\rt{Our results here primarily serve to highlight the complexity in such a decision-making system. In particular, \xr{we see that an equilibrium may not exist and under some transitions the qualification rate may oscillate. Specifically,} Table \ref{tbl:compas_update} shows that
Prop. 1 holds under Condition \ref{con:transition}\ref{con:transition_I}-\ref{con:transition_II}; \xr{there is} no oscillation under \xr{Condition} \ref{con:transition}\ref{con:transition_II}-(C); \xr{under Condition 1(C)-(D), there is more uncertainty }
which is discussed in Appendix \ref{app:discussion}. \xr{More results on the oscillation can be found in Appendix \ref{app:exp}.}
}

\vspace{-5pt}
\section{Conclusion}
\vspace{-5pt}
In this paper, we studied the long-term impact of fairness constraints (e.g., \texttt{DP} and \texttt{EqOpt}) on group qualification rates. 
By casting the problem in a POMDP framework, we conducted equilibrium analysis. Specifically, we first identified sufficient conditions for the existence and uniqueness of equilibrium, under which we compared different fairness constraints regarding their long-term impact. Our findings show that the same fairness constraint can have opposite impact depending on the underlying problem scenarios, which highlights the importance of understanding real-world dynamics in decision making systems. 
Our experiments on real-world static datasets with simulated dynamics also show 
\xr{that our framework can be used to facilitate} social science studies. Our analysis has focused on scenarios with a unique equilibrium; scenarios with multiple equilibria or oscillating states remain an interesting direction of future research. 

\newpage
\section*{Acknowledgement}
X. Zhang, Y. Liu and M. Liu have been supported by the NSF under grants CNS-1616575, CNS-1646019, CNS-1739517, IIS-2007951, and by the ARO under contract W911NF1810208. 
R. Tu would like to acknowledge the funding support of the Swedish e-Science Research Centre and the material suggestion regarding the social impact of polices given by Yating Zhang. Part of the work was done when R. Tu was a visiting student in Microsoft Research, Cambridge, and he would like to acknowledge Microsoft’s travel support. K. Zhang would like to acknowledge the support by the United States Air Force under Contract No. FA8650-17-C-7715. 

\section*{Broader Impact}
In this paper, we focus on the (un)fairness issue that arises in automated decision-making systems and aim to understand the long-term impact of algorithmic (fair) decisions on the well-being of different sub-groups in a population. Our partially observed sequential decision making framework is  applicable to a wide range of domains (e.g., lending, recruitment, admission, criminal justice, etc.). By conducting an equilibrium analysis and evaluating the long-term impact of different fairness criteria, our results provide a theoretical foundation that can help answer questions such as whether/when imposing short-term fairness constraints are effective in promoting long-term equality. 

First of all, our results can help policymakers (e.g., companies, banks, governments, etc.) in their decision making process by highlighting the potential pitfalls of commonly used static fairness criteria and providing guidance on how to design effective interventions that can avoid such unintended consequences and result in positive long-term societal
impacts. 

Secondly, our results may be useful to research in fields outside of the computer science community. The experiments on static real-world datasets have shown consistent findings with literature in social sciences \citep{mallett2018disproportionate,gotham1998race,paxton2000modeling}. Although these empirical results are obtained using simulated dynamics due to a lack of real datasets, they may provide  insights and theoretical supports for research in other fields.    

Lastly, while this work is limited to binary decisions, the main take-away can be applied in other applications such as computer vision, natural language processing, etc., using more complicated classifiers such as DNN. We hope that our work will encourage researchers in these domains to similarly consider  discrimination risks when developing techniques, and raise awareness that static fairness constraint may not suffice and long-term fairness cannot be designed in a vacuum without considering the human element. We thus emphasize the importance of performing real-time measurements and developing proper fair classifiers from dynamic datasets.   

Having mentioned the potential positive impact of our work, we also want to point out the limitations in our model and analysis. Firstly, in this work we use a set of transitions $T^s_{yd}$ to capture individuals' abilities to improve/maintain future qualifications, and our analysis and conclusions rely on this set of values.  In practice, however, these quantities can be extremely hard to measure due to the complexity of human behaviors and environmental factors.  In addition, as we have noted in the paper, in some cases the conclusion can be highly sensitive to minor changes in these transitions. 
Secondly, our theoretical results have focused on scenarios with a unique equilibrium, while in practice the situation can be much more complicated (multiple equilibria or no equilibria), as demonstrated by the oscillations we see in the COMPAS simulation study. Thus, it is worthwhile for future work to consider these more complex cases. Lastly, due to the lack of dynamic datasets, our experiments are performed over static real-world datasets with simulated dynamics. Thus, an accurate model of real-world dynamics is needed when deploying our method for practical decision making.

\newpage
\bibliography{ref}

\begin{thebibliography}{47}
\providecommand{\natexlab}[1]{#1}
\providecommand{\url}[1]{\texttt{#1}}
\expandafter\ifx\csname urlstyle\endcsname\relax
  \providecommand{\doi}[1]{doi: #1}\else
  \providecommand{\doi}{doi: \begingroup \urlstyle{rm}\Url}\fi

\bibitem[tar(2015)]{target}
Target corporation to pay \$2.8 million to resolve eeoc discrimination finding.
\newblock In \emph{U.S. Equal Employment Opportunity Commission}, 2015.
\newblock
  \url{https://www.eeoc.gov/newsroom/target-corporation-pay-28-million-resolve-eeoc-discrimination-finding}.

\bibitem[Aneja and Avenancio-Le{\'o}n(2019)]{aneja2019no}
A.~P. Aneja and C.~F. Avenancio-Le{\'o}n.
\newblock No credit for time served? incarceration and credit-driven crime
  cycles.
\newblock 2019.

\bibitem[Angwin et~al.(2016)Angwin, Larson, Mattu, and
  Kirchner]{angwin2016machine}
J.~Angwin, J.~Larson, S.~Mattu, and L.~Kirchner.
\newblock Machine bias: There’s software used across the country to predict
  future criminals.
\newblock \emph{And it’s biased against blacks. ProPublica}, 23, 2016.

\bibitem[Auer et~al.(2002)Auer, Cesa-Bianchi, and Fischer]{auer2002finite}
P.~Auer, N.~Cesa-Bianchi, and P.~Fischer.
\newblock Finite-time analysis of the multiarmed bandit problem.
\newblock \emph{Machine learning}, 47\penalty0 (2-3):\penalty0 235--256, 2002.

\bibitem[Barocas et~al.(2019)Barocas, Hardt, and
  Narayanan]{barocas-hardt-narayanan}
S.~Barocas, M.~Hardt, and A.~Narayanan.
\newblock \emph{Fairness and Machine Learning}.
\newblock fairmlbook.org, 2019.
\newblock \url{http://www.fairmlbook.org}.

\bibitem[Bechavod et~al.(2019)Bechavod, Ligett, Roth, Waggoner, and
  Wu]{bechavod2019equal}
Y.~Bechavod, K.~Ligett, A.~Roth, B.~Waggoner, and S.~Z. Wu.
\newblock Equal opportunity in online classification with partial feedback.
\newblock In \emph{Advances in Neural Information Processing Systems 32}, pages
  8972--8982. 2019.

\bibitem[Chaney et~al.(2018)Chaney, Stewart, and
  Engelhardt]{chaney2018algorithmic}
A.~J. Chaney, B.~M. Stewart, and B.~E. Engelhardt.
\newblock How algorithmic confounding in recommendation systems increases
  homogeneity and decreases utility.
\newblock In \emph{Proceedings of the 12th ACM Conference on Recommender
  Systems}, pages 224--232. ACM, 2018.

\bibitem[Chen et~al.(2019)Chen, Cuellar, Luo, Modi, Nemlekar, and
  Nikolaidis]{chen2019fair}
Y.~Chen, A.~Cuellar, H.~Luo, J.~Modi, H.~Nemlekar, and S.~Nikolaidis.
\newblock Fair contextual multi-armed bandits: Theory and experiments.
\newblock \emph{arXiv preprint arXiv:1912.08055}, 2019.

\bibitem[Corbett-Davies et~al.(2017)Corbett-Davies, Pierson, Feller, Goel, and
  Huq]{corbett2017algorithmic}
S.~Corbett-Davies, E.~Pierson, A.~Feller, S.~Goel, and A.~Huq.
\newblock Algorithmic decision making and the cost of fairness.
\newblock In \emph{Proceedings of the 23rd acm sigkdd international conference
  on knowledge discovery and data mining}, pages 797--806, 2017.

\bibitem[D'Amour et~al.(2020)D'Amour, Srinivasan, Atwood, Baljekar, Sculley,
  and Halpern]{d2020fairness}
A.~D'Amour, H.~Srinivasan, J.~Atwood, P.~Baljekar, D.~Sculley, and Y.~Halpern.
\newblock Fairness is not static: deeper understanding of long term fairness
  via simulation studies.
\newblock In \emph{Proceedings of the 2020 Conference on Fairness,
  Accountability, and Transparency}, pages 525--534, 2020.

\bibitem[Dimitrakakis et~al.(2019)Dimitrakakis, Liu, Parkes, and
  Radanovic]{dimitrakakis2019bayesian}
C.~Dimitrakakis, Y.~Liu, D.~Parkes, and G.~Radanovic.
\newblock Bayesian fairness.
\newblock In \emph{AAAI}, 2019.

\bibitem[Ensign et~al.(2018{\natexlab{a}})Ensign, Friedler, Neville,
  Scheidegger, and Venkatasubramanian]{ensign2018runaway}
D.~Ensign, S.~A. Friedler, S.~Neville, C.~Scheidegger, and
  S.~Venkatasubramanian.
\newblock Runaway feedback loops in predictive policing.
\newblock In \emph{Conference of Fairness, Accountability, and Transparency},
  2018{\natexlab{a}}.

\bibitem[Ensign et~al.(2018{\natexlab{b}})Ensign, Sorelle, Scott, Carlos, and
  Suresh]{ensign2018decision}
D.~Ensign, F.~Sorelle, N.~Scott, S.~Carlos, and V.~Suresh.
\newblock Decision making with limited feedback.
\newblock In \emph{Algorithmic Learning Theory}, pages 359--367,
  2018{\natexlab{b}}.

\bibitem[Fehr et~al.(2009)Fehr, Goette, and Zehnder]{fehr2009behavioral}
E.~Fehr, L.~Goette, and C.~Zehnder.
\newblock A behavioral account of the labor market: The role of fairness
  concerns.
\newblock \emph{Annu. Rev. Econ.}, 1\penalty0 (1):\penalty0 355--384, 2009.

\bibitem[Fuster et~al.(2018)Fuster, Goldsmith-Pinkham, Ramadorai, and
  Walther]{fuster2018predictably}
A.~Fuster, P.~Goldsmith-Pinkham, T.~Ramadorai, and A.~Walther.
\newblock Predictably unequal? the effects of machine learning on credit
  markets.
\newblock \emph{The Effects of Machine Learning on Credit Markets}, 2018.

\bibitem[Gillen et~al.(2018)Gillen, Jung, Kearns, and Roth]{gillen2018online}
S.~Gillen, C.~Jung, M.~Kearns, and A.~Roth.
\newblock Online learning with an unknown fairness metric.
\newblock In \emph{Advances in Neural Information Processing Systems}, pages
  2600--2609, 2018.

\bibitem[Gotham(1998)]{gotham1998race}
K.~F. Gotham.
\newblock Race, mortgage lending and loan rejections in a us city.
\newblock \emph{Sociological Focus}, 31\penalty0 (4):\penalty0 391--405, 1998.

\bibitem[Hardt et~al.(2016)Hardt, Price, and Srebro]{hardt2016equality}
M.~Hardt, E.~Price, and N.~Srebro.
\newblock Equality of opportunity in supervised learning.
\newblock In \emph{Advances in neural information processing systems}, pages
  3315--3323, 2016.

\bibitem[Hashimoto et~al.(2018)Hashimoto, Srivastava, Namkoong, and
  Liang]{pmlr-v80-hashimoto18a}
T.~Hashimoto, M.~Srivastava, H.~Namkoong, and P.~Liang.
\newblock Fairness without demographics in repeated loss minimization.
\newblock In J.~Dy and A.~Krause, editors, \emph{Proceedings of the 35th
  International Conference on Machine Learning}, volume~80 of \emph{Proceedings
  of Machine Learning Research}, pages 1929--1938. PMLR, 2018.

\bibitem[Heidari and Krause(2018)]{heidari2018preventing}
H.~Heidari and A.~Krause.
\newblock Preventing disparate treatment in sequential decision making.
\newblock In \emph{Proceedings of the 27th International Joint Conference on
  Artificial Intelligence}, pages 2248--2254, 2018.

\bibitem[Heidari et~al.(2019)Heidari, Nanda, and Gummadi]{heidari2019long}
H.~Heidari, V.~Nanda, and K.~Gummadi.
\newblock On the long-term impact of algorithmic decision policies: Effort
  unfairness and feature segregation through social learning.
\newblock In \emph{International Conference on Machine Learning}, pages
  2692--2701, 2019.

\bibitem[Homonoff et~al.()Homonoff, O'Brien, and Sussman]{homonoff2019does}
T.~Homonoff, R.~O'Brien, and A.~B. Sussman.
\newblock Does knowing your fico score change financial behavior? evidence from
  a field experiment with student loan borrowers.
\newblock \emph{Review of Economics and Statistics}, pages 1--45.

\bibitem[Hu et~al.(2019)Hu, Immorlica, and Vaughan]{hu2019disparate}
L.~Hu, N.~Immorlica, and J.~W. Vaughan.
\newblock The disparate effects of strategic manipulation.
\newblock In \emph{Proceedings of the Conference on Fairness, Accountability,
  and Transparency}, pages 259--268, 2019.

\bibitem[Jabbari et~al.(2017)Jabbari, Joseph, Kearns, Morgenstern, and
  Roth]{jabbari2017fairness}
S.~Jabbari, M.~Joseph, M.~Kearns, J.~Morgenstern, and A.~Roth.
\newblock Fairness in reinforcement learning.
\newblock In \emph{Proceedings of the 34th International Conference on Machine
  Learning-Volume 70}, pages 1617--1626. JMLR. org, 2017.

\bibitem[James(1978)]{james1978estimation}
I.~James.
\newblock Estimation of the mixing proportion in a mixture of two normal
  distributions from simple, rapid measurements.
\newblock \emph{Biometrics}, pages 265--275, 1978.

\bibitem[Joseph et~al.(2016)Joseph, Kearns, Morgenstern, and
  Roth]{joseph2016fairness}
M.~Joseph, M.~Kearns, J.~H. Morgenstern, and A.~Roth.
\newblock Fairness in learning: Classic and contextual bandits.
\newblock In \emph{Advances in Neural Information Processing Systems}, pages
  325--333, 2016.

\bibitem[Joseph et~al.(2018)Joseph, Kearns, Morgenstern, Neel, and
  Roth]{joseph2018meritocratic}
M.~Joseph, M.~Kearns, J.~Morgenstern, S.~Neel, and A.~Roth.
\newblock Meritocratic fairness for infinite and contextual bandits.
\newblock In \emph{Proceedings of the 2018 AAAI/ACM Conference on AI, Ethics,
  and Society}, pages 158--163. ACM, 2018.

\bibitem[Kallus and Zhou(2018)]{kallus2018residual}
N.~Kallus and A.~Zhou.
\newblock Residual unfairness in fair machine learning from prejudiced data.
\newblock In \emph{Proceedings of the 35th International Conference on Machine
  Learning}, 2018.

\bibitem[Kannan et~al.(2019)Kannan, Roth, and Ziani]{kannan2019downstream}
S.~Kannan, A.~Roth, and J.~Ziani.
\newblock Downstream effects of affirmative action.
\newblock In \emph{Proceedings of the Conference on Fairness, Accountability,
  and Transparency}, pages 240--248. ACM, 2019.

\bibitem[Khajehnejad et~al.(2019)Khajehnejad, Tabibian, Sch{\"o}lkopf, Singla,
  and Gomez-Rodriguez]{khajehnejad2019optimal}
M.~Khajehnejad, B.~Tabibian, B.~Sch{\"o}lkopf, A.~Singla, and
  M.~Gomez-Rodriguez.
\newblock Optimal decision making under strategic behavior.
\newblock \emph{arXiv preprint arXiv:1905.09239}, 2019.

\bibitem[Kilbertus et~al.(2019)Kilbertus, Gomez-Rodriguez, Sch{\"o}lkopf,
  Muandet, and Valera]{kilbertus2019improving}
N.~Kilbertus, M.~Gomez-Rodriguez, B.~Sch{\"o}lkopf, K.~Muandet, and I.~Valera.
\newblock Improving consequential decision making under imperfect predictions.
\newblock \emph{arXiv preprint arXiv:1902.02979}, 2019.

\bibitem[Li et~al.(2019)Li, Liu, and Ji]{li2019combinatorial}
F.~Li, J.~Liu, and B.~Ji.
\newblock Combinatorial sleeping bandits with fairness constraints.
\newblock In \emph{IEEE INFOCOM 2019-IEEE Conference on Computer
  Communications}, pages 1702--1710. IEEE, 2019.

\bibitem[Liu et~al.(2018)Liu, Dean, Rolf, Simchowitz, and
  Hardt]{pmlr-v80-liu18c}
L.~T. Liu, S.~Dean, E.~Rolf, M.~Simchowitz, and M.~Hardt.
\newblock Delayed impact of fair machine learning.
\newblock In J.~Dy and A.~Krause, editors, \emph{Proceedings of the 35th
  International Conference on Machine Learning}, volume~80 of \emph{Proceedings
  of Machine Learning Research}, pages 3150--3158. PMLR, 2018.

\bibitem[Liu et~al.(2020)Liu, Wilson, Haghtalab, Kalai, Borgs, and
  Chayes]{liu2019disparate}
L.~T. Liu, A.~Wilson, N.~Haghtalab, A.~T. Kalai, C.~Borgs, and J.~Chayes.
\newblock The disparate equilibria of algorithmic decision making when
  individuals invest rationally.
\newblock In \emph{Proceedings of the 2020 Conference on Fairness,
  Accountability, and Transparency}, pages 381--391, 2020.

\bibitem[Liu et~al.(2017)Liu, Radanovic, Dimitrakakis, Mandal, and
  Parkes]{liu2017calibrated}
Y.~Liu, G.~Radanovic, C.~Dimitrakakis, D.~Mandal, and D.~C. Parkes.
\newblock Calibrated fairness in bandits.
\newblock \emph{arXiv preprint arXiv:1707.01875}, 2017.

\bibitem[Mallett(2018)]{mallett2018disproportionate}
C.~A. Mallett.
\newblock Disproportionate minority contact in juvenile justice: Today’s, and
  yesterdays, problems.
\newblock \emph{Criminal justice studies}, 31\penalty0 (3):\penalty0 230--248,
  2018.

\bibitem[Mouzannar et~al.(2019)Mouzannar, Ohannessian, and
  Srebro]{socialequality2019}
H.~Mouzannar, M.~I. Ohannessian, and N.~Srebro.
\newblock From fair decision making to social equality.
\newblock In \emph{Proceedings of the Conference on Fairness, Accountability,
  and Transparency}, FAT* ’19, page 359–368, 2019.

\bibitem[Nabi et~al.(2019)Nabi, Malinsky, and Shpitser]{nabi2019learning}
R.~Nabi, D.~Malinsky, and I.~Shpitser.
\newblock Learning optimal fair policies.
\newblock \emph{Proceedings of machine learning research}, 97:\penalty0 4674,
  2019.

\bibitem[Patil et~al.(2019)Patil, Ghalme, Nair, and
  Narahari]{patil2019achieving}
V.~Patil, G.~Ghalme, V.~Nair, and Y.~Narahari.
\newblock Achieving fairness in the stochastic multi-armed bandit problem.
\newblock \emph{arXiv preprint arXiv:1907.10516}, 2019.

\bibitem[Patra and Sen(2016)]{patra2016estimation}
R.~K. Patra and B.~Sen.
\newblock Estimation of a two-component mixture model with applications to
  multiple testing.
\newblock \emph{Journal of the Royal Statistical Society: Series B (Statistical
  Methodology)}, 78\penalty0 (4):\penalty0 869--893, 2016.

\bibitem[Paxton et~al.(2000)Paxton, Graham, and Thraen]{paxton2000modeling}
J.~Paxton, D.~Graham, and C.~Thraen.
\newblock Modeling group loan repayment behavior: New insights from burkina
  faso.
\newblock \emph{Economic Development and cultural change}, 48\penalty0
  (3):\penalty0 639--655, 2000.

\bibitem[Reserve(2007)]{reserve2007report}
U.~F. Reserve.
\newblock Report to the congress on credit scoring and its effects on the
  availability and affordability of credit.
\newblock \emph{Board of Governors of the Federal Reserve System}, 2007.

\bibitem[Tang et~al.(2020)Tang, Ho, and Liu]{tang2020fair}
W.~Tang, C.-J. Ho, and Y.~Liu.
\newblock Fair bandit learning with delayed impact of actions.
\newblock \emph{arXiv preprint arXiv:2002.10316}, 2020.

\bibitem[Williams and Kolter(2019)]{williams2019dynamic}
J.~Williams and J.~Z. Kolter.
\newblock Dynamic modeling and equilibria in fair decision making.
\newblock \emph{arXiv preprint arXiv:1911.06837}, 2019.

\bibitem[Zhang and Liu(2020)]{zhang2020fairness}
X.~Zhang and M.~Liu.
\newblock Fairness in learning-based sequential decision algorithms: A survey.
\newblock \emph{arXiv preprint arXiv:2001.04861}, 2020.

\bibitem[Zhang et~al.(2019{\natexlab{a}})Zhang, Khalili, and Liu]{xueru1}
X.~Zhang, M.~M. Khalili, and M.~Liu.
\newblock Long-term impacts of fair machine learning.
\newblock \emph{Ergonomics in Design}, 2019{\natexlab{a}}.
\newblock \doi{10.1177/1064804619884160}.

\bibitem[Zhang et~al.(2019{\natexlab{b}})Zhang, Khalili, Tekin, and
  Liu]{zhang2019group}
X.~Zhang, M.~M. Khalili, C.~Tekin, and M.~Liu.
\newblock Group retention when using machine learning in sequential decision
  making: the interplay between user dynamics and fairness.
\newblock In \emph{Advances in Neural Information Processing Systems}, pages
  15243--15252, 2019{\natexlab{b}}.

\end{thebibliography}
\newpage
\appendix
\section{Notations}
\begin{center}
\begin{tabular}{  m{2.2em} | p{12.5cm}  } 
\hline
$\mathcal{G}_s$ & demographic group, $s\in \{a,b\}$  \\ 
\hline
$X_t$& feature at $t$, $X\in \mathbb{R}^d$  \\
\hline
$Y_t$& true qualification state at $t$, $Y_t \in \{0,1\}$  \\
 \hline
 $S$& sensitive attribute $S \in \{a,b\}$  \\
\hline 
$p_s$ & group proportion of $s$, i.e., 
$p_s = \mathbb{P}(S=s)$\\
\hline
$D_t$& institute's decision at $t$, $D_t\in \{0,1\}$ \\
\hline 
$\pi_t^s(x)$ & policy for $\mathcal{G}_s$ at $t$, i.e., $\pi^s_t(x) = 
\mathbb{P}(D_t = 1\mid X_t = x, S = s)$\\
\hline
$G_y^s(x)$ & feature distribution of unqualified ($y=0$) or qualified ($y=1$) people from 
$\mathcal{G}_s$
, i.e., 
$\mathbb{P}(X_t=x\mid Y_t=y,S=s)$ \\ 
\hline
$\mathbb{G}_y^s(x)$ & CDF of $G_y^s(x)$, i.e., $\mathbb{G}_y^s(x) = \int_{-\infty}^x {G}_y^s(z)dx$\\
\hline 
$\mathcal{P}^s_\mathcal{C}(x)$ & a probability distribution over $X_t$ that specifies the fairness metric $\mathcal{C}$
\\
\hline 
$\alpha_t^s$ & qualification rate of 
$\mathcal{G}_s$
at $t$, i.e., 
$\mathbb{P}(Y_t = 1\mid S = s)$\\
\hline 
$\gamma^s_t(x)$ & qualification profile of $\mathcal{G}_s$ at $t$, i.e., 
$\mathbb{P}(Y_t = 1 \mid X_t = x, S = s)$\\
\hline 
$T_{yd}^s$ & transition probability of $\mathcal{G}_s$,
i.e., 
$\mathbb{P}(Y_{t+1} = 1\mid Y_t = y, D_t = d, S = s)$\\
\hline 
$u_+$ & benefit the institute gains by accepting a qualified individual \\
\hline 
$u_-$ & cost incurred to the institute by accepting an unqualified individual \\
\hline
$\theta^s_{\mathcal{C}}$ & threshold in a threshold policy for 
$\mathcal{G}_s$ under constraint $\mathcal{C}$, i.e., $\pi^s_t(x) = \textbf{1}(x\geq \theta^s_{\mathcal{C}})$\\

\hline 
$\widehat{\alpha}^s_{\mathcal{C}}$ & qualification rate of 
$\mathcal{G}_s$
at the equilibrium under policy with constraint $\mathcal{C} \in \{\texttt{UN}, \texttt{DP}, \texttt{EqOpt}\}$\\
\end{tabular}
\end{center}
\section{Additional results on experiments} \label{app:exp}
\paragraph{Gaussian distributed synthetic data.}
We first verify the conclusions in Section \ref{sec:evolution} and \ref{sec:impacts} using the synthetic data, where $X_t\mid Y_t = y,S=s\sim \mathcal{N}(\mu^s_y,(\sigma^{s})^2)$. 

In Section \ref{sec:evolution}, Figure \ref{fig:trajectory} illustrates sample paths of $\{(\alpha^a_t,\alpha^b_t)\}_t$ under \texttt{EqOpt}, \texttt{DP}, \texttt{UN} optimal policies. The specific parameters are as follows: $[\mu^a_0,\mu^a_1,\mu^b_0,\mu^b_1] = [-5,5,-5,5]$, $[\sigma^a,\sigma^b] = [5,5]$, $\frac{u_+}{u_-} = 1$, $p_a = p_b = 0.5$, $[T_{00}^a, T_{01}^a,T_{10}^a,T_{11}^a] = [0.4,0.5,0.5,0.9]$, $[T_{00}^b, T_{01}^b,T_{10}^b,T_{11}^b] = [0.1,0.5,0.5,0.7]$. 

Table \ref{table:inv_dist} and \ref{table:inv_tran} illustrate the impacts of \texttt{EqOpt} and \texttt{DP} fairness on the equilibrium, where each column shows the value of $\widehat{\alpha}^a_{\mathcal{C}} - \widehat{\alpha}^b_{\mathcal{C}}$ when $\mathcal{C} = \texttt{UN}, \texttt{EqOpt}, \texttt{DP}$ under different sets of parameters. Specifically, in Table \ref{table:inv_dist}, $p_a = p_b = 0.5$, $\frac{u_+}{u_-} = 1$, $[\mu^s_0,\mu^s_1,\sigma^s] = [-5, 5, 5]$,$\forall s\in \{a,b\}$ and transitions satisfying either Condition \ref{con:transition}\ref{con:transition_I} or \ref{con:transition}\ref{con:transition_II} are randomly generated; in Table \ref{table:inv_tran}, transitions satisfying Condition \ref{con:transition}\ref{con:transition_II} and $G^s_y(x)$ that satisfy Condition \ref{ass:inv_transition} are randomly generated, $\frac{u_+}{u_-}$ also satisfies the condition in Theorem \ref{thm:transition}. These results are consistent with Theorem \ref{thm:comp_dp_eqopt} and \ref{thm:transition}.
\begin{table}[h!]
 \caption{$\widehat{\alpha}^a_{\mathcal{C}} - \widehat{\alpha}^b_{\mathcal{C}}$ when $\mathcal{C} = \texttt{UN}, \texttt{EqOpt}, \texttt{DP}$: $G^a_y(x) = G^b_y(x)$ and $T^a_{yd} \neq T^b_{yd}$.}
\centering
\resizebox{\textwidth}{!}{
\begin{tabular}{ |p{2.25cm}|p{1.cm} p{0.8cm} p{0.8cm} p{1.cm} p{1.cm} p{1.cm} p{0.8cm} p{1.cm} p{1.cm}| }
 \hline
 & \multicolumn{9}{c|}{Condition \ref{con:transition}\ref{con:transition_I}} \\
 \hline
\texttt{UN} ($\times 10^{-2}$)& -18.45 & 16.89& 19.82&-7.21& -16.34 & -26.56 & 16.66 & -6.03 & -38.63\\
 \hline
\texttt{EqOpt} ($\times 10^{-2}$) & -21.11   &19.13&   21.78& -7.62& -18.56 &-29.21 & 18.14 & -6.28& -41.52 \\
\hline
\texttt{DP} ($\times 10^{-2}$)&  -27.98  & 23.11   &25.65& -8.90& -23.11 & -33.22 & 21.09  & -6.66 & -43.35\\
\hline
 & \multicolumn{9}{c|}{Condition \ref{con:transition}\ref{con:transition_II}} \\
 \hline
\texttt{UN} ($\times 10^{-2}$)& -19.05 & 18.18 & -0.70& -58.80 & -40.91& 61.30& 12.82& -44.67&2.66 \\
 \hline
\texttt{EqOpt} ($\times 10^{-2}$) & -18.40 & 17.98 &-0.64& -57.62 & -34.50&48.66 &12.35 & -41.43&2.61 \\
\hline
\texttt{DP} ($\times 10^{-2}$)& -17.52  &   17.73  &-0.57 & -55.62 & -28.97&36.10 &11.69 & -37.97&2.57 \\
\hline
\end{tabular}
}
\label{table:inv_dist}
\end{table}

\begin{table}[h!]
 \caption{$\widehat{\alpha}^a_{\mathcal{C}} - \widehat{\alpha}^b_{\mathcal{C}}$ when $\mathcal{C} = \texttt{UN}, \texttt{EqOpt}, \texttt{DP}$: $G^a_y(x)\neq G^b_y(x)$ and $T^a_{yd} = T^b_{yd}$ under Condition \ref{con:transition}\ref{con:transition_II}.}
\centering
\resizebox{\textwidth}{!}{
\begin{tabular}{ |p{2.3cm}|p{0.8cm} p{0.8cm} p{0.8cm} p{0.8cm} p{0.8cm} p{0.8cm} p{0.8cm} p{0.8cm} p{0.8cm}| }
 \hline
\texttt{UN} ($\times 10^{-2}$)& 1.88 & 26.35& 2.12&0.38&5.64&12.35&11.70&0.20&4.12\\
 \hline
\texttt{EqOpt} ($\times 10^{-2}$) & 0.57   &17.43&   1.75& 0.32&5.05&7.81&7.21&0.18&1.68\\
\hline
\texttt{DP} ($\times 10^{-4}$)&  16.26  & 18.29   &-5.94& -0.93&-2.25&1.47&0.92&-1.68&-0.80\\
\hline
\end{tabular}
}
\label{table:inv_tran}
\end{table}
\paragraph{FICO score data.}
\begin{wrapfigure}[15]{r}{0.4\textwidth}
\centering
\includegraphics[trim=1cm 0cm 1cm 1.cm,clip=true,width=0.4\textwidth]{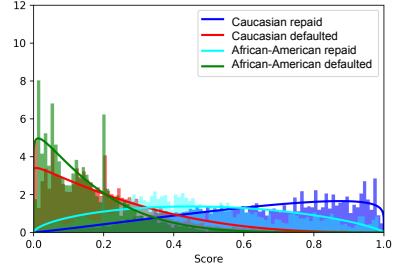}
\caption{The feature distributions: the scores are rescaled so that they are between 0 and 1.}
\label{fig:fico-beta}
\end{wrapfigure} 
From the pre-processed FICO dataset, we got $\mathbb{P}(X=x \mid S=s)$ and $\mathbb{P}(Y=1\mid X=x, S=s)$. In this experiment, we consider two demographic groups, $12\%$ the African American $\mathcal{G}_{AA}$ and $88\%$ the Caucasian $\mathcal{G}_{C}$. According to the empirical feature distributions, we can first simulate the FICO dataset with credit scores $X$, repayment $Y$, and sensitive attribute $S$. We then compute the initial qualification (repayment) rates $({\alpha}^{AA}_0,{\alpha}^{C}_0)$, which is $0.34$ in $\mathcal{G}_{AA}$ and $ 0.76$ in $\mathcal{G}_{C}$; and fit Beta distributions to get the feature distribution $\mathbb{P}(X=x \mid S=s, Y=y)$, as shown in Fig. \ref{fig:fico-beta}. Since the feature distributions are the Beta distributions, we can compute optimal \texttt{UN}, \texttt{EqOpt}, \texttt{DP} thresholds directly using Eqn. \eqref{eq:opt_fair_policy} and update the repayment rates based on dynamics \eqref{eq:dynamics}. This process proceeds and $({\alpha}^{AA}_t,{\alpha}^{C}_t)$ changes over time.

We then consider the demographic-invariant (D-invariant) and demographic-variant (D-variant) transitions and examine the impact of the transition interventions. Specifically, in the context of loan repayment prediction and group lending \citep{paxton2000modeling}, the transitions would satisfy Condition \ref{con:transition}\ref{con:transition_II}. Fig. \ref{fig:fico_full} illustrates the equilibria $(\widehat{\alpha}^{AA},\widehat{\alpha}^C)$ under different sets of transitions. Their specific values are listed as follows, where the system has an equilibrium in all cases. 
\begin{eqnarray*}
\text{D-invariant:  } & T_{00}=0.1, T_{11}=0.9,& T_{10},T_{01}\in \{0.1,\,0.3,\,0.5,\,0.7,\,0.9\}\\
\text{D-variant:  } & T^{AA}_{00}=0.1, T^{AA}_{11}=0.9,&T^{AA}_{10},T^{AA}_{01}\in \{0.20,\,0.36,\,0.53,\,0.69,\,0.85\}\\
& T^{C}_{00}=0.4, T^{C}_{11}=0.9, & T^{C}_{10},T^{C}_{01}\in \{0.45,\,0.55,\,0.65,\,0.75,\,0.85\}
\end{eqnarray*}

\begin{figure}
\centering
\begin{subfigure}[b]{0.31\textwidth}
\centering
\includegraphics[trim=0cm 0cm 0.5cm 0.5cm,clip=true, width=\textwidth]{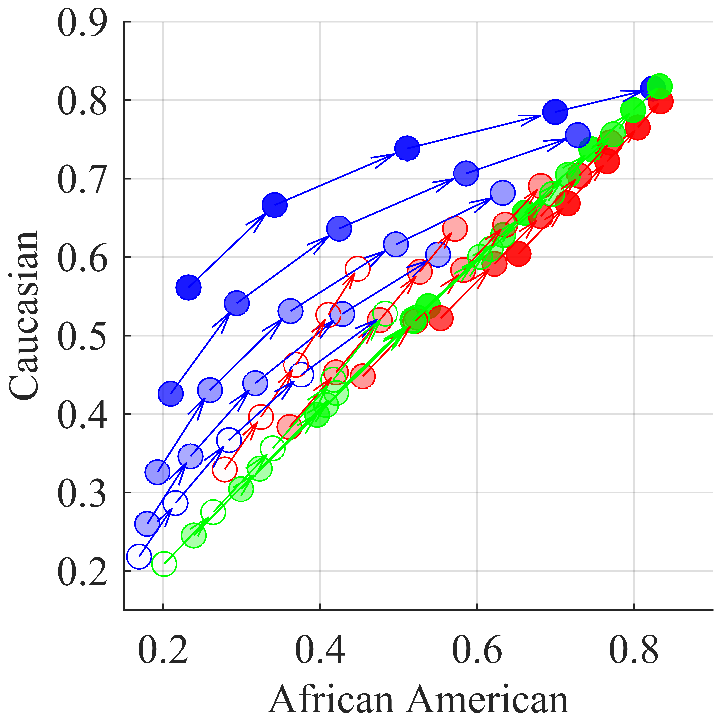}
\caption{D-invariant transitions}
\end{subfigure}
\hspace{3cm}
\begin{subfigure}[b]{0.31\textwidth}
\centering
\includegraphics[trim=0cm 0cm 0.5cm 0.5cm,clip=true, width=\textwidth]{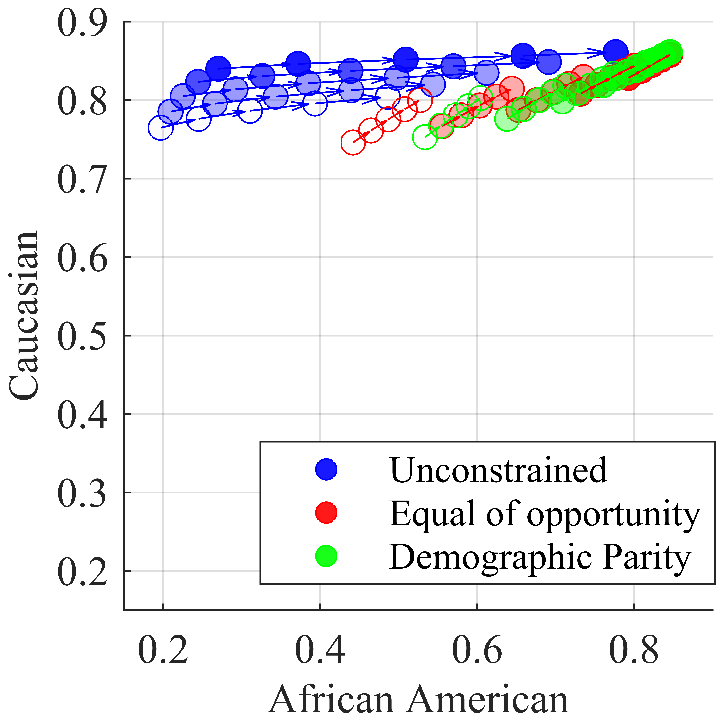}
\caption{
D-variant transitions}
\end{subfigure}
\caption{Results on the FICO dataset: Points are the repayment rates of  $\mathcal{G}_{AA}, \mathcal{G}_{C}$ at the equilibria under Condition \ref{con:transition}\ref{con:transition_II} with different sets of transitions. Arrows indicate the direction of increasing $T_{01}^s$; a more transparent point represents the smaller value of $T_{10}^s$. In panel \subref{fig:sm_t}, $T^{AA}_{yd} = T^{C}_{yd}$, while in panel \subref{fig:diff_t}, $T^{AA}_{yd} < T^{C}_{yd}$.}
\label{fig:fico_full}
\end{figure} 

\paragraph{COMPAS data.}
\begin{figure}
\centering
\includegraphics[trim=3cm 2cm 3cm 3cm,clip=true, width=0.4\textwidth]{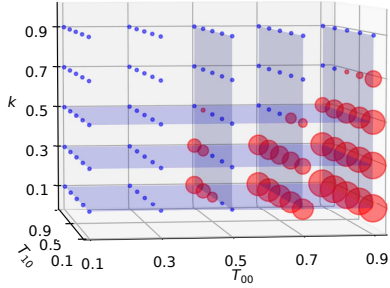}
\caption{The oscillation level of recidivism rates in the long run is represented by the size of red circles, of which the bigger one represents severer oscillation.
The blue dots represent the scenarios with a unique equilibrium. $T_{00}$ and $T_{10}$ axes represent their values respectively; $k$ axis represents the scalar $k$, where $T_{01}=k \times T_{00}$ and $T_{11}=k\times T_{10}$.}
\label{fig:compas_osi}
\end{figure} 

The COMPAS dataset is a high-dimensional dataset with mixed data types (e.g., continuous, binary, and categorical). The number of samples is $5278$. There are 10 features and two demographic groups: $60\%$ African American ($\mathcal{G}_{AA}$) and $40\%$ Caucasian ($\mathcal{G}_{C}$). The qualification rate in COMPAS is the recidivism rate. The initial recidivism rates are $52.3\%$ in $\mathcal{G}_{AA}$ and $39.1\%$ in $\mathcal{G}_{C}$. 

Due to the complexity of the feature distribution, the system can be either in the equilibrium state or oscillate between two recidivism rates in the long-run. Since the feature distribution is fixed and approximated from the COMPAS dataset, we investigate that under which transitions, the system is in an equilibrium state under unconstrained optimal policy. For this purpose, it is sufficient to study the demographic-invariant transitions $T^{AA}=T^{C}$ and consider the entire population without distinguishing two groups; moreover, in the context of recidivism prediction, the transitions would satisfy Condition \ref{con:transition} \ref{con:transition_I}. Therefore, we consider $T_{00}$ and $T_{10}$ taking the values $0.1$, $0.3$, $0.5$, $0.7$ and $0.9$. Figure \ref{fig:compas_osi} shows the results when $T_{01}= k \times T_{00}$ and $T_{11} = k \times T_{10}$. We find that when Corollary \ref{rmk:unique} is satisfied, e.g., when $k \geq 0.5$, most of the corresponding systems have a unique equilibrium (blue dot). Moreover, when $T_{00}\leq 0.5$, the system is also mostly in the unique equilibrium state. For the other transitions, the system oscillates between two states (red circle). We also show the results under all the combinations of $T_{01}$ and $T_{11}$ in Figure \ref{fig:compas_all}. 

Next, we study the impact of policy interventions in cases with equilibrium.
We randomly choose the transitions under which the system has an equilibrium and then apply the unconstrained policy with optimal threshold (classifier threshold $0.5$), a higher and a lower threshold (classifier thresholds $0.8$ and $0.2$ respectively) compared to the optimum respectively. The results are show in Table \ref{tbl:compas_full}.
\begin{table}[H]
\centering
\caption{Recidivism rates in the long run. 
$\texttt{UN}^{*}$: unconstrained policy ($\texttt{UN}$) with the optimal threshold; 
$\texttt{UN}_{\theta_H}$: $\texttt{UN}$ with a higher threshold;
$\texttt{UN}_{\theta_L}$: $\texttt{UN}$ with a lower threshold.}
\label{tbl:compas_full}
\begin{tabular}{llll}
\hline
                    & $\texttt{UN}^*$ & $\texttt{UN}_{\theta_H}$ & $\texttt{UN}_{\theta_L}$ \\ \hline
$\widehat{\alpha}_1$  & 0.164         & 0.166                    & 0.147              \\ 
$\widehat{\alpha}_2$  & 0.343         & 0.356                    & 0.307              \\ 
$\widehat{\alpha}_3$  & 0.230         & 0.246                    & 0.162              \\ 
$\widehat{\alpha}_4$  & 0.306         & 0.3415                    & 0.156              \\ 
$\widehat{\alpha}_5$  & 0.162         & 0.166                    & 0.140\\ \hline
\end{tabular}   
\end{table}

\begin{figure}
\centering
\begin{subfigure}[b]{0.32\textwidth}
\includegraphics[trim=3.2cm 2.cm 3.1cm 2.8cm,clip=true,width=\textwidth]{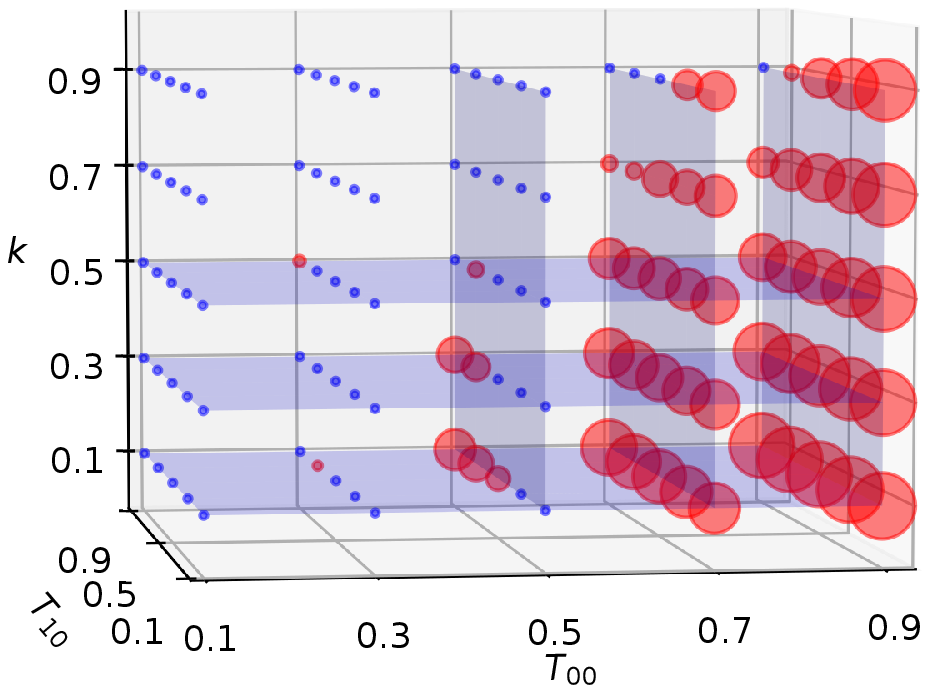}
\caption{$T_{01} = 0.1\times T_{00}$.}
\end{subfigure}
\begin{subfigure}[b]{0.32\textwidth}
\includegraphics[trim=3.2cm 2.cm 3.1cm 2.8cm,clip=true,width=\textwidth]{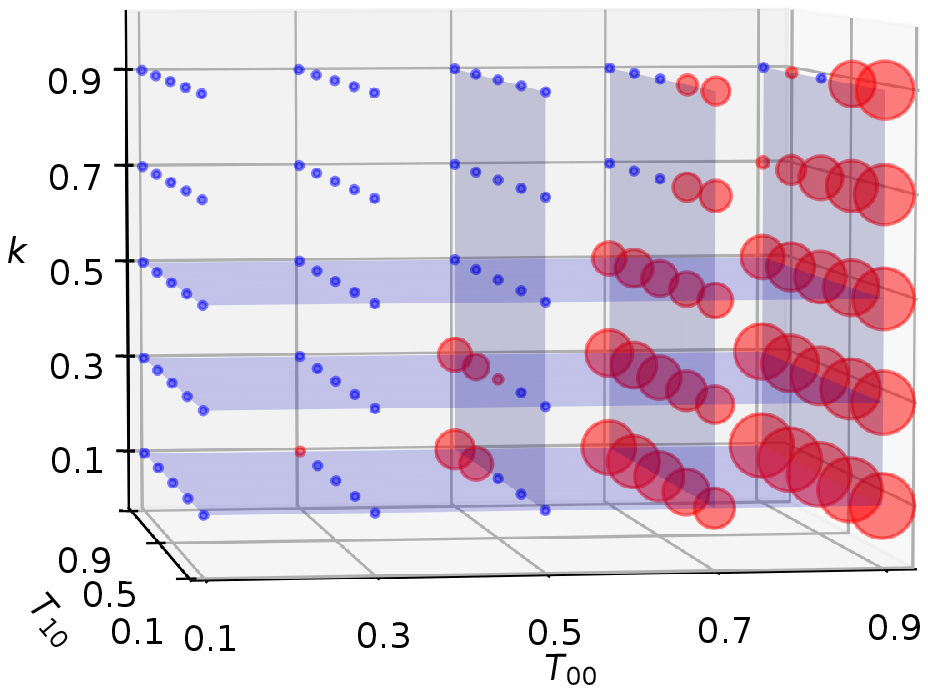}
\caption{$T_{01} = 0.3\times T_{00}$.}
\end{subfigure}
\begin{subfigure}[b]{0.32\textwidth}
\includegraphics[trim=3.2cm 2.cm 3.1cm 2.8cm,clip=true,width=\textwidth]{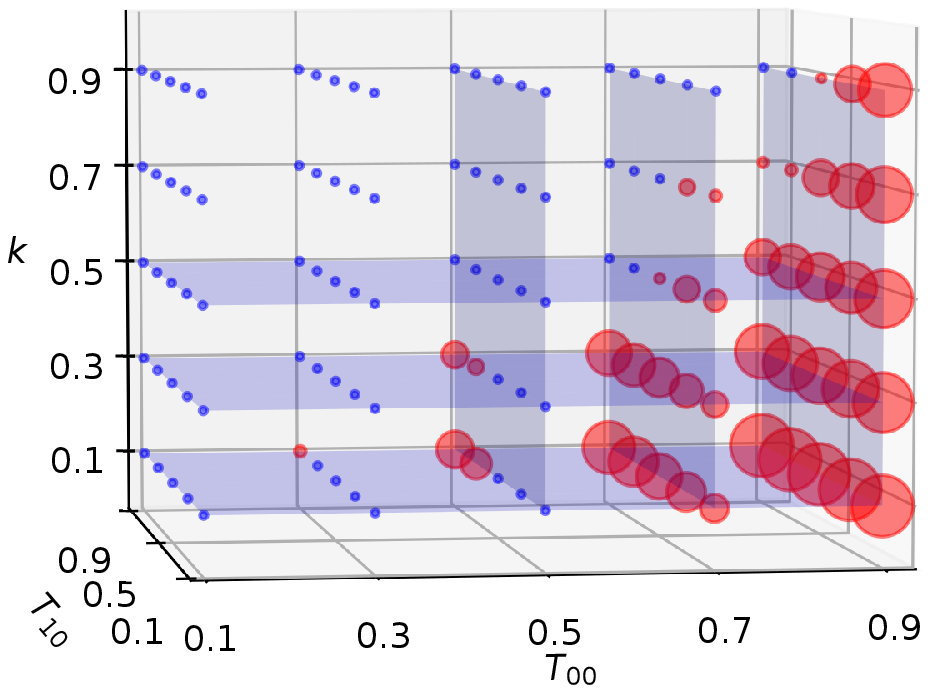}
\caption{$T_{01} = 0.5\times T_{00}$.}
\end{subfigure}

\begin{subfigure}[b]{0.32\textwidth}
\includegraphics[trim=3.2cm 2.cm 3.1cm 2.8cm,clip=true,width=\textwidth]{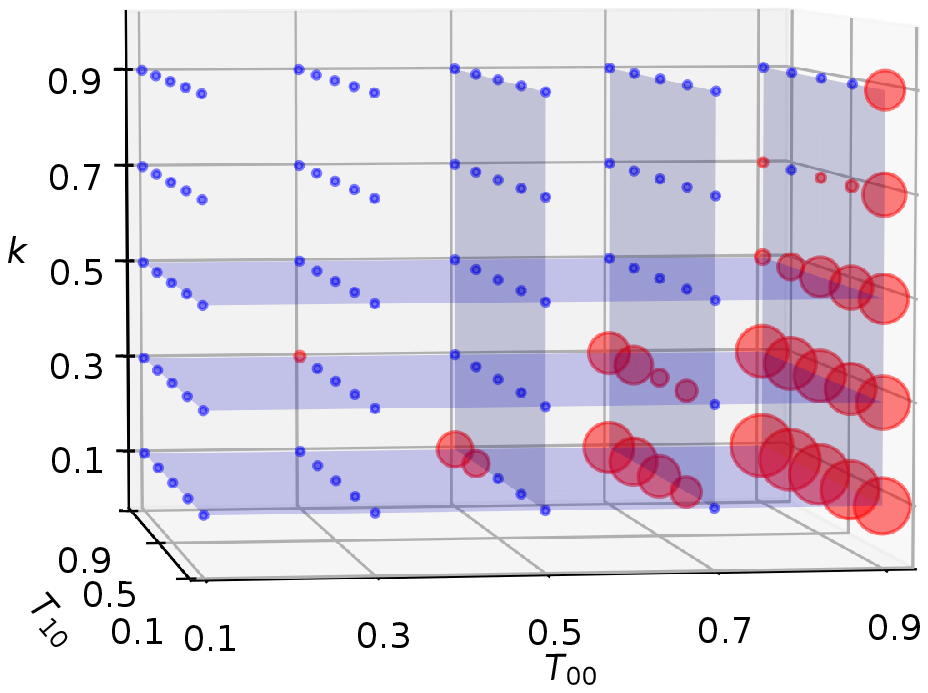}
\caption{$T_{01} = 0.7\times T_{00}$.}
\end{subfigure}
\begin{subfigure}[b]{0.32\textwidth}
\includegraphics[trim=3.2cm 2.cm 3.1cm 2.8cm,clip=true,width=\textwidth]{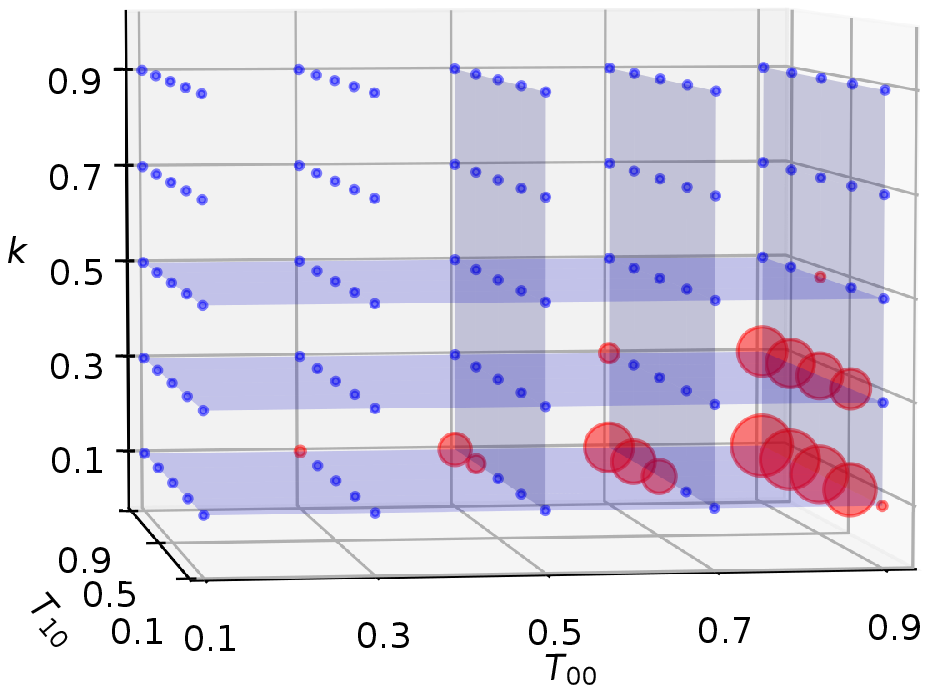}
\caption{$T_{01} = 0.9\times T_{00}$.}
\end{subfigure}
\caption{The oscillation level of recidivism rates under different transitions. In each panel, scalar $k$ denotes the ratio, of which $T_{11} = k\times T_{10}$.}
\label{fig:compas_all}
\end{figure}
\section{Generalization to high-dimensional feature space}\label{app:highDim}
All analysis and conclusions in this paper can be generalized to high-dimensional feature space $x\in \mathbb{R}^d$, where the qualification profile of $\mathcal{G}_s$ is defined as $\gamma^s_t(\mathbf{x})=\mathbb{P}(Y_t = 1 \mid \mathbf{X}_t =\mathbf{x}, S=s )\in[0,1],\;\mathbf{x}\in \mathbb{R}^d$. Different from one-dimensional case where decisions are made based on features, here decisions are made based on  $\gamma^s_t(\mathbf{x})$, i.e., 
high-dimensional features are mapped into a one-dimensional space first and decisions are made in this transformed space. 
The threshold policy in this case becomes $\pi^s_t(\mathbf{x}) = \mathbf{1}(\gamma^s_t(\mathbf{x}) \geq \theta^{s}_t)$ with threshold $\theta^{s}_t\in[0,1]$. Let $\gamma_t^{s^{-1}}(\theta)\subset \mathbb{R}^d$ be defined as the preimage of $\theta$ under qualification profile $\gamma_t^s$, then all analysis in one-dimensional settings can be adjusted using $\gamma_t^{s^{-1}}(\cdot)$. 
For example, Assumption \ref{ass:mono-inference} in high-dimensional case can be adjusted to the following: \textit{$\forall s\in\{a,b\}$, given any two thresholds $0 \leq \theta_j^s < \theta_k^s \leq 1$, we have $\gamma_t^{s^{-1}}(K^s) \subset \gamma_t^{s^{-1}}(J^s)$, where $J^s = \{\theta:  \theta \in [ \theta_j^s,1] \}$ and   $K^s = \{\theta:  \theta\in [ \theta_k^s,1] \}$;} in other words, if an individual can get accepted by a policy with the higher threshold, it must be accepted if a policy with a lower threshold was used. Note that this assumption is still mild and always hold if $G^s_y(\mathbf{x})$ belongs to exponential family.

Specifically, if $\forall s\in \{a,b\}, \forall y\in \{0,1\}$, distribution of $X|Y=y,S=s$ belongs to exponential family and can be written as $G_y^s(\mathbf{x}) \coloneqq B(\mathbf{x})\exp\big(\langle\eta(\omega_y^s),\xi(\mathbf{x})\rangle-A(\omega_y^s)\big)$ for some functions $B(\cdot),\eta(\cdot),\xi(\cdot),A(\cdot)$, where $\langle \mathbf{x},\mathbf{y} \rangle$ represents inner product of two vectors $\mathbf{x},\mathbf{y}$ and $\omega_y^s$ is the parameter. Then $\frac{G_0^s(\mathbf{x})}{G_1^s(\mathbf{x})} = \exp\big( -\langle \eta^s,\xi(\mathbf{x})\rangle + A^s\big)$ where $\eta^s \coloneqq \eta(\omega_1^s)-\eta(\omega_0^s)$ and $A^s \coloneqq A(\omega_1^s)-A(\omega_0^s)$. Then
\begin{eqnarray*}
\gamma_t^{s^{-1}}(J^s) = \{\mathbf{x}: \gamma_t^{s}(\mathbf{x}) \geq \theta^s_j \} =  \{\mathbf{x}:\langle \eta^s,\xi(\mathbf{x})\rangle \geq A^s + \log\Big(\frac{\frac{1}{\alpha_t^s}-1}{\frac{1}{\theta^s_j}-1}\Big) \}
\end{eqnarray*}
If $ \theta^s_j <  \theta^s_k$, then $\log\Big(\frac{\frac{1}{\alpha_t^s}-1}{\frac{1}{\theta^s_j}-1}\Big)< \log\Big(\frac{\frac{1}{\alpha_t^s}-1}{\frac{1}{\theta^s_k}-1}\Big)$. We have $\gamma_t^{s^{-1}}(K^s) \subset \gamma_t^{s^{-1}}(J^s)$.

\section{Discussions}\label{app:discussion}
\paragraph{Transitions under Condition 1(C) or 1(D).}
This paper mainly focus on transitions satisfying Condition \ref{con:transition}\ref{con:transition_I} and \ref{con:transition}\ref{con:transition_II}. As mentioned in Section \ref{sec:evolutions_equilibrium}, there are the other two combinations: (C) $T^s_{01} \geq T^s_{00}$ and  $T^s_{11} \leq T^s_{10}$; (D) $T^s_{01} \leq T^s_{00}$ and  $T^s_{11} \geq T^s_{10}$, in which there is more uncertainty when conducting equilibrium analysis. The slight changes in the feature distributions or the values of transitions may change conclusions significantly. 

Because the system has equilibrium if there is solution to balanced equations defined as Eqn. \eqref{eq:balanced_eqn_a} in Appendix \ref{app:proofs}, i.e., $\frac{1}{\alpha^s} - 1 = \frac{1-g^{1s}(\theta^s(\alpha^a,\alpha^b))}{g^{0s}(\theta^s(\alpha^a,\alpha^b))}$, $\forall s\in \{a,b\}$. Since
\begin{eqnarray*}
  \frac{1-g^{1s}(\theta^s(\alpha^a,\alpha^b))}{g^{0s}(\theta^s(\alpha^a,\alpha^b))} = \frac{1 - ( T_{10}^s\mathbb{G}_1^s(\theta^s(\alpha^a,\alpha^b)) + T_{11}^s\big(1-\mathbb{G}_1^s(\theta^s(\alpha^a,\alpha^b))\big))}{ T_{00}^s\mathbb{G}_0^s(\theta^s(\alpha^a,\alpha^b)) + T_{01}^s\big(1-\mathbb{G}_0^s(\theta^s(\alpha^a,\alpha^b))\big)}.
\end{eqnarray*}
Under optimal (fair) policies and Condition \ref{con:transition}\ref{con:transition_I} or \ref{con:transition}\ref{con:transition_II}, $\frac{1-g^{1s}(\theta^s(\alpha^a,\alpha^b))}{g^{0s}(\theta^s(\alpha^a,\alpha^b))}$ is guaranteed to be either decreasing or increasing in $\alpha^s$. This monotonicity is critical to determine the properties (e.g., uniqueness, quantity, value, etc.) of the consequent equilibrium $(\widehat{\alpha}^a_{\mathcal{C}},\widehat{\alpha}^b_{\mathcal{C}})$ so that impacts of different fairness can be compared. In contrast, under Condition 1(C) or 1(D),  $\frac{1-g^{1s}(\theta^s(\alpha^a,\alpha^b))}{g^{0s}(\theta^s(\alpha^a,\alpha^b))}$ is no longer monotonic, and its intersection with function $\frac{1}{\alpha^s} - 1$, i.e., equilibrium, is thus hard to characterize. As a consequence, the impacts of different fairness constraints cannot be compared in general.   

\paragraph{Comparison between sufficient conditions in Theorem \ref{thm:unique_equilibrium} and Lipschitz condition.}

Let a pair of qualification rats of $\mathcal{G}_a,\mathcal{G}_b$ be noted as $\alpha \coloneqq (\alpha^a,\alpha^b)\in [0,1]\times [0,1]$, and let mapping $\Phi:[0,1]\times [0,1] \to [0,1]\times [0,1]$ be defined such that dynamical system \eqref{eq:dynamics} can be written as $\alpha_{t+1} = \Phi(\alpha_{t})$. Then this dynamical system has an equilibrium $\widehat{\alpha}$ if $\Phi(\widehat{\alpha})=\widehat{\alpha}$. According to Banach Fixed Point Theorem, such equilibrium exists and is unique if the mapping $\Phi$ satisfies $L$-Lipschitz condition with $L < 1$, i.e., $\Phi$ is a contraction mapping. Specifically, $d(\Phi(\alpha_0),\Phi(\alpha_1))\leq Ld(\alpha_0,\alpha_1)$ for some distance function $d$ and Lipschitz constant $L< 1$. 

While Lipschitz condition also ensures the uniqueness of equilibrium, the sufficient conditions given in Theorem \ref{thm:unique_equilibrium} are weaker. Use unconstrained optimal policies as an example, in this case dynamics of two groups can be decoupled because threshold $\theta^s(\alpha^a,\alpha^b)$ used in $\mathcal{G}_s$ is independent of qualification of the other group $\alpha^{-s}$. Therefore, sufficient condition $|\frac{\partial h^s(\theta^s(\alpha^a,\alpha^b))}{\partial \alpha^{-s}} | = 0<1$ under Condition \ref{con:transition}\ref{con:transition_I} always holds. In contrast, for dynamics of $\mathcal{G}_s$ after decoupling $\alpha^s_{t+1} = \Phi^s(\alpha^s_t) = g^{0s}(\theta^s(\alpha^s_t))(1-\alpha^s_t)+g^{1s}(\theta^s(\alpha^s_t))\alpha^s_t$, $\Phi^s$ is not necessarily a contraction mapping. 

Although sufficient conditions in Theorem \ref{thm:unique_equilibrium} are weaker, they do not guarantee the stability of the equilibrium. In contrast, Lipschitz condition with $L<1$ ensures the unique equilibrium is also stable, i.e., we have $(\alpha^a_t,\alpha^b_t) \to (\widehat{\alpha}^a,\widehat{\alpha}^b)$ given an arbitrary initial state $(\alpha^a_0,\alpha^b_0)$.  

\section{Derivations} \label{app:derivations}
\paragraph{Qualification profile of a group.}
\begin{eqnarray*} 
\gamma_t^s(x) &=& \mathbb{P}(Y_t = 1|X_t = x, S=s) 
=\frac{1}{\frac{\mathbb{P}(X_t = x,Y_t = 0, S=s)}{\mathbb{P}(X_t = x,Y_t = 1, S=s)} + 1}
\\
&=&\frac{1}{\frac{\mathbb{P}(X_t = x\mid Y_t = 0, S=s)\mathbb{P}(Y_t = 0\mid S=s)}{\mathbb{P}(X_t = x\mid Y_t = 1, S=s)\mathbb{P}(Y_t = 1\mid S=s)} + 1}\\
&=& \frac{1}{\frac{\mathbb{P}(X_t = x\mid Y_t = 0, S=s)}{\mathbb{P}(X_t = x\mid Y_t = 1, S=s)}(\frac{1}{\mathbb{P}(Y_t = 1\mid S=s)}-1) + 1}\\
&=&  \frac{1}{\frac{G^s_0(x)}{G^s_1(x)}(\frac{1}{\alpha_t^s}-1) + 1}.
\end{eqnarray*}

\paragraph{Utility of an institute.}
\begin{eqnarray*}
\mathcal{U}(D_t,Y_t) = \mathbb{E}[R_t(D_t,Y_t)] = \mathbb{P}(S = a)\mathbb{E}[R_t(D_t,Y_t)|S=a]  +
\mathbb{P}(S = b)\mathbb{E}[R_t(D_t,Y_t)|S=b]
\end{eqnarray*}

Under policy $\pi^s$, we have
\begin{eqnarray*}
&&\mathbb{E}[R_t(D_t,Y_t)|S=s]= \mathbb{P}(D_t=1,Y_t=1|S=s)u_+-\mathbb{P}(D_t=1,Y_t=0|S=s)u_- \\
&=& \int_{x}\Big(\mathbb{P}(D_t=1,Y_t=1,X_t=x|S=s)u_+-\mathbb{P}(D_t=1,Y_t=0,X_t=x|S=s)u_-\Big) dx \\
&=&  \int_x\mathbb{P}(X_t=x|S=s)\Big(\mathbb{P}(D_t=1\mid X_t=x,S=s)\mathbb{P}(Y_t=1\mid X_t=x,S=s)u_+\\
&& -\mathbb{P}(D_t=1\mid X_t=x,S=s)\mathbb{P}(Y_t=0\mid X_t=x,S=s)u_-\Big) dx\\
&=&  \int_x\mathbb{P}(X_t=x|S=s)\Big(\pi^s(x)\gamma_t^s(x)u_+ -\pi^s(x)(1-\gamma_t^s(x))u_-\Big) dx\\
&=& \mathbb{E}_{X_t|S=s}[\pi^s(X_t)(\gamma_t^s(X_t)(u_++u_-) -u_-) ].
\end{eqnarray*}

Therefore, 
\begin{eqnarray*}
\mathcal{U}(D_t,Y_t)  = p_a\mathbb{E}_{X_t|S=a}[\pi^a(X_t)(\gamma_t^a(X_t)(u_++u_-) -u_-) ] +
p_b\mathbb{E}_{X_t|S=b}[\pi^b(X_t)(\gamma_t^b(X_t)(u_++u_-) -u_-) ]
\end{eqnarray*}

\paragraph{Dynamics of qualification rate.}
\begin{eqnarray*}
\alpha^s_{t+1} &=& \mathbb{P}(Y_{t+1} = 1\mid S=s)=\int_{x} \sum_{y,a}\mathbb{P}(Y_{t+1} = 1, Y_{t}=y, D_{t}=d, X_{t}=x \mid S=s) d x\nonumber \\
&=& \int_{x} \sum_{y,a}
\mathbb{P}(Y_{t+1} = 1\mid Y_{t}=y, X_{t}=x, D_{t}=d,S=s)\mathbb{P}(D_{t}=d\mid X_{t}=x,S=s)\\&&\mathbb{P}( X_{t}=x \mid Y_{t}=y,S=s)
\mathbb{P}(Y_{t}=y\mid S=s) d x\nonumber \\
&=&\int_{x}\sum_{a}
\Big\{
\mathbb{P}(Y_{t+1} = 1\mid Y_{t}=0, X_{t}=x, D_{t}=d,S=s)\nonumber\\
&&
\mathbb{P}(D_{t}=d\mid X_{t}=x,S=s) 
\mathbb{P}( X_{t}=x \mid Y_{t}=0,S=s)\Big\}
\mathbb{P}(Y_{t}=0\mid S=s)d x\nonumber\\
&&+ \int_{x}\sum_{d}\Big\{\mathbb{P}(Y_{t+1} = 1\mid Y_{t}=1, X_{t}=x, D_{t}=d,S=s)\nonumber\\
&&\mathbb{P}(D_{t}=d\mid X_{t}=x,S=s) 
\mathbb{P}( X_{t}=x \mid Y_{t}=1,S=s)\Big\}
\mathbb{P}(Y_{t}=1\mid S=s)d x
\nonumber\\
&=& 
\mathbb{E}_{X_{t}|Y_{t}=0,S=s}\Big[(1-\pi^s_{{t}}(X_{t}))T_{00}^s+\pi^s_{{t}}(X_{t})T_{01}^s\Big](1-\alpha_{t}^s) \nonumber\\
&&+~
\mathbb{E}_{X_{t}|Y_{t}=1,S=s}\Big[(1-\pi^s_{{t}}(X_{t}))T_{10}^s+\pi_{{t}}^s(X_{t})T_{11}^s\Big]\alpha_{t}^s\nonumber\\
&=& g^{0s}(\alpha^a_t,\alpha^b_t)\cdot(1-\alpha_{t}^s) + g^{1s}(\alpha^a_t,\alpha^b_t)\cdot \alpha_{t}^s
\end{eqnarray*}
\section{Proofs} \label{app:proofs}
We define balanced equations and functions for the rest proofs.
The dynamics system \eqref{eq:dynamics} can reach equilibrium if $\alpha^s_t = \alpha^s_{t-1}$ holds. Therefore, the system has equilibrium if there exists solution to the \textit{balanced equations} defined as \eqref{eq:balanced_eqn_a}. 
\begin{eqnarray}\label{eq:balanced_eqn_a}
      \frac{1}{\alpha^a} - 1 =\frac{1-g^{1a}(\theta^a(\alpha^a,\alpha^b))}{g^{0a}(\theta^a(\alpha^a,\alpha^b))}; ~~~~ \frac{1}{\alpha^b} - 1  = \frac{1-g^{1b}(\theta^b(\alpha^a,\alpha^b))}{g^{0b}(\theta^b(\alpha^a,\alpha^b))}.
\end{eqnarray}
By removing subscript $t$ and writing threshold $\theta^s$ as a function of $\alpha^a,\alpha^b$, we have 
$g^{ys}(\theta^s(\alpha^a,\alpha^b)) =  T_{y0}^s\mathbb{G}_y^s(\theta^s(\alpha^a,\alpha^b)) + T_{y1}^s\big(1-\mathbb{G}_y^s(\theta^s(\alpha^a,\alpha^b))\big)$, denote CDF of ${G}_y^s(x)$ as $\mathbb{G}_y^s(\theta) = \int_{-\infty}^{\theta}{G}_y^s(x)dx$.

$\forall s \in \{a,b\}$, let $-s \coloneqq \{a,b\}\setminus s$. $\forall \alpha^{-s}\in [0,1]$, define \textit{balanced set} w.r.t. dynamics as $\Psi^s(\alpha^{-s}) \coloneqq \{\overline{\alpha}^s: \frac{1}{\overline{\alpha}^s} - 1 =\frac{1-g^{1s}(\theta^a(\overline{\alpha}^s,\alpha^{-s}))}{g^{0s}(\theta^s(\overline{\alpha}^s,\alpha^{-s}))}\}$. If the set size $|\Psi^s(\alpha^{-s})| = 1$ holds $ \forall \alpha^{-s}\in [0,1]$, we define \textit{balanced functions} w.r.t. dynamics as $\psi^s:[0,1] \to [0,1]$ with $\psi^s(\alpha^{-s}) \in \Psi^s(\alpha^{-s}), \forall \alpha^{-s}\in [0,1]$.

\paragraph{The proof \rt{that the threshold policies are optimal under our formulation}.}
\begin{proof}
In the following proof, we focus on optimal policy at $t$ and omit the subscript $t$. 

First consider unconstrained optimal policy, noted as $\pi^s_{\texttt{UN}}$, we have,
\begin{eqnarray*}
\pi^s_{\texttt{UN}} = \arg\max_{\pi^s}\mathbb{E}_{X\mid S=s}[\pi^s(X)(\gamma^s(X)(u_++u_-)-u_-)]
\end{eqnarray*}
Therefore, the optimal policy satisfies $\pi^s_{\texttt{UN}}(x) = \textbf{1}( \gamma^s(x) \geq \frac{u_-}{u_++u_-})$. Since $\gamma^s(x)$ is monotonically increasing in $x$ under Assumption \ref{ass:mono-inference}, $\pi^s_{\texttt{UN}}(x) = \textbf{1}( x \geq (\gamma^s)^{-1}(\frac{u_-}{u_++u_-}))$ is threshold policy where $(\gamma^s)^{-1}(\cdot)$ denotes the inverse function of $\gamma(\cdot)$.

Now consider optimal fair policy under some fairness constraint $\mathcal{C}$ satisfying Assumption \ref{ass:fair_prob}. Consider any pair of policies $(\pi^a,\pi^b)$ that satisfies fairness constraint $\mathcal{C}$, and define fairness constant $c  = \mathbb{E}_{X\sim \mathcal{P}_{\mathcal{C}}^a}[\pi^{a}(X)] = \mathbb{E}_{X\sim \mathcal{P}_{\mathcal{C}}^b}[\pi^{b}(X)]  \in [0,1]$. To show the optimal fair policy is threshold policy, we will show that there always exists a pair of threshold policies $(\pi_d^{a},\pi_d^{b})$ such that $ \mathbb{E}_{X\sim \mathcal{P}_{\mathcal{C}}^a}[\pi_d^{a}(X)] = \mathbb{E}_{X\sim \mathcal{P}_{\mathcal{C}}^b}[\pi_d^{b}(X)] =c$, i.e., the fairness constant is the same as $(\pi^a,\pi^b)$, 
{and the utility of $(\pi_d^{a},\pi_d^{b})$ is no less than the utility attained under $(\pi^a,\pi^b)$.}

 $\forall s\in \{a,b\}$, let threshold policy $\pi^s_d$ be defined such that $\pi_d^{s}(x) = \textbf{1}(x\geq \theta_d^s)$ and $\mathbb{E}_{X\sim \mathcal{P}_{\mathcal{C}}^s}[\pi_d^s(X)] =c$ are satisfied. Such policy must exist and the threshold is given by $\theta_d^s = (\mathbb{P}_{\mathcal{C}}^{s})^{-1}(1-c)$, where $\mathbb{P}_{\mathcal{C}}^{s}(\theta^s) = \int_{-\infty}^{\theta^s}\mathcal{P}_{\mathcal{C}}^{s}(x)dx$ is CDF of $\mathcal{P}_{\mathcal{C}}^{s}$ and $(\mathbb{P}_{\mathcal{C}}^{s})^{-1}(\cdot)$ is the inverse of it. 
 
 Let $R_{\pi^{s}_d}(D,Y)$, $R_{\pi^{s}}(D,Y)$ denote the utility attained under policies $\pi^{s}_d$, $\pi^{s}$ respectively. Next we will show that $\forall s \in \{a,b\}$, $\mathbb{E}[R_{\pi^{s}_d}(D,Y)\mid S=s] \geq \mathbb{E}[R_{\pi^{s}}(D,Y)\mid S=s] $ holds, i.e.,
 $$\mathbb{E}_{X|S=s}[\pi^{s}_d(X)(\gamma^s(X)(u_++u_-) -u_-) ]\geq  \mathbb{E}_{X|S=s}[\pi^{s}(X)(\gamma^s(X)(u_++u_-) -u_-) ]$$
 
 Since $\pi^{s}_d(x) = \textbf{1}(x\geq \theta_d^s)$, we have the followings,
\begin{eqnarray*}
\resizebox{0.38\hsize}{!}{$\mathbb{E}_{X|S=s}[\pi^{s}_d(X)(\gamma^s(X)(u_++u_-) -u_-) ] $}
&=& \resizebox{0.48\hsize}{!}{$\int_{\theta_d^s}^{\infty}(\gamma^s(x)(u_++u_-) -u_-)\mathbb{P}(X=x\mid S=s) dx$}\\
\resizebox{0.38\hsize}{!}{$\mathbb{E}_{X|S=s}[\pi^{s}(X)(\gamma^s(X)(u_++u_-) -u_-) ]$} &=& \resizebox{0.48\hsize}{!}{$
\int_{\theta_d^s}^{\infty}(\gamma^s(x)(u_++u_-) -u_-)\mathbb{P}(X=x\mid S=s) dx$} \\
&+& \resizebox{0.52\hsize}{!}{$\int_{-\infty}^{\theta_d^s}\pi^{s}(x)(\gamma^s(x)(u_++u_-) -u_-)\mathbb{P}(X=x\mid S=s) dx$} \\
&-& \resizebox{0.54\hsize}{!}{$\int_{\theta_d^s}^{\infty}(1-\pi^{s}(x))(\gamma^s(x)(u_++u_-) -u_-)\mathbb{P}(X=x\mid S=s)dx$}
\end{eqnarray*}
 
Since $\mathbb{E}_{X\sim \mathcal{P}^s_{\mathcal{C}}}[\pi^{s}(X)] = c = \mathbb{E}_{X\sim \mathcal{P}^s_{\mathcal{C}}}[\pi^{s}_d(X)]$, we have
\begin{eqnarray}\label{eq:threshold_optimal}
\int_{\theta_d^s}^{\infty}(1-\pi^{s}(x))\mathcal{P}^s_{\mathcal{C}}(x)dx = \int_{-\infty}^{\theta_d^s}\pi^{s}(x)\mathcal{P}^s_{\mathcal{C}}(x)dx
\end{eqnarray}

Under Assumption \ref{ass:fair_prob}, 
\rt{ $\frac{\mathbb{P}(X=x\mid S=s)}{\mathcal{P}^s_{\mathcal{C}}(x)}$ is non-decreasing. Since $\gamma^s(x) = \alpha^s \frac{G^s_1(x)}{\mathbb{P}(X=x\mid S=s)}$ is non-decreasing and $1-\gamma^s(x) =(1-\alpha^s) \frac{G^s_0(x)}{\mathbb{P}(X=x\mid S=s)}$ is non-increasing, 
we have $\frac{G^s_1(x)}{\mathbb{P}(X=x\mid S=s)}$ is non-decreasing and $\frac{G^s_0(x)}{\mathbb{P}(X=x\mid S=s)}$ is non-increasing.
Therefore, 
\begin{eqnarray*}
&&(\gamma^s(x)(u_++u_-) -u_-)\frac{\mathbb{P}(X=x\mid S=s)}{\mathcal{P}^s_{\mathcal{C}}(x)}\\&=& 
\alpha^s \frac{G^s_1(x)}{\mathcal{P}^s_{\mathcal{C}}(x)} u_+ -
(1-\alpha^s)\frac{G^s_0(x)}{\mathcal{P}^s_{\mathcal{C}}(x)}u_-\\
&=& 
\alpha^s 
\frac{G^s_1(x)}{\mathbb{P}(X=x\mid S=s)} 
\frac{\mathbb{P}(X=x\mid S=s)}{\mathcal{P}^s_{\mathcal{C}}(x)}
u_+ -
(1-\alpha^s)
\frac{G^s_0(x)}{\mathbb{P}(X=x\mid S=s)}
\frac{\mathbb{P}(X=x\mid S=s)}{\mathcal{P}^s_{\mathcal{C}}(x)}
u_-
\end{eqnarray*}
is non-decreasing in $x$.} Combine with Eqn. \eqref{eq:threshold_optimal}, we have the followings,
\begin{eqnarray*}
&&\int_{-\infty}^{\theta_d^s}\pi^{s}(x)(\gamma^s(x)(u_++u_-) -u_-)\mathbb{P}(X=x\mid S=s) dx \\ &=& \int_{-\infty}^{\theta_d^s}\pi^{s}(x)(\gamma^s(x)(u_++u_-) -u_-)\frac{\mathbb{P}(X=x\mid S=s)}{\mathcal{P}^s_{\mathcal{C}}(x)}\mathcal{P}^s_{\mathcal{C}}(x) dx \\
&\leq&
\int_{-\infty}^{\theta_d^s}\pi^{s}(x)
(\gamma^s(\theta_d^s)(u_++u_-) -u_-)\frac{\mathbb{P}(X=\theta_d^s\mid S=s)}{\mathcal{P}^s_{\mathcal{C}}(\theta_d^s)}
\mathcal{P}^s_{\mathcal{C}}(x) dx \\
&=&
\int_{\theta_d^s}^{\infty}(1-\pi^{s}(x))
(\gamma^s(\theta_d^s)(u_++u_-) -u_-)\frac{\mathbb{P}(X=\theta_d^s\mid S=s)}{\mathcal{P}^s_{\mathcal{C}}(\theta_d^s)}
\mathcal{P}^s_{\mathcal{C}}(x) dx\\
&\leq&\int_{\theta_d^s}^{\infty}(1-\pi^{s}(x))(\gamma^s(x)(u_++u_-) -u_-)\frac{\mathbb{P}(X=x\mid S=s)}{\mathcal{P}^s_{\mathcal{C}}(x)}\mathcal{P}^s_{\mathcal{C}}(x) dx\\
&=&\int_{\theta_d^s}^{\infty}(1-\pi^{s}(x))(\gamma^s(x)(u_++u_-) -u_-)\mathbb{P}(X=x\mid S=s) dx.
\end{eqnarray*}
 Therefore, the following holds $\forall s\in \{a,b\}$, $$\mathbb{E}_{X|S=s}[\pi^{s}_d(X)(\gamma^s(X)(u_++u_-) -u_-) ]\geq  \mathbb{E}_{X|S=s}[\pi^{s}(X)(\gamma^s(X)(u_++u_-) -u_-) ].$$

It shows that the utility attained under threshold policy $(\pi^a_d,\pi^b_d)$ is no less than the utility of $(\pi^a,\pi^b)$, which concludes that the optimal fair policy $(\pi^a_{\mathcal{C}},\pi^b_{\mathcal{C}})$ must be threshold policies.

Lemma \ref{lemma:policy_function_prop} below further shows that the optimal threshold policy $\theta^s(\alpha^a,\alpha^b)$ is continuous and non-increasing in $\alpha^a$ and $\alpha^b$.

\begin{lemma}\label{lemma:policy_function_prop}
Let $\Big(\theta^a(\alpha^a,\alpha^b),\theta^b(\alpha^a,\alpha^b)\Big)$ be a pair of solutions to Eqn. \eqref{eq:opt_fair_policy} under $\alpha^a,\alpha^b$. $\forall s\in \{a,b\}$, if $\frac{G^s_1(x)}{\mathcal{P}^s_{\mathcal{C}}(x)}$ and $\frac{G^s_0(x)}{\mathcal{P}^s_{\mathcal{C}}(x)}$ are continuous everywhere in $x$, then $\theta^s(\alpha^a,\alpha^b)$ is continuous in both $\alpha^a$ and $\alpha^b$. Moreover, under Assumption \ref{ass:fair_prob}, $\theta^s(\alpha^a,\alpha^b)$ is non-increasing in $\alpha^a$ and $\alpha^b$.
\end{lemma}
\begin{proof}
To prove that \textit{a sufficient condition under which $\theta^s(\alpha^a,\alpha^b)$ is continuous in $\alpha^a,\alpha^b\in[0,1]$ is that $\frac{G^s_1(x)}{\mathcal{P}^s_{\mathcal{C}}(x)}$ and $\frac{G^s_0(x)}{\mathcal{P}^s_{\mathcal{C}}(x)}$ are continuous everywhere in $x$}, we define a function $f^s(\theta^s, \alpha^a, \alpha^b)$:
\begin{eqnarray*}
f^s(\theta^s, \alpha^a, \alpha^b) &=& 
(\gamma^s(\theta^s)-\frac{u_-}{u_++u_-})
\frac{\mathbb{P}(X=\theta^s\mid S=s)}{\mathcal{P}^s_{\mathcal{C}}(\theta^s)}\\
&=&  [
\alpha^s u_+ G^s_1(\theta^s) + \alpha^s u_- G^s_0(\theta^s) - u_- G^s_0(\theta^s) ]\frac{1}{\mathcal{P}^s_{\mathcal{C}}(\theta^s)} \\
&=& [
\alpha^s \frac{G^s_1(\theta^s)}{\mathcal{P}^s_{\mathcal{C}}(\theta^s)} u_+ + 
(\alpha^s-1)\frac{G^s_0(\theta^s)}{\mathcal{P}^s_{\mathcal{C}}(\theta^s)}u_-].
\end{eqnarray*}
According to Equation \eqref{eq:opt_fair_policy}, we have $p_af^a(\theta^a, \alpha^a, \alpha^b) + p_bf^b(\theta^b, \alpha^a, \alpha^b) = 0$.

Given any $\alpha^a$ and $\alpha^b$,
and any constant $k$, let $\Tilde{\theta}_i^s$ be one solution to $f^s(\theta^s, \alpha^a, \alpha^b) = k$, where $i = 1,...,N$ and $N$ is the number of solutions.
Firstly, we show that $\Tilde{\theta}_i^s(\alpha^a, \alpha^b)$ 
is continuous in $\alpha^a$ and $\alpha^b$, for any $i \in\{ 1,...,N\}$.
Because $\frac{G^s_1(x)}{\mathcal{P}^s_{\mathcal{C}}(x)}$ and $\frac{G^s_0(x)}{\mathcal{P}^s_{\mathcal{C}}(x)}$ are continuous, $f^s(\theta^s, \alpha^a, \alpha^b)$ is continuous in  $\alpha^a$, $\alpha^b$, and $\theta^s$. 
Therefore, $\forall \epsilon > 0$, $\exists \delta >0$ such that for all $|\alpha^{a'} - \alpha^a|<\delta$ and $|\alpha^{b'} - \alpha^b|<\delta$ 
$\Longrightarrow$ $|\Tilde{\theta}^{s'}_i-\Tilde{\theta}^s_i|<\epsilon$. Thus, $\Tilde{\theta}_i^s(\alpha^a, \alpha^b)$ 
is continuous in $\alpha^a$ and $\alpha^b$
, $\forall i \in \{1,...,N\}$.

Next, we show that given $\alpha^a$ and $\alpha^b$, the solutions to $p_af^a(\theta^a, \alpha^a, \alpha^b) + p_bf^b(\theta^b, \alpha^a, \alpha^b) = 0$ under fairness constraint $\mathcal{C}$ are continuous in $\alpha^a$ and $\alpha^b\in[0,1]$.
Under fairness constraints in Equation \eqref{eq:fair_constraint_e}, 
$\theta^a = \phi_\mathcal{C}(\theta^b)$ 
holds for some continuous function $\phi_\mathcal{C}(\cdot)$. 
Consequently, we have $p_af^a(\phi_\mathcal{C}(\theta^b), \alpha^a, \alpha^b) + p_bf^b(\theta^b, \alpha^a, \alpha^b) = 0$. Because $f^s(\cdot,\cdot,\cdot)$ and $\phi_\mathcal{C}(\cdot)$ are continuous functions, 
with the same reasoning,
given $\alpha^a$ and $\alpha^b$, the solutions to $p_af^a(\phi_\mathcal{C}(\theta^b), \alpha^a, \alpha^b) + p_bf^b(\theta^b, \alpha^a, \alpha^b) = 0$ are continuous in $\alpha^a$ and $\alpha^b$. In other words, $\theta_i^s(\alpha^a, \alpha^b)$ is continuous. 

Under Assumption 
\ref{ass:fair_prob}, 
$f^s(\theta^s, \alpha^a, \alpha^b)$ and 
$\theta^s(\alpha^a, \alpha^b)$
are continuous. 
We then prove that
if $\frac{G^s_1(x)}{\mathcal{P}^s_{\mathcal{C}}(x)}$ is non-decreasing and $\frac{G^s_0(x)}{\mathcal{P}^s_{\mathcal{C}}(x)}$ is non-increasing in $x$, then $\theta^s(\alpha^a,\alpha^b)$ is non-increasing in $\alpha^a$ and $\alpha^b$.

Let $(\phi_\mathcal{C}(\theta^b),\theta^b)$ be a pair that satisfies fairness constraint, where $\phi_\mathcal{C}(\cdot)$ is some continuous and strictly increasing function, then the optimal one is the pair that satisfies Equation \eqref{eq:opt_fair_policy} as follows:
\begin{eqnarray*}
\resizebox{0.9\hsize}{!}{$p_a (\gamma^a(\phi_\mathcal{C}(\theta^b))-\frac{u_-}{u_++u_-})\frac{\mathbb{P}(X=\phi_\mathcal{C}(\theta^b) \mid S=a)}{\mathcal{P}^a_{\mathcal{C}}(\phi_\mathcal{C}(\theta^b))} + p_b (\gamma^b(\theta^b)-\frac{u_-}{u_++u_-})\frac{\mathbb{P}(X=\theta^b \mid S=b)}{\mathcal{P}^b_{\mathcal{C}}(\theta^b)}$}
\\\resizebox{0.97\hsize}{!}{$= p_a\Big[
\alpha^a \frac{G^a_1(\phi_\mathcal{C}(\theta^b))}{\mathcal{P}^a_{\mathcal{C}}(\phi_\mathcal{C}(\theta^b))} u_+ + 
(\alpha^a-1)\frac{G^a_0(\phi_\mathcal{C}(\theta^b))}{\mathcal{P}^a_{\mathcal{C}}(\phi_\mathcal{C}(\theta^b))}u_-\Big] +  p_b\Big[
\alpha^b \frac{G^b_1(\theta^b)}{\mathcal{P}^b_{\mathcal{C}}(\theta^b)} u_+ + 
(\alpha^b-1)\frac{G^b_0(\theta^b)}{\mathcal{P}^b_{\mathcal{C}}(\theta^b)}u_-\Big] = 0.$}
\end{eqnarray*}
Note that $\forall s\in \{a,b\}$, LHS of above equation is strictly increasing in $\alpha^s$ because the coefficient of $\alpha^s$ is positive. Because $\frac{G^s_1(x)}{\mathcal{P}^s_{\mathcal{C}}(x)}$ is non-decreasing and $\frac{G^s_0(x)}{\mathcal{P}^s_{\mathcal{C}}(x)}$ is non-increasing in $x$, $\frac{G^s_1(x)}{\mathcal{P}^s_{\mathcal{C}}(x)}-\frac{G^s_0(x)}{\mathcal{P}^s_{\mathcal{C}}(x)}$ is non-decreasing in $x$. As $\alpha^s$ increases, both $\frac{G^a_1(\phi_\mathcal{C}(\theta^b))}{\mathcal{P}^a_{\mathcal{C}}(\phi_\mathcal{C}(\theta^b))}-\frac{G^a_0(\phi_\mathcal{C}(\theta^b))}{\mathcal{P}^a_{\mathcal{C}}(\phi_\mathcal{C}(\theta^b))}$ and $\frac{G^b_1(\theta^b)}{\mathcal{P}^b_{\mathcal{C}}(\theta^b)}-\frac{G^b_0(\theta^b)}{\mathcal{P}^b_{\mathcal{C}}(\theta^b)}$ must not increase so that the optimal fair equation can be maintained. It requires that both $\theta^b$ and $\theta^a = \phi_{\mathcal{C}}(\theta^b)$ must not increase. In other words,  $\forall s\in \{a,b\}$, $\theta^s(\alpha^a,\alpha^b)$ must be non-increasing in $\alpha^a$ and $\alpha^b$.
\end{proof}

\end{proof}
\paragraph{The proof of Lemma \ref{lemma:opt_fair_policy_eq}.}
\begin{proof}
In the following proof, we focus on optimal policy at $t$ and omit the subscript $t$. 

First consider unconstrained optimal policy. Under threshold policy,
\begin{eqnarray*}
\theta^{s*}_{\texttt{UN}} &=& \arg\max_{\theta^s}\mathbb{E}_{X\mid S=s}[\pi^s(X)(\gamma^s(X)(u_++u_-)-u_-)] \\&=& \arg\max_{\theta^s} \int_{\theta^s}^{\infty}(\gamma^s(x)(u_++u_-)-u_-)\mathbb{P}(X=x\mid S=s) dx
\end{eqnarray*}
 Since $\gamma^s(x)$ is monotonically increasing in $x$ under Assumption \ref{ass:mono-inference}, $\theta^{s*}_{\texttt{UN}}$ satisfies $\gamma^s(\theta^{s*}_{\texttt{UN}}) = \frac{u_-}{u_++u_-}$.

Now consider optimal policy under fairness constraint, to satisfy constraint $\mathcal{C}$,  $\int_{\theta^{a}}^{\infty}\mathcal{P}^a_{\mathcal{C}}(x)dx = \int_{\theta^{b}}^{\infty}\mathcal{P}^b_{\mathcal{C}}(x)dx$ should hold. Denote CDF $\mathbb{P}_{\mathcal{C}}^s(\theta^s) = \int_{-\infty}^{\theta^{s}}\mathcal{P}^s_{\mathcal{C}}(x)dx$, then for any pair $(\theta^a,\theta^b)$ that is fair, we have $
\theta^a 
= (\mathbb{P}^a_{\mathcal{C}})^{-1} \mathbb{P}_{\mathcal{C}}^b(\theta^b)
=\phi_{\mathcal{C}}(\theta^b)$ hold for some strictly increasing function $\phi_{\mathcal{C}}(\cdot)$. Denote $u = \mathbb{P}^b_{\mathcal{C}}(\theta^b)$ and $\theta^a = (\mathbb{P}_{\mathcal{C}}^a)^{-1}(u)$, the following holds, 
\begin{eqnarray*}
\resizebox{0.99\hsize}{!}{$\frac{d\phi_{\mathcal{C}}(\theta^b)}{d \theta^b}  
= \frac{d (\mathbb{P}_{\mathcal{C}}^a)^{-1}\mathbb{P}_{\mathcal{C}}^b(\theta^b)}{d\theta^b}
= \frac{d(\mathbb{P}_{\mathcal{C}}^a)^{-1}(u)}{du}
\frac{du}{d\theta^b}
= \frac{1}{(\mathbb{P}_{\mathcal{C}}^a)'((\mathbb{P}_{\mathcal{C}}^a)^{-1}(u))}\frac{du}{d\theta^b}
=\frac{(\mathbb{P}_{\mathcal{C}}^b)'( \theta^b)}{(\mathbb{P}_{\mathcal{C}}^a)'( \theta^a)} = \frac{\mathcal{P}_{\mathcal{C}}^b(\theta^b)}{\mathcal{P}_{\mathcal{C}}^a(\theta^a)}.$}
\end{eqnarray*}

Denote $f^s(x) \coloneqq ({\gamma}^s(x)(u_++u_-) -u_-)\mathbb{P}(X=x\mid S=s)$, then we have
\begin{eqnarray*}
\theta^{b*}_{\mathcal{C}}  = 
\arg\max_{\theta^{b} } {\mathcal{U}}(D,Y) 
= \arg\max_{\theta^{b} }
\left(
p_a
\int_{\phi_{\mathcal{C}}(\theta^{b} )}^{\infty}{f}^a(x)dx  + 
p_b
\int_{\theta^{b} }^{\infty} {f}^b(x)dx\right).
\end{eqnarray*}

Let $F(\theta^{b} ) \coloneqq p_a
\int_{\phi_{\mathcal{C}}(\theta^{b} )}^{\infty}{f}^a(x)dx  + 
p_b
\int_{\theta^{b} }^{\infty} {f}^b(x)dx$.
Because $\gamma^s(x)$ is monotonically increasing in $x$ under Assumption \ref{ass:mono-inference}, the optimal $\theta^{b*}_{\mathcal{C}} $ satisfies
\begin{eqnarray*}
\frac{dF(\theta^{b} )}{d\theta^{b} }\Big|_{\theta^{b} =\theta^{b*}_{\mathcal{C}} } 
&=& -p_a  f^a(\phi_{\mathcal{C}}(\theta^{b} )) \frac{d\phi_{\mathcal{C}}(\theta^{b} )}{d\theta^{b} } 
-p_b  f^b(\theta^{b} )
\Big|_{\theta^{b} =\theta^{b*}_{\mathcal{C}} }\\
&=& -p_a  
({\gamma}^a(\phi_{\mathcal{C}}(\theta^{b*}_{\mathcal{C}} ))(u_++u_-) -u_-)\mathbb{P}(X=\phi_{\mathcal{C}}(\theta^{b*}_{\mathcal{C}} )\mid S=a) \frac{\mathcal{P}_{\mathcal{C}}^b(\theta^{b*}_{\mathcal{C}})}{\mathcal{P}_{\mathcal{C}}^a(\phi_{\mathcal{C}}(\theta^{b*}_{\mathcal{C}} ))}\\
&&-p_b  
({\gamma}^b(\theta^{b*}_{\mathcal{C}} )(u_++u_-) -u_-)\mathbb{P}(X=\theta^{b*}_{\mathcal{C}} \mid S=b) \\
&=&0.
\end{eqnarray*}
Therefore,
\begin{eqnarray*}
\resizebox{0.99\hsize}{!}{$p_a  
({\gamma}^a(\theta^{a*}_{\mathcal{C}} )(u_++u_-) -u_-)
\frac{\mathbb{P}(X=\theta^{a*}_{\mathcal{C}} \mid S=a) }{\mathcal{P}_{\mathcal{C}}( \theta^{a*}_{\mathcal{C}} )}
+p_b  
({\gamma}^b(\theta^{b*}_{\mathcal{C}} )(u_++u_-) -u_-)
\frac{\mathbb{P}(X=\theta^{b*}_{\mathcal{C}} \mid S=b) }
{\mathcal{P}_{\mathcal{C}}( \theta^{b*}_{\mathcal{C}} )}
=0.$}
\end{eqnarray*}

\end{proof}
\paragraph{ The proof of Theorem \ref{lemma:exist_equilibrium}.}
\begin{proof}
$\forall s\in \{a,b\}$, define function $l^s(\alpha^s)\coloneqq \frac{1}{\alpha^s} - 1$ and $h^s(\theta^s(\alpha^a,\alpha^b)) \coloneqq \frac{1-g^{1s}(\theta^s(\alpha^a,\alpha^b))}{g^{0s}(\theta^s(\alpha^a,\alpha^b))}$,
\begin{eqnarray*}
 h^s(\theta^s(\alpha^a,\alpha^b)) = \frac{1 - ( T_{10}^s\mathbb{G}_1^s(\theta^s(\alpha^a,\alpha^b)) + T_{11}^s\big(1-\mathbb{G}_1^s(\theta^s(\alpha^a,\alpha^b))\big))}{ T_{00}^s\mathbb{G}_0^s(\theta^s(\alpha^a,\alpha^b)) + T_{01}^s\big(1-\mathbb{G}_0^s(\theta^s(\alpha^a,\alpha^b))\big)}.
\end{eqnarray*}

Firstly, we prove that given a fixed $\alpha^{-s}\in[0,1]$ there must exist at least one $\overline{\alpha}^{s}\in (0,1)$ such that $h^{s}(\theta^{s}(\alpha^{-s},\overline{\alpha}^{s})) = l^{s}(\overline{\alpha}^{s})$, $s\in \{a,b\}$, $-s = \{a,b\}\setminus s$.

Since $\mathbb{G}_y^s(x)$ is continuous in $x$, and $\theta^s(\alpha^a,\alpha^b)$ is continuous in $\alpha^a$ and $\alpha^b$, $\mathbb{G}_y^s(\theta^s(\alpha^a,\alpha^b))$ is continuous in $\alpha^a$ and $\alpha^b$. Therefore,  $h^s(\theta^s(\alpha^a,\alpha^b))$ is continuous in $\alpha^a$ and $\alpha^b$.

Moreover, $g^{1s}(\theta^s(\alpha^a,\alpha^b))$ is the convex combination of $T_{11}^s$ and $T_{10}^s$, and $g^{0s}(\theta^s(\alpha^a,\alpha^b))$ is the convex combination of $T_{01}^s$ and $T_{00}^s$, the following holds $\forall \alpha^a\in [0,1], \alpha^b\in [0,1]$,
\begin{eqnarray*}
\min\{T_{10}^s,T_{11}^s \}\leq g^{1s}(\theta^s(\alpha^a,\alpha^b)) \leq \max\{T_{10}^s,T_{11}^s \}~;\\
\min\{T_{00}^s,T_{01}^s \}\leq g^{0s}(\theta^s(\alpha^a,\alpha^b)) \leq \max\{T_{00}^s,T_{01}^s \}~,
\end{eqnarray*}
which implies $ 0 < \frac{1-\max\{T_{10}^s,T_{11}^s \}}{\max\{T_{00}^s,T_{01}^s \}}\leq h^s(\theta^s(\alpha^a,\alpha^b)) \leq \frac{1-\min\{T_{10}^s,T_{11}^s \}}{\min\{T_{00}^s,T_{01}^s \}}< 
+\infty.$ 

Furthermore, $l^s(\alpha^s)\coloneqq \frac{1}{\alpha^s} - 1$ is continuous and strictly decreasing in $\alpha^s$, and $$\lim_{\alpha^s \to 0} l^s(\alpha^s)= + \infty; ~~\lim_{\alpha^s \to 1} l^s(\alpha^s)= 0,$$ 
Given a fixed $\alpha^a\in [0,1]$,
because $h^b(\theta^b(\alpha^a,\alpha^b))$ is continuous over $\alpha^b \in [0,1]$ and with value varying between $ \frac{1-\max\{T_{10}^b,T_{11}^b \}}{\max\{T_{00}^b,T_{01}^b \}}$ and $\frac{1-\min\{T_{10}^b,T_{11}^b \}}{\min\{T_{00}^b,T_{01}^b \}}$, and $l^b(\alpha^b)$ is continuous with value varying from $+\infty$ to 0, there must exist at least one $\overline{\alpha}^b\in (0,1)$ such that $h^b(\theta^b(\alpha^a,\overline{\alpha}^b)) = l^b(\overline{\alpha}^b)$.
Similarly, given a fixed $\alpha^b\in [0,1]$, there must exist at least one $\overline{\alpha}^a\in (0,1)$ such that $h^a(\theta^a(\overline{\alpha}^a,\alpha^b)) = l^a(\overline{\alpha}^a)$. 

Secondly, we prove that all the solutions $(\overline{\alpha}^a,\alpha^b)$ and $(\alpha^a,\overline{\alpha}^b)$ are on continuous curves in the 2D plane $\{(\alpha^a,\alpha^b): \alpha^a\in [0,1], \alpha^b\in [0,1]\}$. 

According to the continuity of $l^s(\cdot)$ and $h^s(\cdot)$, we have $\forall \alpha^a \in [0,1]$, $\lim_{\alpha^{a'}\to\alpha^a} l^{a}(\alpha^{a'})= l^{a}(\alpha^a)$; furthermore, $\forall \alpha^a \in [0,1]$ and $\forall \theta_i^a \in \{\theta^a: \; l^a(\alpha^a) = h^a(\theta^a)\}$, $\lim_{\theta^{a'}_i \to\theta^a_i} h^{a}(\theta^{a'}_i)= h^a(\theta^a_i)$. 
Thus, $\forall \epsilon >0$, $\exists \delta >0$, such that $\forall \alpha^a \in [0,1]$, $|\alpha^{a'}-\alpha^a| < \delta$ $\Longrightarrow$  $|\theta^{a'}_i - \theta^a_i |< \epsilon$. 
Consequently, $\forall \epsilon>0,$ $\exists \delta' >0$ and $\exists \delta >0$, such that $\forall \alpha^a \in [0,1]$, $|\alpha^{a'} - \alpha^a|<\delta$ $\Longrightarrow$  $|\theta^{a'}_i - \theta_i^a|<\delta'$ $\Longrightarrow$ $|\alpha^{b'}_i-\alpha^b_i|<\epsilon$, the last statement is because of the continuity of $\theta^{a}(\alpha^a,\alpha^b)$; 
in other words, $\forall \alpha^a \in [0,1]$, $\lim_{\alpha^{a'}\to\alpha^a} \alpha^{b'}_i = \alpha^b_i$, where $i=1,...,N$. Therefore, $(\overline{\alpha}^a,\alpha^b)$ is on a set of continuous curves with $\alpha^b$ varying from 0 to 1. Similarly, one can prove that $(\alpha^a,\overline{\alpha}^b)$ is also on a set of continuous curves with $\alpha^a$ varying from 0 to 1.

Finally, we show the existence of equilibrium $(\widehat{\alpha}^a,\;\widehat{\alpha}^b)$. 

Consider a 2D plane $\{(\alpha^a,\alpha^b): \alpha^a\in [0,1], \alpha^b\in [0,1]\}$, 
 and $\mathcal{C}_1 =\{(\overline{\alpha}^a,\alpha^b)\}$ and $\mathcal{C}_2 = \{(\alpha^a,\overline{\alpha}^b)\}$
that are two sets of continuous curves in the plane defined earlier.
It is straightforward to see that
there is at least one curve among $\mathcal{C}_1$ whose $\alpha^b$ varies from 0 to 1 and at least one curve among $\mathcal{C}_2$ whose $\alpha^a$ varies from 0 to 1. These two continuous curves must have at least one intersection. Moreover,
this intersection $(\widehat{\alpha}^a,\widehat{\alpha}^b)$ satisfies $h^b(\theta^b(\widehat{\alpha}^a,\widehat{\alpha}^b)) = l^b(\widehat{\alpha}^b)$ and $h^a(\theta^a(\widehat{\alpha}^a,\widehat{\alpha}^b)) = l^a(\widehat{\alpha}^a)$, is an equilibrium of system. 

Moreover, we also realized that the proof can also be done by using Brouwer's Fixed Point Theorem in topology.
\end{proof} 
\paragraph{The proof of Theorem \ref{thm:unique_equilibrium}.}
\begin{proof}
Following the proof of Theorem \ref{lemma:exist_equilibrium}, 
\begin{eqnarray*}
 h^s(\theta^s(\alpha^a,\alpha^b)) =\frac{1-g^{1s}(\theta^s(\alpha^a,\alpha^b))}{g^{0s}(\theta^s(\alpha^a,\alpha^b))}= \frac{1 - ( T_{10}^s\mathbb{G}_1^s(\theta^s(\alpha^a,\alpha^b)) + T_{11}^s\big(1-\mathbb{G}_1^s(\theta^s(\alpha^a,\alpha^b))\big))}{ T_{00}^s\mathbb{G}_0^s(\theta^s(\alpha^a,\alpha^b)) + T_{01}^s\big(1-\mathbb{G}_0^s(\theta^s(\alpha^a,\alpha^b))\big)}.
\end{eqnarray*}

Note that $\forall y\in \{0,1\}$, $ T_{y0}^s\mathbb{G}_y^s(\theta^s(\alpha^a,\alpha^b)) + T_{y1}^s\big(1-\mathbb{G}_y^s(\theta^s(\alpha^a,\alpha^b))\big)$ is the convex combination of $T_{y0}^s$ and $T_{y1}^s$ with CDF $\mathbb{G}_y^s(\theta^s(\alpha^a,\alpha^b))$ as the weight. Because $\mathbb{G}_y^s(\theta^s(\alpha^a,\alpha^b))$ is continuous and non-decreasing in $\theta^s(\alpha^a,\alpha^b)$, 
under Condition \ref{con:transition}\ref{con:transition_I},  $h^s(\theta^s(\alpha^a,\alpha^b))$ is non-decreasing in $\theta^s(\alpha^a,\alpha^b)$; while under Condition \ref{con:transition}\ref{con:transition_II},  $h^s(\theta^s(\alpha^a,\alpha^b))$ is non-increasing in $\theta^s(\alpha^a,\alpha^b)$.

Under unconstrained optimal policy or optimal fair policy with constraint $\mathcal{C}$ satisfying Assumption \ref{ass:mono-inference} and \ref{ass:fair_prob}, $\theta^s(\alpha^a,\alpha^b)$ is non-increasing in $\alpha^a,\alpha^b$. 
Therefore, under Condition \ref{con:transition}\ref{con:transition_I}, $h^s(\theta^s(\alpha^a,\alpha^b))$ is non-decreasing in $\alpha^a,\alpha^b$, while under Condition \ref{con:transition}\ref{con:transition_II},  $h^s(\theta^s(\alpha^a,\alpha^b))$ is non-increasing in $\alpha^a,\alpha^b$. Moreover, 
\begin{eqnarray*}
\text{Under Condition \ref{con:transition}\ref{con:transition_I}:    }~~~~~~ 0< \frac{1-T^s_{10}}{T^s_{00}} \leq h^s(\theta^s(\alpha^a,\alpha^b)) \leq \frac{1-T^s_{11}}{T^s_{01}} < +\infty\\
\text{Under Condition \ref{con:transition}\ref{con:transition_II}:    }~~~~~~ 0< \frac{1-T^s_{11}}{T^s_{01}} \leq h^s(\theta^s(\alpha^a,\alpha^b)) \leq \frac{1-T^s_{10}}{T^s_{00}} < +\infty
\end{eqnarray*}

First consider the case when Condition \ref{con:transition}\ref{con:transition_I} is satisfied.

Because function $l^s(\alpha^s)\coloneqq \frac{1}{\alpha^s} - 1$ is continuous and strictly decreasing from $+\infty$ to 0 over $\alpha^s\in [0,1]$, $\forall s\in \{a,b\}$. Thus, given any fixed $\alpha^{b}\in [0,1]$, strictly decreasing function $l^a(\alpha^a)$ and non-decreasing function $h^a(\theta^a(\alpha^a,\alpha^b))$ has exactly one intersection, i.e., $\exists$ only one $\overline{\alpha}^a$ such that  $h^a(\theta^a(\overline{\alpha}^a,\alpha^b)) = l^a(\overline{\alpha}^a)$. 
 $\forall \alpha^b$, the set $\Psi^a(\alpha^b) = \{\overline{\alpha}^a: h^a(\theta^a(\overline{\alpha}^a,\alpha^b)) = l^a(\overline{\alpha}^a)\}$ has only one element, and they constitute continuous function $\overline{\alpha}^a = \psi^a(\alpha^b)$ (balanced function). Similarly, $\forall \alpha^a$, set $\Psi^b(\alpha^a) = \{\overline{\alpha}^b: h^b(\theta^b(\alpha^a,\overline{\alpha}^b)) = l^b(\overline{\alpha}^b)\}$ also has only one element,  which forms continuous function $\overline{\alpha}^b = \psi^b(\alpha^a)$.

Because given any $\alpha^a$, $h^a(\theta^a(\alpha^a,\alpha^b))$ is non-decreasing in $\alpha^b$, as  $\alpha^b$ increases, the intersection with $l^a(\overline{\alpha}^a)$ is non-increasing. Therefore, $\psi^a(\alpha^b)$ is non-increasing in $\alpha^b$. Similarly,    $\psi^b(\alpha^a)$ is also non-increasing in $\alpha^a$. 

On the 2D plane  $\{(\alpha^a,\alpha^b): \alpha^a\in [0,1], \alpha^b\in [0,1]\}$, two curves $\mathcal{C}_1 = \{(\alpha^a,\alpha^b): \alpha^a = \psi^a(\alpha^b),\alpha^b\in [0,1]\}$ and $\mathcal{C}_2 = \{(\alpha^a,\alpha^b): \alpha^b = \psi^b(\alpha^a),\alpha^a\in [0,1]\}$ are both continuous and non-increasing. One sufficient condition to guarantee $\mathcal{C}_1$ and $\mathcal{C}_2$ have exact one intersection, is that $|\frac{d \psi^a(\alpha^b)}{d \alpha^b}| < 1, \forall \alpha^b\in[0,1]$ and  $|\frac{d \psi^b(\alpha^a)}{d \alpha^a}| < 1,\forall \alpha^a\in[0,1]$. In the followings, we show these sufficient conditions will hold if $|\frac{\partial h^a(\theta^a(\alpha^a,\alpha^b))}{\partial \alpha^{b}}| < 1$ and $|\frac{\partial h^b(\theta^b(\alpha^a,\alpha^b))}{\partial \alpha^{a}}| < 1, \forall \alpha^a,\alpha^b$.

Denote $u \coloneqq  h^a(\theta^a({\psi^a(\alpha^b)},\alpha^b))$, {because $l^a(\psi^a(\alpha^b)) =  h^a(\theta^a(\psi^a(\alpha^b),\alpha^b)),\forall \alpha^b$,}
\begin{eqnarray*}
\frac{d \psi^a(\alpha^b)}{d \alpha^b}
&=& \frac{d (l^a)^{-1}(u)}{d\alpha^b}
= \frac{d(l^a)^{-1}(u)}{d u}\frac{d u}{d\alpha^b}
=  \frac{1}{(l^a)'((l^a)^{-1}(u))}\frac{du}{d\alpha^b}
=  -((l^a)^{-1}(u))^2\frac{du}{d\alpha^b}.
\end{eqnarray*}
Because $(l^a)^{-1}(u) = {\psi^a(\alpha^b)}\in [0,1]$, $-((l^a)^{-1}(u))^2 \in [-1,0]$. Moreover, because of the condition $|\frac{d h^a(\theta^a(\alpha^a,\alpha^b))}{d \alpha^b}| < 1$, we have $$\Big|\frac{d \psi^a(\alpha^b)}{d \alpha^b}\Big|<1.$$

Similarly, we can show that $|\frac{d \psi^b(\alpha^a)}{d \alpha^a}| < 1$ holds $\forall \alpha^a$ if $|\frac{\partial h^b(\theta^b(\alpha^a,\alpha^b))}{\partial \alpha^{a}} |< 1$. Therefore, $\mathcal{C}_1$, $\mathcal{C}_2$ have only one intersection, the equilibrium $(\widehat{\alpha}^a,\widehat{\alpha}^b)$ is unique. 

Now consider the case when Condition \ref{con:transition}\ref{con:transition_II} is satisfied. 

Because $\frac{d l^s(\alpha^s)}{d \alpha^s} =- \frac{1}{(\alpha^s)^2} < -1, \forall \alpha^s \in (0,1)$, and $-1 
\leq \frac{\partial h^s(\theta^s(\alpha^a,\alpha^b))}{\partial\alpha^s} \leq 0$ for any fixed $\alpha^{-s}\in [0,1]$. Strictly decreasing function $l^s(\alpha^s)$ and non-increasing function $h^s(\theta^s(\alpha^a,\alpha^b))$ has exactly one intersection. Therefore, $\forall \alpha^b$, balanced set $\Psi^a(\alpha^b) = \{\overline{\alpha}^a: h^a(\theta^a(\overline{\alpha}^a,\alpha^b)) = l^a(\overline{\alpha}^a)\}$ has only one element, and they constitute continuous function $\overline{\alpha}^a = \psi^a(\alpha^b)$ (balanced function). Similarly, $\forall \alpha^b$, set $\Psi^a(\alpha^b) = \{\overline{\alpha}^a: h^a(\theta^a(\overline{\alpha}^a,\alpha^b)) = l^a(\overline{\alpha}^a)\}$ also has only one element, which forms continuous function $\overline{\alpha}^a = \psi^a(\alpha^b)$.  

Because given any $ \alpha^a$, $h^a(\theta^a(\alpha^a,\alpha^b))$ is non-increasing in $\alpha^b$. As $\alpha^b$ increases, the intersection with $l^a(\overline{\alpha}^a)$ is non-decreasing. Therefore, $\psi^a(\alpha^b)$ is non-decreasing in $\alpha^b$. Similarly,    $\psi^b(\alpha^a)$ is also non-decreasing in $\alpha^a$. 

On the 2D plane  $\{(\alpha^a,\alpha^b): \alpha^a\in [0,1], \alpha^b\in [0,1]\}$, two curves $\mathcal{C}_1 = \{(\alpha^a,\alpha^b): \alpha^a = \psi^a(\alpha^b),\alpha^b\in [0,1]\}$ and $\mathcal{C}_2 = \{(\alpha^a,\alpha^b): \alpha^b = \psi^b(\alpha^a),\alpha^a\in [0,1]\}$ are both continuous and non-decreasing. One sufficient condition to guarantee $\mathcal{C}_1$ and $\mathcal{C}_2$ have exact one intersection, is that $\frac{d \psi^a(\alpha^b)}{d \alpha^b} < 1, \forall \alpha^b\in[0,1]$ and  $\frac{d \psi^b(\alpha^a)}{d \alpha^a} < 1,\forall \alpha^a\in[0,1]$. {Using the same analysis as the case under Condition \ref{con:transition}\ref{con:transition_I}, 
we can show these sufficient conditions will hold if $\mid \frac{\partial h^a(\theta^a(\alpha^a,\alpha^b))}{\partial \alpha^{b}} \mid < 1$ and $\mid \frac{\partial h^b(\theta^b(\alpha^a,\alpha^b))}{\partial \alpha^{a}}\mid < 1, \forall \alpha^a,\alpha^b$.}

Therefore, $\mathcal{C}_1$, $\mathcal{C}_2$ have only one intersection, the equilibrium $(\widehat{\alpha}^a,\widehat{\alpha}^b)$ is unique. 
\end{proof}

\paragraph{The proof of Corollary \ref{rmk:unique}.}\label{app:corollary1}

 \begin{corollary}\label{rmk:unique}
 For any feature distribution $\{G^s_y(x)\}_{s,y}$, suppose that $\big |\frac{\partial \mathbb{G}_y^s(\theta^s(\alpha^a,\alpha^b))}{\partial \alpha^u} \big| \leq M_y$ holds for some constant $M_y\in [0,\infty)$, $\forall y \in \{0,1\}, \forall u\in \{a,b\}$. Under either Condition \ref{con:transition}\ref{con:transition_I} or \ref{con:transition}\ref{con:transition_II}, $\exists \epsilon^s_y>0$ such that for any transitions that satisfy $|T^s_{y1}- T^s_{y0}|< \epsilon^s_y$, $s\in \{a,b\}, y\in \{0,1\}$, the corresponding dynamics system has a unique equilibrium.
\end{corollary}

\begin{proof}
Define notations $ \mathbb{G}_y^s \coloneqq \mathbb{G}_y^s(\theta^s(\alpha^a,\alpha^b))$, $\Delta T_0^s \coloneqq T_{01}^s-T_{00}^s $ and $\Delta T_1^s \coloneqq T_{11}^s-T_{10}^s$. 
\begin{eqnarray*}
h^s(\theta^s(\alpha^a,\alpha^b)) = \frac{ ( 1-T_{10}^s)\mathbb{G}_1^s+ (1-T_{11}^s)\big(1-\mathbb{G}_1^s\big)}{ T_{00}^s\mathbb{G}_0^s+ T_{01}^s\big(1-\mathbb{G}_0^s\big)}
= \frac{ ( 1-T_{11}^s) + \Delta T^s_1\mathbb{G}_1^s}{ T_{00}^s+ \Delta T^s_0\big(1-\mathbb{G}_0^s\big)}
\end{eqnarray*}
Take derivative w.r.t. $\alpha^u$, $\forall u \in \{a,b\}$,
\begin{eqnarray*}
\frac{\partial h^s(\theta^s(\alpha^a,\alpha^b))}{\partial \alpha^u} &=& \frac{\Delta T^s_1\frac{\partial \mathbb{G}_1^s}{\partial \alpha^u}(T_{00}^s+ \Delta T^s_0\big(1-\mathbb{G}_0^s\big)) + \Delta T^s_0\frac{\partial \mathbb{G}_0^s}{\partial \alpha^u}(( 1-T_{11}^s) + \Delta T^s_1\mathbb{G}_1^s)}{(T_{00}^s+ \Delta T^s_0\big(1-\mathbb{G}_0^s\big))^2}
\end{eqnarray*}

Consider case under Condition \ref{con:transition}\ref{con:transition_I}. Since $\Delta T_0^s < 0$, $\Delta T_1^s < 0$, 
$T_{00}^s+ \Delta T^s_0\big(1-\mathbb{G}_0^s\big)>0$,
and $( 1-T_{11}^s) + \Delta T^s_1\mathbb{G}_1^s > 0$
, we have $|\frac{\partial h^s(\theta^s(\alpha^a,\alpha^b))}{\partial \alpha^u}|\leq \mid \frac{\Delta T^s_1M_1T_{00}^s + \Delta T^s_0M_0( 1-T_{11}^s)}{(T_{01}^s)^2}\mid $.

Take $\epsilon_1^s  = \epsilon_0^s =  \frac{(T_{01}^s)^2}{M_1T_{00}^s + M_0( 1-T_{11}^s)}$, if $|\Delta T^s_1| < \epsilon_1^s $ and  $|\Delta T^s_0| < \epsilon_0^s $, then $|\frac{\partial h^s(\theta^s(\alpha^a,\alpha^b))}{\partial \alpha^u}| < 1$ holds. From Theorem \ref{thm:unique_equilibrium}, the equilibrium of dynamics \ref{eq:dynamics} is unique. 

Consider case under Condition \ref{con:transition}\ref{con:transition_II}.

Since $\Delta T_0^s > 0$ and $\Delta T_1^s > 0$, we have $|\frac{\partial h^s(\theta^s(\alpha^a,\alpha^b))}{\partial \alpha^u}|\leq \frac{\Delta T^s_1M_1T_{01}^s + \Delta T^s_0M_0( 1-T_{10}^s)}{(T_{00}^s)^2}$. 

Take $\epsilon_1^s  = \epsilon_0^s =  \frac{(T_{00}^s)^2}{M_1T_{01}^s + M_0( 1-T_{10}^s)}$, if $\Delta T^s_1 < \epsilon_1^s $ and  $\Delta T^s_0 < \epsilon_0^s $, then $|\frac{\partial h^s(\theta^s(\alpha^a,\alpha^b))}{\partial \alpha^u}| < 1$ holds. From Theorem \ref{thm:unique_equilibrium}, the equilibrium of dynamics \ref{eq:dynamics} is unique.
\end{proof}

\paragraph{The proof of Theorem \ref{thm:general}.}
\begin{proof}
$\forall s\in\{a,b\}$, an equilibrium $\widehat{\alpha}^s_{\texttt{UN}}$ satisfies: 
$$\frac{1-g^{1s}(\theta_{\texttt{UN}}^s(\widehat{\alpha}^s_{\texttt{UN}}))}{g^{0s}(\theta_{\texttt{UN}}^s(\widehat{\alpha}^s_{\texttt{UN}}))} = \frac{1- \big(T_{11}^s(1-\mathbb{G}_1^s(\theta_{\texttt{UN}}^s(\widehat{\alpha}^s_{\texttt{UN}}))) + T_{10}^s\mathbb{G}_1^s(\theta_{\texttt{UN}}^s(\widehat{\alpha}^s_{\texttt{UN}}))\big)}{ T_{01}^s(1-\mathbb{G}_0^s(\theta_{\texttt{UN}}^s(\widehat{\alpha}^s_{\texttt{UN}})))+T_{00}^s\mathbb{G}_0^s(\theta_{\texttt{UN}}^s(\widehat{\alpha}^s_{\texttt{UN}})) }= \frac{1}{\widehat{\alpha}^s_{\texttt{UN}}} - 1.$$

One solution to the above equation is:
\begin{eqnarray*}
\widehat{\alpha}^s_{\texttt{UN}} = T_{11}^s(1-\mathbb{G}_1^s(\theta_{\texttt{UN}}^s(\widehat{\alpha}^s_{\texttt{UN}}))) + T_{10}^s\mathbb{G}_1^s(\theta_{\texttt{UN}}^s(\widehat{\alpha}^s_{\texttt{UN}})) = T_{01}^s(1-\mathbb{G}_0^s(\theta_{\texttt{UN}}^s(\widehat{\alpha}^s_{\texttt{UN}})))+T_{00}^s\mathbb{G}_0^s(\theta_{\texttt{UN}}^s(\widehat{\alpha}^s_{\texttt{UN}})) 
\end{eqnarray*}

It shows that $\widehat{\alpha}^s_{\texttt{UN}}$ is a convex combination of $T_{00}^s$, $T_{01}^s$, and also a convex combination of $T_{10}^s$, $T_{11}^s$. 

$\forall {\alpha}_{\texttt{UN}}$ and $\mathbb{G}_0^s(x) $, $\mathbb{G}_1^s(x)$, there is a set of transitions with $T_{00}^s< {\alpha}_{\texttt{UN}} < T_{01}^s$ and  $T_{10}^s< {\alpha}_{\texttt{UN}} < T_{11}^s$ (satisfy Condition \ref{con:transition}\ref{con:transition_II}), or $T_{01}^s< {\alpha}_{\texttt{UN}} < T_{00}^s$ and  $T_{11}^s< {\alpha}_{\texttt{UN}} < T_{10}^s$ (satisfy Condition \ref{con:transition}\ref{con:transition_I}), such that the above equation holds with $\widehat{\alpha}^s_{\texttt{UN}} = \alpha_{\texttt{UN}}$, $\forall s\in \{a,b\}$, i.e., equitable equilibrium is attained. 
 
Next we show that if $G^a_y(x)\neq G^b_y(x)$, then $\widehat{\alpha}_{\mathcal{C}}^b\neq \widehat{\alpha}_{\mathcal{C}}^a$ under these sets of transitions. Under the conditions of Theorem \ref{thm:unique_equilibrium}, 
$(\widehat{\alpha}^a_{\mathcal{C}},\widehat{\alpha}^b_{\mathcal{C}})$ is the intersection of two curves $\mathcal{C}_1 = \{(\alpha^a,\alpha^b): \alpha^a = \psi^a_{\mathcal{C}}(\alpha^b),\alpha^b\in [0,1]\}$ and $\mathcal{C}_2 = \{(\alpha^a,\alpha^b): \alpha^b = \psi^b_{\mathcal{C}}(\alpha^a),\alpha^a\in [0,1]\}$; furthermore, let $\widetilde{\alpha}^a_{\mathcal{C}}$, $\widetilde{\alpha}^b_{\mathcal{C}}$ be defined such that $\widetilde{\alpha}^a_{\mathcal{C}}= \psi^a_{\mathcal{C}}(\widetilde{\alpha}^a_{\mathcal{C}})$,  $\widetilde{\alpha}_{\mathcal{C}}^b= \psi^b_{\mathcal{C}}(\widetilde{\alpha}^b_{\mathcal{C}})$, which are the intersections of $\alpha^a = \psi^a_{\mathcal{C}}(\alpha^{b})$ and $\alpha^a=\alpha^{b}$, as well as $\alpha^b = \psi^b_{\mathcal{C}}(\alpha^{a})$ and $\alpha^a=\alpha^{b}$, respectively. Then in order to prove $\widehat{\alpha}_{\mathcal{C}}^b\neq \widehat{\alpha}_{\mathcal{C}}^a$, it is sufficient to show $\widetilde{\alpha}^a_{\mathcal{C}} \neq \widetilde{\alpha}^b_{\mathcal{C}}$.
 
Given $\alpha^a = \alpha^b = \alpha_{\texttt{UN}}$, because $G^a_y(x)\neq G^b_y(x)$, we have $\theta^s_{\texttt{UN}}(\alpha_{\texttt{UN}})\neq \theta^s_{\mathcal{C}}(\alpha_{\texttt{UN}},\alpha_{\texttt{UN}})$ and to satisfy Eqn. \eqref{eq:opt_fair_policy}, there are only two possibilities: (1) $\theta^a_{\texttt{UN}}(\alpha_{\texttt{UN}})> \theta^a_{\mathcal{C}}(\alpha_{\texttt{UN}},\alpha_{\texttt{UN}})$, $\theta^b_{\texttt{UN}}(\alpha_{\texttt{UN}})< \theta^b_{\mathcal{C}}(\alpha_{\texttt{UN}},\alpha_{\texttt{UN}})$; (2) $\theta^a_{\texttt{UN}}(\alpha_{\texttt{UN}})< \theta^a_{\mathcal{C}}(\alpha_{\texttt{UN}},\alpha_{\texttt{UN}})$, $\theta^b_{\texttt{UN}}(\alpha_{\texttt{UN}})> \theta^b_{\mathcal{C}}(\alpha_{\texttt{UN}},\alpha_{\texttt{UN}})$.

WLOG, suppose the first case holds. Under Condition \ref{con:transition}\ref{con:transition_II},
$$ \frac{1-g^{1b}(\theta_{\texttt{UN}}^b({\alpha}_{\texttt{UN}}))}{g^{0b}(\theta_{\texttt{UN}}^b({\alpha}_{\texttt{UN}}))} < \frac{1-g^{1b}( \theta_{\mathcal{C}}^b({\alpha}_{\texttt{UN}},\alpha_{\texttt{UN}}))}{g^{0b}( \theta_{\mathcal{C}}^b({\alpha}_{\texttt{UN}},\alpha_{\texttt{UN}}))}; ~~ \frac{1-g^{1a}(\theta_{\texttt{UN}}^a({\alpha}_{\texttt{UN}}))}{g^{0a}(\theta_{\texttt{UN}}^a({\alpha}_{\texttt{UN}}))} > \frac{1-g^{1a}( \theta_{\mathcal{C}}^a({\alpha}_{\texttt{UN}},\alpha_{\texttt{UN}}))}{g^{0a}( \theta_{\mathcal{C}}^a({\alpha}_{\texttt{UN}},\alpha_{\texttt{UN}}))}$$
It implies that $\widetilde{\alpha}^b_{\mathcal{C}}< \widehat{\alpha}^b_{\texttt{UN}} = \widehat{\alpha}^a_{\texttt{UN}} < \widetilde{\alpha}^a_{\mathcal{C}}$. Similarly, under Condition \ref{con:transition}\ref{con:transition_I}, $\widetilde{\alpha}^b_{\mathcal{C}}> \widehat{\alpha}^b_{\texttt{UN}} = \widehat{\alpha}^a_{\texttt{UN}} > \widetilde{\alpha}^a_{\mathcal{C}}$. Therefore, $\widehat{\alpha}^a_{\mathcal{C}}\neq \widehat{\alpha}^b_{\mathcal{C}} $.

In contrast, if $G^a_y(x) = G^b_y(x)$, we have $\theta^s_{\texttt{UN}}(\alpha)= \theta^s_{\mathcal{C}}(\alpha,\alpha)$ and $\widetilde{\alpha}^b_{\mathcal{C}}= \widetilde{\alpha}^a_{\mathcal{C}}$. Therefore, $\widehat{\alpha}^a_{\mathcal{C}}= \widehat{\alpha}^b_{\mathcal{C}}$.
\end{proof}
\paragraph{The proof of Theorem \ref{thm:comp_dp_eqopt}.}
\begin{proof}
WLOG, suppose that $\widehat{\alpha}^a_{\texttt{UN}} > \widehat{\alpha}^b_{\texttt{UN}}$
in the proof. Let $\psi^a_{\mathcal{C}}(\cdot),\psi^b_{\mathcal{C}}(\cdot)$ be balanced functions as defined in Theorem \ref{thm:unique_equilibrium} under constraint $\mathcal{C}$.
Firstly, we show that $\widehat{\alpha}^b_{\texttt{UN}} $ and $\widehat{\alpha}^a_{\texttt{UN}}$ are solutions to $\begin{cases}
\alpha^b = \psi^b_{\mathcal{C}}(\alpha^a)\\
\alpha^a = \alpha^b
\end{cases}$
and 
$\begin{cases}
\alpha^a = \psi^a_{\mathcal{C}}(\alpha^b)\\
\alpha^a = \alpha^b
\end{cases}$, respectively, i.e., $\widehat{\alpha}^b_{\texttt{UN}} = \psi^b_{\mathcal{C}}(\widehat{\alpha}^b_{\texttt{UN}})$ and $\widehat{\alpha}^a_{\texttt{UN}} = \psi_{\mathcal{C}}^a(\widehat{\alpha}^a_{\texttt{UN}})$. 

Because $G^a_y(x) = G^b_y(x)$, $\forall y\in \{0,1\}, \forall x$, when $\alpha^a = \alpha^b = \alpha$,
we have $\gamma^a(x) = \gamma^b(x)$, $\mathcal{P}^a_{\texttt{EqOpt}}(x)=\mathcal{P}^b_{\texttt{EqOpt}}(x)$ and $\mathcal{P}^a_{\texttt{DP}}(x) = \mathcal{P}^b_{\texttt{DP}}(x)$, which implies $\theta^a_{\mathcal{C}}(\alpha,\alpha)=\theta^b_{\mathcal{C}}(\alpha,\alpha)$; furthermore, the optimal fair policies of \texttt{DP} and \texttt{EqOpt} satisfy
${\gamma}^a(\theta^a_{\mathcal{C}}(\alpha,\alpha))={\gamma}^b(\theta^b_{\mathcal{C}}(\alpha,\alpha))=\frac{u_-}{u_++u_-}$ according to the optimal fair policy equation:
\begin{eqnarray*}
\frac{p_a  \alpha^a}{{\gamma}^a(\theta^{a}_{\texttt{EqOpt}})} + \frac{p_b  \alpha^b}{{\gamma}^b(\theta^{b}_{\texttt{EqOpt}})}
 = \frac{p_a  \alpha^a}{
 \frac{u_-}{u_++u_-}
 } + \frac{p_b  \alpha^b}{
 \frac{u_-}{u_++u_-}
 };~~~
p_a  
{\gamma}^a(\theta^{a}_{\texttt{DP}})
+p_b  
{\gamma}^b(\theta^{b}_{\texttt{DP}})
=\frac{u_-}{u_++u_-}.
\end{eqnarray*}
Because ${\gamma}^a(\theta^a_{\texttt{UN}}(\alpha))={\gamma}^b(\theta^b_{\texttt{UN}}(\alpha))=\frac{u_-}{u_++u_-}$ we have ${\gamma}^a(\theta^a_{\texttt{UN}}(\alpha))={\gamma}^a(\theta^a_{\mathcal{C}}(\alpha,\alpha))={\gamma}^b(\theta^b_{\texttt{UN}}(\alpha))={\gamma}^b(\theta^b_{\mathcal{C}}(\alpha,\alpha))$ so that $\theta^a_{\mathcal{C}}(\alpha,\alpha)=\theta^a_{\texttt{UN}}(\alpha)=\theta^b_{\mathcal{C}}(\alpha,\alpha)=\theta^b_{\texttt{UN}}(\alpha)$ holds under any $\alpha$. $\forall s\in \{a,b\}$, because $\widehat{\alpha}^s_{\texttt{UN}}$ is the solution to balanced equation, i.e., $l^s(\widehat{\alpha}^s_{\texttt{UN}}) = h^s(\theta^s_{\texttt{UN}}(\widehat{\alpha}^s_{\texttt{UN}}))$. We have  $l^s(\widehat{\alpha}^s_{\texttt{UN}}) = h^s(\theta^s_{\mathcal{C}}(\widehat{\alpha}^s_{\texttt{UN}},\widehat{\alpha}^s_{\texttt{UN}}))$, which further implies $\widehat{\alpha}^s_{\texttt{UN}} = \psi^s_{\mathcal{C}}(\widehat{\alpha}^s_{\texttt{UN}}) $.

Under Condition \ref{con:transition}\ref{con:transition_II}, according to the proof of Theorem \ref{thm:unique_equilibrium}, we know that $0\leq \frac{d\psi^b_{\mathcal{C}}(\alpha^a)}{d \alpha^a} < 1$ and $0\leq \frac{d\psi^a_{\mathcal{C}}(\alpha^b)}{d \alpha^b} < 1$. Because $\widehat{\alpha}^b_{\texttt{UN}} = \psi^b_{\mathcal{C}}(\widehat{\alpha}^b_{\texttt{UN}})<\widehat{\alpha}^a_{\texttt{UN}}=\psi^a_{\mathcal{C}}(\widehat{\alpha}^a_{\texttt{UN}})$, we have $\widehat{\alpha}^b_{\texttt{UN}}< \psi^b_{\mathcal{C}}(\alpha^a) < \alpha^a$, $\forall \alpha^a \in [\widehat{\alpha}^b_{\texttt{UN}},\widehat{\alpha}^a_{\texttt{UN}}]$. Similarly, we have
$\alpha^b< \psi_{\mathcal{C}}^a(\alpha^b) < \widehat{\alpha}^a_{\texttt{UN}}$, $\forall \alpha^b \in [\widehat{\alpha}^b_{\texttt{UN}},\widehat{\alpha}^a_{\texttt{UN}}]$. Therefore, after representing the two balanced functions as two curves $\mathcal{C}_1 = \{(\alpha^a,\alpha^b): \alpha^a = \psi^a_{\mathcal{C}}(\alpha^b),\alpha^b\in [0,1]\}$ and $\mathcal{C}_2 = \{(\alpha^a,\alpha^b): \alpha^b = \psi^b_{\mathcal{C}}(\alpha^a),\alpha^a\in [0,1]\}$ on the 2D plane  $\{(\alpha^a,\alpha^b): \alpha^a\in [0,1], \alpha^b\in [0,1]\}$, the intersection ($\widehat{\alpha}^a_{\mathcal{C}},\widehat{\alpha}^b_{\mathcal{C}}$) of $\mathcal{C}_1$ and $\mathcal{C}_2$ satisfies: 1) $\widehat{\alpha}^a_{\mathcal{C}} > \widehat{\alpha}^b_{\mathcal{C}}$; 2) $\widehat{\alpha}^b_{\texttt{UN}}<\widehat{\alpha}^a_{\mathcal{C}} < \widehat{\alpha}^a_{\texttt{UN}}$; 3) $\widehat{\alpha}^b_{\texttt{UN}}<\widehat{\alpha}^b_{\mathcal{C}} < \widehat{\alpha}^a_{\texttt{UN}}$. Therefore, $|\widehat{\alpha}^a_{\mathcal{C}}-\widehat{\alpha}^b_{\mathcal{C}}| \leq |\widehat{\alpha}^a_{\texttt{UN}}-\widehat{\alpha}^b_{\texttt{UN}}|$.

Under Condition \ref{con:transition}\ref{con:transition_I}, according to the proof of Theorem \ref{thm:unique_equilibrium}, we know that 
$-1 < \frac{d\psi^b_{\mathcal{C}}(\alpha^a)}{d \alpha^a} \leq 0$ and $-1 < \frac{d\psi^a_{\mathcal{C}}(\alpha^b)}{d \alpha^b} \leq 0$. 
Because $\widehat{\alpha}^b_{\texttt{UN}} = \psi^b_{\mathcal{C}}(\widehat{\alpha}^b_{\texttt{UN}})<\widehat{\alpha}^a_{\texttt{UN}}=\psi^a_{\mathcal{C}}(\widehat{\alpha}^a_{\texttt{UN}})$, we have $\psi_{\mathcal{C}}^b(\alpha^a)<\widehat{\alpha}^b_{\texttt{UN}}$, $\forall \alpha^a>\widehat{\alpha}^b_{\texttt{UN}}$. Similarly, we have
$\psi_{\mathcal{C}}^a(\alpha^b)>\widehat{\alpha}^a_{\texttt{UN}}$, $\forall \alpha^b<\widehat{\alpha}^a_{\texttt{UN}}$. Due to the existence of equilibrium, the intersection ($\widehat{\alpha}^a_{\mathcal{C}},\widehat{\alpha}^b_{\mathcal{C}}$) of $\mathcal{C}_1$ and $\mathcal{C}_2$ must satisfy: 1) $\widehat{\alpha}^a_{\mathcal{C}} > \widehat{\alpha}^b_{\mathcal{C}}$; 2) $\widehat{\alpha}^a_{\texttt{UN}}<\widehat{\alpha}^a_{\mathcal{C}}$; 
3) $\widehat{\alpha}^b_{\mathcal{C}} < \widehat{\alpha}^b_{\texttt{UN}}$. Therefore, $|\widehat{\alpha}^a_{\mathcal{C}}-\widehat{\alpha}^b_{\mathcal{C}}| \geq |\widehat{\alpha}^a_{\texttt{UN}}-\widehat{\alpha}^b_{\texttt{UN}}|$.
\end{proof}

\paragraph{The proof of Theorem \ref{thm:transition}.}
\begin{proof}
{The proof is under the conditions of Theorem \ref{thm:unique_equilibrium} such that there is unique equilibrium of qualification rate.
Under fairness constraint $\mathcal{C} = \texttt{EqOpt}$ or \texttt{DP}, consider 2D plane $\{(\alpha^a,\alpha^b): \alpha^a \in [0,1],\alpha^b \in [0,1]\}$, and note that equilibrium $(\widehat{\alpha}^a_{\mathcal{C}},\widehat{\alpha}^b_{\mathcal{C}})$ is the intersection of two curves $\mathcal{C}_1 = \{(\alpha^a,\alpha^b): \alpha^a = \psi^a_{\mathcal{C}}(\alpha^b),\alpha^b\in [0,1]\}$ and  $\mathcal{C}_2 = \{(\alpha^a,\alpha^b):\alpha^b = \psi^b_{\mathcal{C}}(\alpha^a),\alpha^a\in [0,1]\}$.
Consider a line $\{(\alpha^a,\alpha^b):\alpha^a = \alpha^b,\alpha^a\in[0,1],\alpha^b\in[0,1]\}$, which has unique intersection $\widetilde{\alpha}^a_{\mathcal{C}}$ with $\mathcal{C}_1$, and unique intersection $\widetilde{\alpha}^b_{\mathcal{C}}$ with  $\mathcal{C}_2$. That is, $\widetilde{\alpha}^a_{\mathcal{C}} = \psi^a_{\mathcal{C}}(\widetilde{\alpha}^a_{\mathcal{C}})$, $\widetilde{\alpha}^b_{\mathcal{C}} = \psi^b_{\mathcal{C}}(\widetilde{\alpha}^b_{\mathcal{C}})$.}

First of all, we show that \emph{if $ \frac{u_+}{u_-} \geq \frac{1-T_{10}}{T_{00}}\beta(\widehat{x})$, under Condition \ref{con:transition}\ref{con:transition_II}, $\widehat{\alpha}^b_{\texttt{UN}} < \widehat{\alpha}^a_{\texttt{UN}}$.} 

By Condition \ref{ass:inv_transition}, given any $\alpha^a = \alpha^b=\alpha$, the corresponding qualification profiles of $\mathcal{G}_a$, $\mathcal{G}_b$ satisfy the followings: $\gamma^b(\widehat{x}) = \gamma^a(\widehat{x})$; $\gamma^b(x) < \gamma^a(x), \forall x < \widehat{x}$; $\gamma^b(x) > \gamma^a(x),\forall x > \widehat{x}$. Let $\overline{\alpha}$ be qualification rate such that 
${\gamma^a(\widehat{x})= \gamma^b(\widehat{x}) = \frac{u_-}{u_++u_-} \Longrightarrow} 
\frac{u_+}{u_-} = \beta(\widehat{x})(\frac{1}{\overline{\alpha}}-1)$, where $\beta(\widehat{x}) \coloneqq \frac{G^a_0(\widehat{x})}{G^a_1(\widehat{x})} = \frac{G^b_0(\widehat{x})}{G^b_1(\widehat{x})}$, then $\forall \alpha \in [\overline{\alpha},1]$, 
$\gamma^a(\theta^a_{\texttt{UN}}(\alpha)) = \gamma^b(\theta^b_{\texttt{UN}}(\alpha)) = \frac{u_-}{u_++u_-} < \frac{1}{\beta(\widehat{x})(\frac{1}{\alpha}-1) +1} = \gamma^a(\widehat{x})= \gamma^b(\widehat{x}).$ 
Thus, 
$\forall \alpha \in [\overline{\alpha},1],\; \theta^a_{\texttt{UN}}(\alpha) < \theta^b_{\texttt{UN}}(\alpha) < \widehat{x}$, which implies $\mathbb{G}_1^a(\theta^a_{\texttt{UN}}(\alpha)) < \mathbb{G}_1^b(\theta^b_{\texttt{UN}}(\alpha))$
and  $\mathbb{G}_0^a(\theta^a_{\texttt{UN}}(\alpha)) < \mathbb{G}_0^b(\theta^b_{\texttt{UN}}(\alpha))${; furthermore, under Condition \ref{con:transition}\ref{con:transition_II}, we have 
\begin{eqnarray*}
\frac{1-T_{11}}{T_{01}}< \frac{1-g^{1a}(\theta^a_{\texttt{UN}}(\alpha) )}{g^{0a}(\theta^a_{\texttt{UN}}(\alpha) )} < \frac{1-g^{1b}(\theta^b_{\texttt{UN}}(\alpha) )}{g^{0b}(\theta^b_{\texttt{UN}}(\alpha) )} 
< \frac{1-T_{10}}{T_{00}}, ~\forall \alpha \in [\overline{\alpha},1].
\end{eqnarray*}
Because $\widehat{\alpha}^a_{\texttt{UN}}$ and $\widehat{\alpha}^b_{\texttt{UN}}$ are solutions to balance equations, i.e.,
$\frac{1}{\widehat{\alpha}^a_{\texttt{UN}}}-1 = \frac{1-g^{1a}(\theta^a_{\texttt{UN}}(\widehat{\alpha}^a_{\texttt{UN}}) )}{g^{0a}(\theta^a_{\texttt{UN}}(\widehat{\alpha}^a_{\texttt{UN}}) )}, \frac{1}{\widehat{\alpha}^b_{\texttt{UN}}}-1 =
\frac{1-g^{1b}(\theta^b_{\texttt{UN}}(\widehat{\alpha}^b_{\texttt{UN}}) )}{g^{0b}(\theta^b_{\texttt{UN}}(\widehat{\alpha}^b_{\texttt{UN}}) )}$. If $\overline{\alpha} \leq \widehat{\alpha}^b_{\texttt{UN}}$, the  $\widehat{\alpha}^b_{\texttt{UN}} < \widehat{\alpha}^a_{\texttt{UN}}$ must hold under Condition \ref{con:transition}\ref{con:transition_II}. Next, we show that 
a sufficient condition of $\overline{\alpha} \leq \widehat{\alpha}^b_{\texttt{UN}}$ is $ \frac{u_+}{u_-} \geq \frac{1-T_{10}}{T_{00}}\beta(\widehat{x})$. 

$ \frac{u_+}{u_-} \geq \frac{1-T_{10}}{T_{00}}\beta(\widehat{x})
\Longrightarrow \frac{1}{\overline{\alpha}} - 1 \geq  \frac{1-T_{10}}{T_{00}}$. 
Since
$ \frac{1}{\widehat{\alpha}^b_{\texttt{UN}}} - 1 < \frac{1-T_{10}}{T_{00}}$, we have 
$\frac{1}{\widehat{\alpha}^b_{\texttt{UN}}} - 1<\frac{1}{\overline{\alpha}} - 1$. Thus, $\overline{\alpha} \leq \widehat{\alpha}^b_{\texttt{UN}}$. Therefore, if $ \frac{u_+}{u_-} \geq \frac{1-T_{10}}{T_{00}}\beta(\widehat{x})$, under Condition \ref{con:transition}\ref{con:transition_II}, $\widehat{\alpha}^b_{\texttt{UN}} < \widehat{\alpha}^a_{\texttt{UN}}$.}

\paragraph{Fairness constraint \texttt{EqOpt}.} 
Secondly, we show that for \texttt{EqOpt} fair policy,  if $ \frac{u_+}{u_-} \geq \frac{1-T_{10}}{T_{00}}\beta(\widehat{x})$, under Condition \ref{con:transition}\ref{con:transition_II}, $\widehat{\alpha}_{\texttt{UN}}^a -  \widehat{\alpha}_{\texttt{UN}}^b > \widehat{\alpha}_{\texttt{EqOpt}}^a -  \widehat{\alpha}_{\texttt{EqOpt}}^b 
\geq
0$. Because two curves $\mathcal{C}_1, \mathcal{C}_2$ are monotonic increasing.
It's sufficient to show two parts: (1) $\widetilde{\alpha}^a_{\texttt{EqOpt}}< \widehat{\alpha}^a_{\texttt{UN}}$, $\widetilde{\alpha}^b_{\texttt{EqOpt}}> \widehat{\alpha}^b_{\texttt{UN}}$; (2) $\widetilde{\alpha}^a_{\texttt{EqOpt}}\geq\widetilde{\alpha}^b_{\texttt{EqOpt}}$. 

Under \texttt{EqOpt} constraint, $\forall \alpha^a,\alpha^b$, $\mathbb{G}_1^a(\theta^a_{\texttt{EqOpt}}(\alpha^a,\alpha^b)) = \mathbb{G}_1^b(\theta^b_{\texttt{EqOpt}}(\alpha^a,\alpha^b))$ must hold so that $\theta^a_{\texttt{EqOpt}}(\alpha^a,\alpha^b) = \theta^b_{\texttt{EqOpt}}(\alpha^a,\alpha^b)$. Consider the case $\alpha^a = \alpha^b = \alpha$, $\forall \alpha \geq \overline{\alpha}$, we have $\theta^a_{\texttt{EqOpt}}(\alpha,\alpha) = \theta^b_{\texttt{EqOpt}}(\alpha,\alpha)$ and $\theta^a_{\texttt{UN}}(\alpha) < \theta^b_{\texttt{UN}}(\alpha)$. It implies that $\theta^a_{\texttt{UN}}(\alpha) <\theta^a_{\texttt{EqOpt}}(\alpha,\alpha) = \theta^b_{\texttt{EqOpt}}(\alpha,\alpha) < \theta^b_{\texttt{UN}}(\alpha) < \widehat{x}$, otherwise Equation \eqref{eq:opt_fair_policy} will be violated. Therefore, the followings hold $\forall \alpha \in [\overline{\alpha},1]$,
\begin{eqnarray*}
 \frac{1-g^{1a}(\theta^a_{\texttt{EqOpt}}(\alpha,\alpha) )}{g^{0a}(\theta^a_{\texttt{EqOpt}}(\alpha,\alpha))} > \frac{1-g^{1a}(\theta^a_{\texttt{UN}}(\alpha) )}{g^{0a}(\theta^a_{\texttt{UN}}(\alpha) )}; ~~~
  \frac{1-g^{1b}(\theta^b_{\texttt{EqOpt}}(\alpha,\alpha) )}{g^{0b}(\theta^b_{\texttt{EqOpt}}(\alpha,\alpha))} < \frac{1-g^{1b}(\theta^b_{\texttt{UN}}(\alpha) )}{g^{0b}(\theta^b_{\texttt{UN}}(\alpha) )}. 
\end{eqnarray*}
$\forall s \in \{a,b\}$, $\widetilde{\alpha}^s_{\texttt{EqOpt}}$ is the solution to $ \frac{1-g^{1s}(\theta^s_{\texttt{EqOpt}}(\alpha,\alpha) )}{g^{0s}(\theta^s_{\texttt{EqOpt}}(\alpha,\alpha))} = \frac{1}{\alpha}-1$ while $\widehat{\alpha}^s_{\texttt{UN}}$  is the solution to $ \frac{1-g^{1s}(\theta^s_{\texttt{UN}}(\alpha) )}{g^{0s}(\theta^s_{\texttt{UN}}(\alpha))} = \frac{1}{\alpha}-1$. Since $\overline{\alpha} \leq \widehat{\alpha}^b_{\texttt{UN}}< \widehat{\alpha}^a_{\texttt{UN}}$, it implies  $\widetilde{\alpha}^a_{\texttt{EqOpt}}< \widehat{\alpha}^a_{\texttt{UN}}$, $\widetilde{\alpha}^b_{\texttt{EqOpt}}> \widehat{\alpha}^b_{\texttt{UN}}$. 

Next, show that $\widetilde{\alpha}^a_{\texttt{EqOpt}}\geq \widetilde{\alpha}^b_{\texttt{EqOpt}}$.
$\forall \alpha\geq \overline{\alpha}$, $\theta^a_{\texttt{EqOpt}}(\alpha,\alpha) = \theta^b_{\texttt{EqOpt}}(\alpha,\alpha)$ implies $\mathbb{G}^a_1(\theta^a_{\texttt{EqOpt}}(\alpha,\alpha) ) = \mathbb{G}^b_1(\theta^b_{\texttt{EqOpt}}(\alpha,\alpha) )$ and $\mathbb{G}^a_0(\theta^a_{\texttt{EqOpt}}(\alpha,\alpha) ) \leq \mathbb{G}^b_0(\theta^b_{\texttt{EqOpt}}(\alpha,\alpha) )$. Therefore,  
$$\frac{1-g^{1a}(\theta^a_{\texttt{EqOpt}}(\alpha,\alpha) )}{g^{0a}(\theta^a_{\texttt{EqOpt}}(\alpha,\alpha))} \leq
\frac{1-g^{1b}(\theta^b_{\texttt{EqOpt}}(\alpha,\alpha) )}{g^{0b}(\theta^b_{\texttt{EqOpt}}(\alpha,\alpha))}.$$ 
Intersections with function $\frac{1}{\alpha}-1$ satisfies $\widetilde{\alpha}^a_{\texttt{EqOpt}}\geq
\widetilde{\alpha}^b_{\texttt{EqOpt}}$. 

It thus concludes that $\widehat{\alpha}_{\texttt{UN}}^a -  \widehat{\alpha}_{\texttt{UN}}^b > \widehat{\alpha}_{\texttt{EqOpt}}^a -  \widehat{\alpha}_{\texttt{EqOpt}}^b 
\geq
0$. 
\paragraph{Fairness constraint \texttt{DP}.} 
Finally, consider \texttt{DP} fair policy, where $\forall \alpha^a,\alpha^b$, $(1-\alpha^a)\mathbb{G}^a_0(\theta^a_\texttt{DP}(\alpha^a,\alpha^b)) + \alpha^a\mathbb{G}^a_1(\theta^a_\texttt{DP}(\alpha^a,\alpha^b)) = (1-\alpha^b)\mathbb{G}^b_0(\theta^b_\texttt{DP}(\alpha^a,\alpha^b)) + \alpha^b\mathbb{G}^b_1(\theta^b_\texttt{DP}(\alpha^a,\alpha^b))$ must hold.

We first show that  under Condition \ref{con:transition}\ref{con:transition_II}, $\widetilde{\alpha}^a_{\texttt{DP}}< \widehat{\alpha}^a_{\texttt{UN}}$, $\widetilde{\alpha}^b_{\texttt{DP}}> \widehat{\alpha}^b_{\texttt{UN}}$. Consider the case $\alpha^a=\alpha^b=\alpha$, $\forall \alpha\geq\overline{\alpha}$. Since
$\forall x, \, (1-\alpha)\mathbb{G}^b_0(x) + \alpha\mathbb{G}^b_1(x) \geq  (1-\alpha)\mathbb{G}^a_0(x) + \alpha\mathbb{G}^a_1(x),$  $(1-\alpha)\mathbb{G}^a_0(\theta^a_\texttt{DP}(\alpha,\alpha)) + \alpha\mathbb{G}^a_1(\theta^a_\texttt{DP}(\alpha,\alpha)) = (1-\alpha)\mathbb{G}^b_0(\theta^b_\texttt{DP}(\alpha,\alpha)) + \alpha\mathbb{G}^b_1(\theta^b_\texttt{DP}(\alpha,\alpha))$ implies 
 $\theta^a_\texttt{DP}(\alpha,\alpha) 
 \geq\theta^b_\texttt{DP}(\alpha,\alpha)$.
Because $\theta^a_\texttt{UN}(\alpha) < \theta^b_\texttt{UN}(\alpha)$, $\forall \alpha \geq \overline{\alpha}$. It implies that $\theta^a_\texttt{DP}(\alpha,\alpha) > \theta^a_\texttt{UN}(\alpha)$ and 
$\widehat{x} > \theta^b_\texttt{UN}(\alpha) >
\theta^b_\texttt{DP}(\alpha,\alpha)$ must hold. 
Therefore, $\forall \alpha \in [\overline{\alpha},1]$,
\begin{eqnarray*}
 \frac{1-g^{1a}(\theta^a_{\texttt{DP}}(\alpha,\alpha) )}{g^{0a}(\theta^a_{\texttt{DP}}(\alpha,\alpha))} > \frac{1-g^{1a}(\theta^a_{\texttt{UN}}(\alpha) )}{g^{0a}(\theta^a_{\texttt{UN}}(\alpha) )}; ~~~
  \frac{1-g^{1b}(\theta^b_{\texttt{DP}}(\alpha,\alpha) )}{g^{0b}(\theta^b_{\texttt{DP}}(\alpha,\alpha))} < \frac{1-g^{1b}(\theta^b_{\texttt{UN}}(\alpha) )}{g^{0b}(\theta^b_{\texttt{UN}}(\alpha) )} 
\end{eqnarray*}
Similar to reasoning in \texttt{EqOpt} case, we have $\widetilde{\alpha}^a_{\texttt{DP}}< \widehat{\alpha}^a_{\texttt{UN}}$, $\widetilde{\alpha}^b_{\texttt{DP}}> \widehat{\alpha}^b_{\texttt{UN}}$.

Different from \texttt{EqOpt} fairness where $\widetilde{\alpha}^a_{\texttt{EqOpt}}\geq
\widetilde{\alpha}^b_{\texttt{EqOpt}}$, both $\widetilde{\alpha}^a_{\texttt{DP}}\geq
\widetilde{\alpha}^b_{\texttt{DP}}$ and $\widetilde{\alpha}^a_{\texttt{DP}}\leq
\widetilde{\alpha}^b_{\texttt{DP}}$ are likely to occur, depending on distributions $G^a_0(x), G^b_0(x)$, $G^a_1(x)$ and $ G^b_1(x)$. It is because $\theta^a_\texttt{DP}(\alpha,\alpha) > \theta^b_\texttt{DP}(\alpha,\alpha)$ can result in either $\mathbb{G}^a_0(\theta^a_{\texttt{DP}}(\alpha,\alpha) ) \leq \mathbb{G}^b_0(\theta^b_{\texttt{DP}}(\alpha,\alpha) )$ or $\mathbb{G}^a_0(\theta^a_{\texttt{DP}}(\alpha,\alpha) ) \geq \mathbb{G}^b_0(\theta^b_{\texttt{DP}}(\alpha,\alpha) )$. 

For these two outcomes, if $\widetilde{\alpha}^a_{\texttt{DP}}\geq \widetilde{\alpha}^b_{\texttt{DP}}$, then \texttt{DP} fair policy results in a more equitable equilibrium than unconstrained policy; if $\widetilde{\alpha}^a_{\texttt{DP}}\leq\widetilde{\alpha}^b_{\texttt{DP}}$, it means the disadvantaged group is flipped from $\mathcal{G}_b$ to $\mathcal{G}_a$.

\end{proof}
\paragraph{The proof of Proposition \ref{prop:intervention_policy}.}
\begin{proof}
In the proof, we simplify the notations by removing subscript $\mathcal{C}$. 

Let $\psi^s(\cdot)$, $\psi^{s'}(\cdot)$ be balanced function of policies $(\theta^{a},\theta^{b})$ and $(\theta^{a'},\theta^{b'})$, respectively.

According to the balanced equation \eqref{eq:balanced_eqn_a},  \begin{equation*}
      \frac{1}{\alpha^s} - 1 =\frac{1-g^{1s}(\theta^s(\alpha^a,\alpha^b))}{g^{0s}(\theta^s(\alpha^a,\alpha^b))}
      =
      \frac{1-(T_{11}^s(1-\mathbb{G}_1^s(\theta^s(\alpha^a,\alpha^b))) + T_{10}^s\mathbb{G}_1^s(\theta^s(\alpha^a,\alpha^b)))}{T_{01}^s(1-\mathbb{G}_0^s(\theta^s(\alpha^a,\alpha^b))) + T_{00}^s\mathbb{G}_0^s(\theta^s(\alpha^a,\alpha^b))}.
  \end{equation*}
Under Condition \ref{con:transition_II}, $\forall \alpha^a, \alpha^b \in [0,1]$, $\theta^{a'}(\alpha^a, \alpha^b)<\theta^{a}(\alpha^a, \alpha^b)$ and $\theta^{b'}(\alpha^a, \alpha^b)<\theta^{b}(\alpha^a, \alpha^b)$. 
 
Under Condition \ref{con:transition_I}, $\forall \alpha^a, \alpha^b \in [0,1]$, $\theta^{a'}(\alpha^a, \alpha^b)>\theta^{a}(\alpha^a, \alpha^b)$ and $\theta^{b'}(\alpha^a, \alpha^b)>\theta^{b}(\alpha^a, \alpha^b)$.
 
Both imply that $\frac{1-g^{1s}(\theta^s(\alpha^a,\alpha^b))}{g^{0s}(\theta^s(\alpha^a,\alpha^b))}>\frac{1-g^{1s}(\theta^{s'}(\alpha^a,\alpha^b))}{g^{0s}(\theta^{s'}(\alpha^a,\alpha^b))}$, and $\forall \alpha^a,\alpha^b \in [0,1]$, $\psi^a(\alpha^b) < \psi^{a'}(\alpha^b)$ and  $\psi^b(\alpha^a) < \psi^{b'}(\alpha^a)$ hold. As a consequence, $\widehat{\alpha}^{a'}>\widehat{\alpha}^a$ and $\widehat{\alpha}^{b'}>\widehat{\alpha}^b$. 

Now consider the long-run average utility of institute $\overline{U}(\theta^a,\theta^b) = \lim_{T\to\infty}\frac{1}{T}\sum_{t=1}^T\mathcal{U}_t(\theta^a,\theta^b)$, where the instantaneous utility at $t$ under threshold policies $\theta^a, \theta^b$ is
\begin{eqnarray*}
\mathcal{U}_t(\theta^a,\theta^b) &=& \sum_{s=a,b}p_s \mathbb{E}_{X_t|S=s}[\textbf{1}(X_t\geq \theta^s)(\gamma_t^s(X_t)(u_++u_-) -u_-)] \\&=& \sum_{s=a,b}p_s\int_{\theta^s}^{\infty}(\gamma^s_t(x)(u_++u_-) -u_-)\mathbb{P}(X_t=x\mid S=s) dx\\
&=&\sum_{s=a,b}p_s\int_{\theta^s}^{\infty} \alpha^s_t\big(G^s_1(x)u_+ +G^s_0(x)u_-\big)- G^s_0(x)u_-dx
\end{eqnarray*}
In the followings, we use a special case ($\mathcal{C} = \texttt{EqOpt}$, $G^a_y(x) = G^b_y(x), \forall x,y=0,1$, under Condition \ref{con:transition}\ref{con:transition_II}) to show that $\overline{U}(\theta^{a'},\theta^{b'}) > \overline{U}(\theta^a,\theta^b)$ can be attained, i.e., the long-run average utility under policy $(\theta^{a'},\theta^{b'})$ can be higher than myopic optimal policy $(\theta^{a},\theta^{b})$.

Since the qualification rates of two groups converge to equilibrium, $\overline{U}(\theta^a,\theta^b) =\mathcal{U}_{\infty}(\theta^a,\theta^b)$ is the same as instantaneous expected utility of institute at the equilibrium state. To show that $\overline{U}(\theta^{a'},\theta^{b'}) > \overline{U}(\theta^a,\theta^b)$, we prove the following holds,
\begin{eqnarray}\label{eq:highUtility}
\sum_{s=a,b}p_s\int_{\theta^{s'}(\widehat{\alpha}^{a'},\widehat{\alpha}^{b'})}^{\infty}f(x;\widehat{\alpha}^{s'})dx >\sum_{s=a,b}p_s \int_{\theta^s(\widehat{\alpha}^a,\widehat{\alpha}^b)}^{\infty}f(x;\widehat{\alpha}^{s})dx
\end{eqnarray}
where $f(x;\widehat{\alpha}^{s})\coloneqq  \widehat{\alpha}^{s}\big(G^s_1(x)u_+ +G^s_0(x)u_-\big)- G^s_0(x)u_-$.

Because $\widehat{\alpha}^{s'}>\widehat{\alpha}^{s}$, $\theta^{s'}(\widehat{\alpha}^{a'},\widehat{\alpha}^{b'})<\theta^{s}(\widehat{\alpha}^{a'},\widehat{\alpha}^{b'})<\theta^{s}(\widehat{\alpha}^{a},\widehat{\alpha}^{b})$ holds under Condition \ref{con:transition_II}. LHS of above inequality can be written as
$$\sum_{s=a,b}p_s\Big(\int_{\theta^{s'}(\widehat{\alpha}^{a'},\widehat{\alpha}^{b'})}^{\theta^{s}(\widehat{\alpha}^{a},\widehat{\alpha}^{b})}f(x;\widehat{\alpha}^{s'})dx + \int_{\theta^{s}(\widehat{\alpha}^{a},\widehat{\alpha}^{b})}^{\infty}f(x;\widehat{\alpha}^{s'})dx\Big). $$
Inequality \eqref{eq:highUtility} can further be re-organized,
\begin{eqnarray}\label{eq:highUtility1}
\sum_{s=a,b}p_s\int_{\theta^{s'}(\widehat{\alpha}^{a'},\widehat{\alpha}^{b'})}^{\theta^{s}(\widehat{\alpha}^{a},\widehat{\alpha}^{b})}f(x;\widehat{\alpha}^{s'})dx> \sum_{s=a,b}p_s\int_{\theta^s(\widehat{\alpha}^a,\widehat{\alpha}^b)}^{\infty}\big(f(x;\widehat{\alpha}^{s})-f(x;\widehat{\alpha}^{s'})\big)dx
\end{eqnarray}
Consider a special case where $\mathcal{C} = \texttt{EqOpt}$ and $G_y^a(x) = G^b_y(x) = G_y(x), \forall x$, $\forall y\in \{0,1\}$. Then $\forall \alpha^a,\alpha^b$, we have $\theta^a(\alpha^a,\alpha^b) = \theta^b(\alpha^a,\alpha^b)$ and $\theta^{a'}(\alpha^a,\alpha^b) = \theta^{b'}(\alpha^a,\alpha^b)$. Inequality \eqref{eq:highUtility1} can be reduced to the following, $\forall s\in \{a,b\}$, simplify notations and let $\widehat{\theta}\coloneqq \theta^{s}(\widehat{\alpha}^{a},\widehat{\alpha}^{b})$, $\widehat{\theta}'\coloneqq \theta^{s'}(\widehat{\alpha}^{a'},\widehat{\alpha}^{b'})$.
\begin{eqnarray}\label{eq:highUtility2}
&&\Big(p_a\widehat{\alpha}^{a'}+p_b\widehat{\alpha}^{b'}\Big)\Big(\mathbb{G}_1(\widehat{\theta})-\mathbb{G}_1(\widehat{\theta}')\Big)u_+ \nonumber\\&+& \underbrace{\Big(u_+(1-\mathbb{G}_1(\widehat{\theta}))+u_-(1-\mathbb{G}_0(\widehat{\theta}))\Big)\Big(p_a(\widehat{\alpha}^{a'}-\widehat{\alpha}^{a})+p_b(\widehat{\alpha}^{b'}-\widehat{\alpha}^{b})\Big)}_{\textbf{term 1}} \nonumber\\
&>& \Big(p_a(1-\widehat{\alpha}^{a'})+p_b(1-\widehat{\alpha}^{b'})\Big)\Big(\mathbb{G}_0(\widehat{\theta})-\mathbb{G}_0(\widehat{\theta}')\Big)u_-
\end{eqnarray}
Because $\frac{1}{ \widehat{\alpha}^{s'}}-1 = \frac{1-g^{1s}(\widehat{\theta'}) }{g^{0s}(\widehat{\theta'}) }$ and $\frac{1}{ \widehat{\alpha}^{s}}-1 = \frac{1-g^{1s}(\widehat{\theta}) }{g^{0s}(\widehat{\theta})}$.
\begin{eqnarray*}
\widehat{\alpha}^{s'}-\widehat{\alpha}^{s} > \frac{T^s_{01}-T^s_{00}}{1-T^s_{10}+T^s_{01}}(\mathbb{G}_0(\widehat{\theta})-\mathbb{G}_0(\widehat{\theta}') )
\end{eqnarray*}
We have $\textbf{term 1} >$
\begin{eqnarray*}
 \underbrace{\Big(\frac{u_+}{u_-}(1-\mathbb{G}_1(\widehat{\theta}))+(1-\mathbb{G}_0(\widehat{\theta}))\Big)\Big(p_a \frac{T^a_{01}-T^a_{00}}{1-T^a_{10}+T^a_{01}} + p_b \frac{T^b_{01}-T^b_{00}}{1-T^b_{10}+T^b_{01}} \Big)}_{\coloneqq h(\widehat{\theta})>0}
\Big(\mathbb{G}_0(\widehat{\theta})-\mathbb{G}_0(\widehat{\theta}') \Big)u_-
\end{eqnarray*}
For the optimal \texttt{EqOpt} fair threshold $\theta(\widehat{\alpha}^{a'},\widehat{\alpha}^{b'})$, the following holds 
\begin{eqnarray*}
\Big(p_a\widehat{\alpha}^{a'}+p_b\widehat{\alpha}^{b'}\Big){G}_1(\theta(\widehat{\alpha}^{a'},\widehat{\alpha}^{b'}))u_+ = \Big(p_a(1-\widehat{\alpha}^{a'})+p_b(1-\widehat{\alpha}^{b'})\Big)G_0(\theta(\widehat{\alpha}^{a'},\widehat{\alpha}^{b'}))u_-\\
\Big(p_a\widehat{\alpha}^{a'}+p_b\widehat{\alpha}^{b'}\Big){G}_1(x)u_+ > \Big(p_a(1-\widehat{\alpha}^{a'})+p_b(1-\widehat{\alpha}^{b'})\Big)G_0(x)u_-, \forall x > \theta(\widehat{\alpha}^{a'},\widehat{\alpha}^{b'})\\
\Big(p_a\widehat{\alpha}^{a'}+p_b\widehat{\alpha}^{b'}\Big){G}_1(x)u_+ < \Big(p_a(1-\widehat{\alpha}^{a'})+p_b(1-\widehat{\alpha}^{b'})\Big)G_0(x)u_-, \forall x < \theta(\widehat{\alpha}^{a'},\widehat{\alpha}^{b'})
\end{eqnarray*}
It implies that $\exists$ some $\delta>0$ s.t. $\forall x\in (\theta(\widehat{\alpha}^{a'},\widehat{\alpha}^{b'})-\delta,\theta(\widehat{\alpha}^{a'},\widehat{\alpha}^{b'})+\delta ) \coloneqq \mathcal{B}(\theta(\widehat{\alpha}^{a'},\widehat{\alpha}^{b'}),\delta )$,
\begin{eqnarray*}
\Big(p_a\widehat{\alpha}^{a'}+p_b\widehat{\alpha}^{b'}\Big){G}_1(x)u_+ + h(\widehat{\theta})G_0(x)u_- > \Big(p_a(1-\widehat{\alpha}^{a'})+p_b(1-\widehat{\alpha}^{b'})\Big)G_0(x)u_-.
\end{eqnarray*}
$\widehat{\theta},\widehat{\theta}'\in \mathcal{B}(\theta(\widehat{\alpha}^{a'},\widehat{\alpha}^{b'}),\delta )$ can be satisfied as long as $|\theta^s(\alpha^a,\alpha^b)-\theta^{s'}(\alpha^a,\alpha^b)| \leq \epsilon$ for some sufficiently small $\epsilon>0$.

Using the mean value theorem, $\exists G_y(x)$ and $\widetilde{x}\in (\widehat{\theta}',\widehat{\theta} )\subset \mathcal{B}(\theta(\widehat{\alpha}^{a'},\widehat{\alpha}^{b'}),\delta )$ s.t.
\begin{eqnarray*}
&&\Big(p_a\widehat{\alpha}^{a'}+p_b\widehat{\alpha}^{b'}\Big)(\mathbb{G}_1(\widehat{\theta})- \mathbb{G}_1(\widehat{\theta}'))u_+ +  h(\widehat{\theta})(\mathbb{G}_0(\widehat{\theta})- \mathbb{G}_0(\widehat{\theta}'))u_-
\\&=& 
\Big(\Big(p_a\widehat{\alpha}^{a'}+p_b\widehat{\alpha}^{b'}\Big){G}_1(\widetilde{x})u_+ +  h(\widehat{\theta})G_0(\widetilde{x})u_-\Big)(\widehat{\theta}-\widehat{\theta}')\\
&>& \Big(\Big(p_a(1-\widehat{\alpha}^{a'})+p_b(1-\widehat{\alpha}^{b'})\Big)G_0(\widetilde{x})u_-\Big)(\widehat{\theta}-\widehat{\theta}') \\&\geq& \Big(p_a(1-\widehat{\alpha}^{a'})+p_b(1-\widehat{\alpha}^{b'})\Big)(\mathbb{G}_0(\widehat{\theta})- \mathbb{G}_0(\widehat{\theta}'))u_-.
\end{eqnarray*}
Therefore, inequality \eqref{eq:highUtility2}
holds and $\overline{U}(\theta^{a'},\theta^{b'})>\overline{U}(\theta^{a},\theta^{b})$. 
\end{proof}

\paragraph{The proof of Proposition \ref{prop:existance_eq_polices}.}
\begin{proof}
To ensure $\alpha^s_t\rightarrow \widehat{\alpha}$, threshold policy $\theta^s(\alpha^s)$ as a function of $\alpha^s\in[0,1]$ should be designed such that $ \frac{1-g^{1s}(\theta^s(\alpha^s))}{g^{0s}(\theta^s(\alpha^s))} = \frac{1}{\alpha^s}-1$ has a unique solution $\widehat{\alpha}$. Let $\mathcal{I}_s \coloneqq \big[\frac{1-\max\{T_{11}^s,T^s_{10}\} }{\max\{T_{01}^s,T_{00}^s\}}, \frac{1-\min\{T_{11}^s,T^s_{10}\} }{\min\{T_{01}^s,T_{00}^s\}}\big]$, then $ \frac{1-g^{1s}(\theta^s(\alpha^s))}{g^{0s}(\theta^s(\alpha^s))} \in \mathcal{I}_s$ for any threshold policy $\theta^s(\alpha^s)$.

If $ \mathcal{I}_a\cap \mathcal{I}_b = \emptyset$, then 
$\frac{1-g^{1a}(\theta^a({\alpha}))}{g^{0a}(\theta^a({\alpha}))} = \frac{1-g^{1b}(\theta^b({\alpha}))}{g^{0b}(\theta^b({\alpha}))}$ can never be attained, i.e., no threshold policy can result in equitable equilibrium.

If $ \mathcal{I}_a\cap \mathcal{I}_b \neq \emptyset$, then $\forall \widehat{\alpha} \in  \mathcal{I}_a\cap \mathcal{I}_b$ and $\forall s\in \{a,b\}$, 
there exists threshold policy $\theta^s(\alpha^s)$ such that $\frac{1-g^{1s}(\theta^s(\widehat{\alpha}))}{g^{0s}(\theta^s(\widehat{\alpha}))}= \frac{1}{\widehat{\alpha}}-1$. Specifically, under Condition \ref{con:transition}\ref{con:transition_II} (resp. \ref{con:transition}\ref{con:transition_I}), function $$h^s(x)\coloneqq \frac{1-g^{1s}(x)}{g^{0s}(x)} = \frac{1-(T_{11}^s(1-\mathbb{G}_1^s(x))+T_{10}^s\mathbb{G}_1^s(x))}{T_{01}^s(1-\mathbb{G}_0^s(x))+T_{00}^s\mathbb{G}_0^s(x)}$$ is strictly increasing (resp. decreasing) in $x\in(-\infty,+\infty)$ from $\frac{1-T^s_{11}}{T^s_{01}}$ (resp. $\frac{1-T^s_{10}}{T^s_{00}}$) to $\frac{1-T^s_{10}}{T^s_{00}}$ (resp. $\frac{1-T^s_{11}}{T^s_{01}}$) and any non-increasing function $\theta^s(\alpha^s)$ that satisfies $\theta^s(\widehat{\alpha}) = (h^s)^{-1}(\frac{1}{\widehat{\alpha}}-1)$ can result in  $\alpha^s_t\rightarrow \widehat{\alpha}$, where $(h^s)^{-1}(\cdot)$ is the inverse function of $h^s(\cdot)$.
\end{proof}

\paragraph{The proof of Proposition \ref{prop:intervention_transition}.}
\begin{proof}
According to the balanced equation \eqref{eq:balanced_eqn_a},  \begin{equation*}
      \frac{1}{\alpha^s} - 1 =\frac{1-g^{1s}(\theta^s(\alpha^a,\alpha^b))}{g^{0s}(\theta^s(\alpha^a,\alpha^b))}
      =
      \frac{1-(T_{11}^s(1-\mathbb{G}_1^s(\theta^s(\alpha^a,\alpha^b))) + T_{10}^s\mathbb{G}_1^s(\theta^s(\alpha^a,\alpha^b)))}{T_{01}^s(1-\mathbb{G}_0^s(\theta^s(\alpha^a,\alpha^b))) + T_{00}^s\mathbb{G}_0^s(\theta^s(\alpha^a,\alpha^b))}.
  \end{equation*}
$\forall \alpha^a,\alpha^b \in[0,1] $, increasing any $T_{yd}^s$ decreases $\frac{1-g^{1s}(\theta^s(\alpha^a,\alpha^b))}{g^{0s}(\theta^s(\alpha^a,\alpha^b))}$. Let $\psi^{s'}(\cdot)$ be the consequent balanced function after increasing $T_{yd}^s$, and $\widehat{\alpha}^{s'}$ be corresponding equilibrium. 
Given any $\alpha^a,\alpha^b \in [0,1]$, we have
$\psi^a(\alpha^b) < \psi^{a'}(\alpha^b)$ and  $\psi^b(\alpha^a) < \psi^{b'}(\alpha^a)$. Therefore, 
$\widehat{\alpha}^{a'}>\widehat{\alpha}^a$ and 
$\widehat{\alpha}^{b'}>\widehat{\alpha}^b$.
\end{proof}

\bibliographystyle{abbrvnat}
\end{document}